\documentclass{article}

\usepackage{microtype}
\usepackage{graphicx}
\usepackage{booktabs} %
\usepackage{makecell}
\usepackage{hyperref}

\usepackage[accepted]{icml2024}

\usepackage{amsmath}
\usepackage{amssymb}
\usepackage{mathtools}
\usepackage{amsthm}
\usepackage{physics}
\usepackage{bbm}

\usepackage[capitalize,noabbrev]{cleveref}

\theoremstyle{plain}
\newtheorem{theorem}{Theorem}[section]
\newtheorem{proposition}[theorem]{Proposition}
\newtheorem{lemma}[theorem]{Lemma}
\newtheorem{corollary}[theorem]{Corollary}
\theoremstyle{definition}
\newtheorem{definition}[theorem]{Definition}
\newtheorem{assumption}[theorem]{Assumption}
\theoremstyle{remark}
\newtheorem{remark}[theorem]{Remark}

\usepackage[disable,textsize=tiny]{todonotes}

\icmltitlerunning{Particle Denoising Diffusion Sampler}

\usepackage{upgreek}
\usepackage{subcaption}

\def\E{\mathbb{E}}

\def\R{\mathbb{R}}
\def\P{\mathbb{P}}

\def\Q{\mathbb{Q}}

\def \rmd {\mathrm{d}}

\newcommand \KLLigne[2]{\mathrm{KL}(#1|#2)}
\newcommand \chisquared[2]{\chi^2(#1||#2)}

\newcommand{\pr}[1]{\left( #1 \right)} %
\newcommand{\ps}[1]{\left[ #1 \right]} %
\newcommand{\px}[1]{\left\{ #1 \right\}} %
\newcommand{\CE}[2]{\E\ps{\left. {#1} \right \vert {#2}}} %
\newcommand{\Var}{\operatorname{Var}}
\newcommand{\VarE}[2]{\Var\pr{\left. {#1} \right \vert {#2}}}  %

\newcommand{\gradlog}{\nabla\log}

\newcommand{\simiid}{\overset{\mathrm{iid}}{\sim}}

\renewcommand{\algorithmicrequire}{\textbf{Input:}}
\renewcommand{\algorithmicensure}{\textbf{Output:}}

\begin{document}

\twocolumn[
	\icmltitle{Particle Denoising Diffusion Sampler}

	\icmlsetsymbol{equal}{*}

	\begin{icmlauthorlist}
		\icmlauthor{Angus Phillips}{oxf}
		\icmlauthor{Hai-Dang Dau}{oxf}
		\icmlauthor{Michael John Hutchinson}{oxf}
		\icmlauthor{Valentin De Bortoli}{cnrs}
		\icmlauthor{George Deligiannidis}{oxf}
		\icmlauthor{Arnaud Doucet}{oxf}
	\end{icmlauthorlist}

	\icmlaffiliation{oxf}{University of Oxford}
	\icmlaffiliation{cnrs}{CNRS, ENS Ulm}

	\icmlcorrespondingauthor{Angus Phillips}{angus.phillips@stats.ox.ac.uk}

	\vskip 0.3in
]  %

\printAffiliationsAndNotice{}  %

\begin{abstract}
	Denoising diffusion models have become ubiquitous for generative modeling. The core idea is to transport the data distribution to a Gaussian by using a diffusion. Approximate samples from the data distribution are then obtained by estimating the time-reversal of this diffusion using score matching ideas. We follow here a similar strategy to sample from unnormalized probability densities and compute their normalizing constants. However, the time-reversed diffusion is here simulated by using an original iterative particle scheme relying on a novel score matching loss. Contrary to standard denoising diffusion models, the resulting Particle Denoising Diffusion Sampler (PDDS) provides asymptotically consistent estimates under mild assumptions. We demonstrate PDDS on multimodal and high dimensional sampling tasks.
\end{abstract}

\section{Introduction}
Consider a target probability density $\pi$ on $\mathbb{R}^d$ of the form
\vspace{-0.2cm}
\begin{equation}
	\label{eq:target}
	\pi(x)=\frac{\gamma(x)}{\mathcal{Z}},\qquad \mathcal{Z}=\int_{\mathbb{R}^d} \gamma(x) \mathrm{d}x,
	\vspace{-0.2cm}
\end{equation}
where $\gamma:\mathbb{R}^d \rightarrow \mathbb{R}^{+}$ can be evaluated pointwise but its normalizing constant $\mathcal{Z}$ is intractable. We develop here a Monte Carlo scheme to sample approximately from $\pi$ and estimate $\mathcal{Z}$.

We follow an approach inspired by denoising diffusion models \citep{ho2020denoising,song2020score,song2019generative} by considering a ``noising'' diffusion progressively transporting the original target to a Gaussian. The time-reversal of this diffusion, the ``denoising'' diffusion, allows us theoretically to sample from the target starting from Gaussian noise. However, it is impossible to simulate this process exactly as its drift depends on the gradient of the logarithm of the intractable marginal densities of the noising diffusion, i.e.\ the score. For generative modeling, where one has access to samples from \(\pi\), one can rely on neural networks and score matching \cite{Hyvarinen:2005a,vincent2011connection}. This strategy is not applicable in the Monte Carlo sampling context as we cannot sample the ``noising'' diffusion since we do not have access to any samples from \(\pi\) to approximate the initial distribution of the diffusion with. %

The idea of using denoising diffusion models for Monte Carlo sampling has already been explored by \citet{berner2022optimal,mcdonald2022proposal,vargasDDSampler2023,huang2023monte,richter2023improved,zhang2023diffusion}. \citet{berner2022optimal, richter2023improved,vargasDDSampler2023} focus on the minimization of a reverse Kullback--Leibler divergence or log-variance criterion while \citet{mcdonald2022proposal} rely on an importance sampling scheme which scales poorly in high dimensions. Finally, \citet{huang2023monte} relies on a series of Markov chain Monte Carlo (MCMC) to estimate the score. This last scheme does not provide estimates of normalizing constants.

We develop here an alternative approach inspired by denoising diffusion models with guidance. Guided diffusions combine pre-trained diffusion models with a guidance term derived from a likelihood to sample approximately from posterior distributions; see e.g. \citet{song2020score,chung2023diffusion,song2023pseudoinverseguided,corso2023particle}. While they provide samples with appealing perceptual properties, they rely on various approximations, in order of importance: (1) approximation of the score and guidance terms, (2) time-discretization of the diffusion and (3) approximate initialization of the diffusion. \citet{wu2023practical2023,cardoso2023diffusion} have used particle methods also known as Sequential Monte Carlo (SMC) \citep{Doucet:2001,chopin2020book} to obtain consistent estimates in the generative modeling context.

Our contributions are as follows:\\
(1) we adapt guided diffusions to sampling problems, \\
(2) we provide theoretical results quantifying the error introduced by current guided diffusions in a simple scenario,\\ (3) we develop an SMC scheme to provide consistent estimates in this setup and establish limit theorems,\\ (4) we introduce an algorithm that reduces the variance of the SMC estimates based on a novel score matching loss. \\
All proofs are postponed to the Appendix.

\section{Denoising Diffusions with Guidance} \label{sec:conditioned_diffusion}
\subsection{Noising and denoising diffusions}
\label{sec:diffusion_basic}
Consider the following noising diffusion $(X_t)_{t\in[0,T]}$,
\begin{equation}
	\label{eq:forward_diffusion}
	\dd X_t = -\beta_t X_t \dd t + \sqrt{2\beta_t} \dd W_t,\quad X_0 \sim \pi ,
\end{equation}
where $(W_t)_{t\in[0,T]}$ is a $d$-dimensional Brownian motion and $\beta_t>0$. %
The transition density of this diffusion is given by $p(x_t|x_0)=\mathcal{N}(x_t;\sqrt{1-\lambda_t}x_0,\lambda_t \mathrm{I})$ for $\textstyle \lambda_t=1- \exp[-2 \int_0^t \beta_s\dd s]$. We denote by $\pi_t$ the density of $X_t$ under (\ref{eq:forward_diffusion}). In practice, we consider $ \int_0^T \beta_s \dd s \gg 1$, and therefore $\pi_T(x) \approx \mathcal{N}(x;0,\mathrm{I})$. The diffusion (\ref{eq:forward_diffusion}) thus transforms $\pi_0=\pi$ into approximately $\mathcal{N}(0,\mathrm{I})$. If instead we initialize (\ref{eq:forward_diffusion}) using $p_0(x)= \mathcal{N}(x;0,\mathrm{I})$, its marginals $(p_t)_{t\in[0,T]}$ satisfy $p_t=p_0$.

The time-reversal  $(Y_t)_{t\in [0,T]}=(X_{T-t})_{t\in[0,T]}$ of (\ref{eq:forward_diffusion}), the denoising diffusion, satisfies $Y_T \sim \pi$ and
\begin{equation}\label{eq:time_reversal}
	\dd Y_t = \ps{\beta_{T-t} Y_t + 2\beta_{T-t} \gradlog \pi_{T-t}(Y_t)} \dd t +\sqrt{2\beta_{T-t}}\dd B_t ,
\end{equation}
where $(B_t)_{t\in [0,T]}$ is another Brownian motion; see e.g. \citet{haussmann1986time,cattiaux2021time}. The main idea of denoising diffusions is to sample from $\pi$ by sampling~\eqref{eq:time_reversal} as $Y_T \sim \pi$ \citep{ho2020denoising,song2020score}. However, we cannot simulate (\ref{eq:time_reversal}) exactly as, in order of importance, (1) the score terms $(\nabla \log \pi_t)_{t\in [0,T]}$ are intractable,  (2) it is necessary to time-discretize the diffusion and (3) $\pi_T$ cannot be sampled. We can always use numerical integrators and approximate $\pi_T$ with a unit Gaussian distribution to mitigate (2) and (3). In generative modeling (1) is addressed by leveraging tools from the score matching literature \cite{vincent2011connection,Hyvarinen:2005a} and using neural network estimators. In our \emph{sampling} setting we do not have access to access to samples from $\pi$ but only to its unnormalized density. Therefore, alternative approximations must be developed.

\subsection{Denoising diffusions with guidance}
\label{sec:guidance}
For generative modeling, the use of denoising diffusions with guidance terms to sample approximately from posterior distributions has become prominent in the inverse problem literature, see e.g. \citet{song2020score,chung2023diffusion,song2023pseudoinverseguided,corso2023particle}. We present here a simple extension of this idea applicable to any target $\pi(x)$ defined by~\eqref{eq:target} by the rewriting
\begin{equation}
	\label{eq:def_g0}
	\pi(x) = \frac{p_0(x) g_0(x)}{\mathcal Z}, \quad\text{for~~} p_0(x)=\mathcal{N}(x;0,\mathrm{I})
\end{equation}
where $g_0(x_0) = \gamma(x_0)/p_0(x_0)$.
\begin{lemma}\label{lemma:score}
	The following identities hold
	\begin{equation}\label{eq:twisted_densityandgrad_guidance}
		\pi_t(x_t) = \frac{p_0(x_t) g_t(x_t)}{\mathcal Z}; 	\nabla \log \pi_t(x_t) = -x_t+ \nabla \log g_t(x_t) ,
	\end{equation}
	where
	\begin{equation}
		\textstyle \label{eq:scorepotential}
		g_t(x_t)=\int g_0(x_0)p(x_0|x_t) \dd x_0 ,
	\end{equation}
	and  $p(x_0|x_t)=\mathcal{N}(x_0;\sqrt{1-\lambda_t}x_t,\lambda_t \mathrm{I})$ is the conditional density of $X_0$ given $X_t=x_t$ for the diffusion (\ref{eq:forward_diffusion}) initialized using $X_0\sim p_0$.
\end{lemma}

From Lemma \ref{lemma:score}, it follows that (\ref{eq:time_reversal}) can be rewritten as
\begin{align}
	\label{eq:conditioned_diffusion}
	\dd Y_t= & \ps{-\beta_{T-t}Y_t + 2\beta_{T-t}\gradlog g_{T-t}(Y_t)} \dd t \\
	         & +\sqrt{2\beta_{T-t}} \dd B_t. \nonumber
\end{align}
This is akin to having a diffusion model with tractable scores $\nabla \log p_t(x_t)=\nabla \log p_0(x_t)=-x_t$ and with \(g_t(x_t)\) as a guidance term.

We stress that this guidance formulation is simply a restatement of Section~\ref{sec:diffusion_basic} using some new notation. While it is possible to write all the sequel in terms of $\pi_t$ alone, introducing $g_t$ will make the exposition more intuitive.

\subsection{Guidance approximation}
\label{sec:naive_approximation}
The most important source of error when approximating \eqref{eq:conditioned_diffusion} is the lack of  a closed form expression for $g_t$ as it involves an intractable integral. In the context of inverse problems, a simple approximation used by \cite{chung2023diffusion,song2023pseudoinverseguided} is given by
\begin{align}
	\label{eq:naive_approximation}
	\textstyle g_t(x_t)\approx g_0 (\int x_0 p(x_0|x_t)\mathrm{d}x_0 ) & = \textstyle
	g_0(\sqrt{1-\lambda_t}x_t) \nonumber
	\\&:=\hat{g}_t(x_t).
\end{align}
This approximation is good when $t$ is close to $0$ or $T$ but crude otherwise, as established by the following result.

\newcommand{\revZ}{Z_t^{(T)}}
\begin{proposition}
	\label{lem:error_guidance}
	Let $\pi(x) = \mathcal N(x;\mu, \sigma^2)$ and $(\beta_t)_{t\in[0,T]}$ be any schedule satisfying $\lim_{T
			\rightarrow \infty} \int_0^T \beta_s \dd s = \infty$. Consider the following approximation of~\eqref{eq:conditioned_diffusion}
	\begin{align}
		\label{eq:naive_approximation_in_action}
		\dd \revZ & = [-\beta_{T-t} \revZ + 2\beta_{T-t} \gradlog \hat g_{T-t}(\revZ)] \dd t       \\
		          & + \sqrt{2\beta_{T-t}} \dd B_t,\quad  Z_0^{(T)} \sim \mathcal N(0, 1).\nonumber
	\end{align}
	Then $\lim_{T \to \infty} \E[Z_T^{(T)}]=\mu$ and  $\lim_{T \to \infty} \operatorname{Var}(Z_T^{(T)})=1$ for $\sigma=1$ and otherwise
	\begin{align}
		 & \lim_{T \to \infty} \E[Z_T^{(T)}]=\frac{\mu}{1-\sigma^2}(1 - e^{-(1/\sigma^2 - 1)}),                  \\
		 & \lim_{T \to \infty} \operatorname{Var}(Z_T^{(T)})=\frac{1 - e^{-2(1/\sigma^2-1)}}{2(1/\sigma^2 - 1)}.
	\end{align}

\end{proposition}
Hence, even without considering any time discretization error, $\lim_{T \to \infty} \E[Z_T^{(T)}] \neq \mu$ and $\lim_{T \to \infty} \operatorname{Var}(Z_T^{(T)}) \neq \sigma^2$ for $\sigma \neq 1$. If we consider the target $\mathcal N(\mu, \sigma^2)^{\otimes d}$, a similar result holds along each dimension. Therefore, the Kullback--Leibler divergence between the target and the result of running~\eqref{eq:naive_approximation_in_action} grows linearly with $d$. As a result an exponentially increasing number of samples is needed to obtain an importance sampling approximation of the target of reasonable relative variance \cite{chatterjee2018}.

\section{Particle Denoising Diffusion Sampler}\label{sec:Particle}
In this section, we propose a particle method to correct the discrepancy between the distribution outputted by the guided diffusion and the target. Let $[P]=\{1,...,P\}$ for $P\in\mathbb{N}$. We first present the exact joint distribution of $(X_{t_k})_{k\in \{0,...,K\}}$ with $t_k=k\delta$ for a fixed step size $\delta$ for the diffusion \eqref{eq:forward_diffusion} where $K=T/\delta$ is an integer. We then make explicit the corresponding reverse-time Markov transitions for this joint distribution and show how they can be approximated. Finally we show how these approximations can be corrected to obtain consistent estimates using SMC. Instead of writing $X_{t_k}$, we will write $X_k$ to simplify notation. Similarly we will write $g_k$ for $g_{t_k}$, $\lambda_k$ for $\lambda_{t_k}$ etc.
\subsection{From continuous time to discrete time}
For $k \in [K]$, let $\alpha_k = 1 - \exp[-2\int_{(k-1)\delta}^{k\delta} \beta_s \dd s]$.
The joint distribution of $X_{0:K}=(X_0,X_1,..., X_K)$ under \eqref{eq:forward_diffusion} satisfies
\begin{equation}\label{eq:DTnoising}
	\textstyle
	\pi(x_{0:K})=\pi(x_0) \prod\nolimits_{k\in[K]} p(x_k|x_{k-1}) ,
\end{equation}
for $p(x_k|x_{k-1})=\mathcal{N}(x_k;\sqrt{1-\alpha_k}x_{k-1},\alpha_k \mathrm{I})$. This implies in particular
\begin{equation}
	\label{eq:forward_discrete}
	\pi(x_k, x_{k+1}) = \pi_k(x_k) p(x_{k+1}|x_k).
\end{equation}
We also denote by $p(x_{0:K})$ the joint distribution of \eqref{eq:forward_diffusion} initialized at $p_0(x_0)$, in this case the Markov process is stationary, i.e. $p_k(x_k)=p_0(x_k)$, and reversible w.r.t.\ $p_0$.
From Bayes' theorem and equations \eqref{eq:twisted_densityandgrad_guidance} and~\eqref{eq:forward_discrete}, the backward transitions of (\ref{eq:DTnoising}) satisfy
\begin{align}
	 & \pi(x_k|x_{k+1})=\frac{\pi_k(x_k) p(x_{k+1}|x_k)}{\pi_{k+1}(x_{k+1})} \nonumber                                         \\
	 & = \frac{p_0(x_k) p(x_{k+1}|x_k)}{p_0(x_{k+1})} \frac{g_k(x_k)}{g_{k+1}(x_{k+1})}
	= \frac{p(x_{k}|x_{k+1}) g_k(x_k)}{g_{k+1}(x_{k+1})} \nonumber                                                             \\
	 & \approx p(x_k|x_{k+1}) \exp[\langle \gradlog g_{k+1}(x_{k+1}), x_k - x_{k+1}\rangle ] \nonumber                         \\
	 & = \textstyle \mathcal N(x_k; \sqrt{1-\alpha_{k+1}} x_{k+1} \nonumber                                                    \\
	 & \qquad + \alpha_{k+1} \gradlog g_{k+1}(x_{k+1}), \alpha_{k+1} \mathrm{I}) ,  \label{eq:approximate_proposal_derivation}
\end{align}
where we used $g_k \approx g_{k+1}$ and a Taylor expansion of $\log g_{k+1}(x_k)$ around $x_{k+1}$; both of which are reasonable when $\delta \ll 1$. Since $\alpha_{k+1} \approx 2 \beta_{k+1} \delta$ and $\sqrt{1-\alpha_{k+1}}\approx 1 -\beta_{k+1}\delta $, this approximation of $\pi(x_k|x_{k+1})$ corresponds to a discretization of the time-reversal~\eqref{eq:conditioned_diffusion}.

While this discrete-time representation is interesting, it is clearly typically impossible to exploit it to obtain exact samples from $\pi$ since, like its continuous time counterpart \eqref{eq:conditioned_diffusion}, it relies on the intractable potentials $(g_k)_{k=1}^K$. In the following \Cref{sec:sampling_via_particles}, given an approximation $\hat g_k$ of $g_k$, we propose a particle mechanism to sample exactly from $\pi$ as the number of particles goes to infinity.

\subsection{From discrete time to particles}
\label{sec:sampling_via_particles}

We use a particle method to sample from $\pi$. The key idea is to break the difficult problem of sampling from $\pi$ into a sequence of simpler intermediate sampling problems. Ideally we would sample from $\pi_K$ first then $\pi_{K-1}, \pi_{K-2},...$ until $\pi_0=\pi$. Unfortunately this is not possible as this requires knowing $g_k$. Suppose that we have an approximation $\hat g_k$ of $g_k$ (for instance, the guidance approximation defined in~\eqref{eq:naive_approximation}). Inspired by ~\eqref{eq:twisted_densityandgrad_guidance}, we sample instead for $k\in [K]$ from the sequence of densities
\begin{equation}\label{eq:marginalintermediatetarget}
	\textstyle
	\hat \pi_k(x_k)\propto p_0(x_k) \hat g_k(x_k),\quad \hat {\mathcal Z}_k=\int p_0(x_k) \hat g_k(x_k) \rmd x_k ,
\end{equation}
backward in time. At $k=K$, the function $g_K$ is almost constant so we choose $\hat{g}_K \equiv 1$ and sampling from $\hat \pi_K=\mathcal{N}(0,\mathrm{I})$ is easy.
Given particles approximately distributed according to $\hat\pi_{k+1}$, we aim to obtain samples from $\hat \pi_k$. Drawing inspiration from~\eqref{eq:approximate_proposal_derivation}, we sample according to the proposal
\begin{multline}\label{eq:smc_proposal}
	\hat \pi(x_k|x_{k+1}) := \mathcal N(x_k;\sqrt{1-\alpha_{k+1}} x_{k+1} \\+ \alpha_{k+1} \gradlog \hat g_{k+1}(x_{k+1}), \alpha_{k+1} \mathrm{I}).
\end{multline}
This is not the only option: we can also use the exponential integrator
\begin{multline*}
	\hat \pi(x_k|x_{k+1}) := \mathcal N(x_k;\sqrt{1-\alpha_{k+1}} x_{k+1} \\+ 2(1 - \sqrt{1-\alpha_{k+1}}) \gradlog \hat g_{k+1}(x_{k+1}), \alpha_{k+1} \mathrm{I}).
\end{multline*}
We could use instead the Euler integrator for~\eqref{eq:conditioned_diffusion} but it is clear that the latter would induce greater error.
We then reweight the pairs $(x_k, x_{k+1})$ using the weights
\begin{align}
	w_k(x_{k}, & x_{k+1}) := \frac{\hat \pi_k(x_k) p(x_{k+1}|x_k)}{\hat \pi_{k+1}(x_{k+1}) \hat \pi(x_k|x_{k+1})} \nonumber                       \\
	           & \propto \frac{\hat g_k(x_k)}{\hat g_{k+1}(x_{k+1}) \hat \pi(x_k|x_{k+1})} \frac{p_0(x_k) p(x_{k+1}|x_k)}{p_0(x_{k+1})} \nonumber \\
	           & =\frac{\hat g_k(x_k) p(x_k|x_{k+1})}{\hat g_{k+1}(x_{k+1}) \hat \pi(x_k|x_{k+1})} , \label{eq:importanceweights}
\end{align}
where the second line follows from~\eqref{eq:marginalintermediatetarget}.
The first line of~\eqref{eq:importanceweights} means that the $x_k$ marginal of the weighted system consistently approximates $\hat \pi_k$. It should be noted that $w_k$ quantifies the error in \eqref{eq:smc_proposal}.

Finally, we resample these particles with weights proportional to \eqref{eq:importanceweights}. This resampling operation allows us to focus the computational efforts on promising regions of the space but some particles are replicated multiple times, reducing the population diversity. Therefore, we then optionally perturb the resampled particles using a MCMC kernel of invariant distribution $\hat \pi_k$. The resulting Particle denoising diffusion sampler (PDDS) is summarized in Algorithm \ref{algo:smc_dds}.

\newcommand{\hmc}{\hat {\mathcal Z}}
\begin{algorithm}[tb]
	\caption{Particle Denoising Diffusion Sampler}
	\label{algo:smc_dds}
	\begin{algorithmic}
		\REQUIRE Schedule $(\beta_t)_{t \in [0,T]}$ as in~\eqref{eq:forward_diffusion}; Approximations $(\hat g_k)_{k=0}^K$ s.t. $\hat g_0=g_0, \hat g_K = 1$;
		Number of particles $N$
		\STATE Sample $X_K^i \overset{\mathrm{iid}}{\sim} \mathcal N(0, \mathrm{I})$ for $i \in [N]$
		\STATE Set $\hat {\mathcal Z}_K \gets 1$ and $\omega_K^i \gets 1/N$ for $i \in [N]$
		\FOR{$k = K-1, \ldots, 0$}
		\STATE \underline{Move}. Sample $\tilde X_k^i \sim \hat \pi(\cdot|X_{k+1}^i)$ for $i \in [N]$ (see~\eqref{eq:smc_proposal})
		\STATE \underline{Weight}. $\omega_k^i \gets \omega_k(\tilde X_k^i, X_{k+1}^i)$ for $i \in [N]$ (see~\eqref{eq:importanceweights})
		\STATE Set $\hmc_k \gets \hmc_{k+1} \times \frac 1N \sum_{i\in[N]} \omega_k^i$
		\STATE Normalize $\omega_k^i \gets \omega_k^i/\sum_{j\in[N]} \omega_k^j$
		\STATE \underline{Resample}. $X_k^{1:N} \gets \operatorname{resample}(\tilde X_k^{1:N}, \omega_k^{1:N})$ (see Section~\ref{sec:algo_settings})
		\STATE \underline{MCMC} (Optional). Sample $X_k^i \gets \mathfrak M_k(X_k^i, \cdot)$ for $i\in[N]$ using a $\hat \pi_k$-invariant MCMC kernel $\mathfrak M_k$. %
		\ENDFOR
		\ENSURE Estimates $\hat \pi^N = \frac 1N \sum_{i\in[N]} \delta_{X_0^i}$ of $\pi$, $\hmc^N_0$ of $\mathcal Z$
	\end{algorithmic}
\end{algorithm}

\subsection{Algorithm settings}
\label{sec:algo_settings}
\textbf{Reparameterization.} In practice, we would like to have $p_0$ to be such that $g_0$ is bounded or has bounded moments. To achieve this, it can be desirable to obtain a variational approximation $\mathcal{N}(x;\mu,\Sigma)$ of $\pi$ then do the change of variables $x'=\Sigma^{-1/2}(x-\mu)$, sample in this space using PDDS before mapping the samples back using $x=\mu+\Sigma^{1/2}x'$. %

\textbf{Resampling.} The idea of resampling is to only propagate  particles in promising regions of the state space.
Given $N$ particles $\tilde X_k^{1:N}$ and $N$ weights $\omega_k^{1:N}$ summing to $1$, resampling selects $N$ output particles $X_k^{1:N}$ such that, for any function $\varphi: \mathbb R^d \to \mathbb R$,
\begin{equation*}
	\mathbb{E}\Bigl[  \tfrac 1N {\textstyle\sum}_{i\in[N]}\varphi(X_k^i) |\tilde X_k^{1:N}, \omega_k^{1:N}\Bigr] = {\textstyle\sum}_{i\in[N]} \omega_k^i \varphi(\tilde X_k^i).
\end{equation*}
Popular schemes satisfying this identity are multinomial, stratified, residual, and systematic resampling \citep{douc2005comparison}.
We employ systematic resampling in all our simulations as it provides the lowest variance estimates.

Resampling can however reduce particle diversity by introducing identical particles in the output. As such, a popular recipe is to trigger resampling at time $k$ only when the Effective Sample Size (ESS), a measure of particle diversity defined by $(\sum_{i\in[N]} (\omega ^i_k)^2)^{-1}$, is below a certain threshold \citep{Del-Moral:2012,dai2022invitation}. This is implemented using Algorithm \ref{algo:smc_dds_adaptive} presented in Appendix~\ref{apx:adaptive_resampling}.

\textbf{MCMC kernel.}
We want to design an MCMC kernel of invariant distribution $\hat \pi_k(x_k)$ defined in \eqref{eq:marginalintermediatetarget}. As we have access to $\nabla \log \hat\pi_k(x_k)=-x_k+ \nabla \log \hat{g}_k(x_k)$, we can use a Metropolis-adjusted Langevin algorithm (MALA) or Hamiltonian Monte Carlo; e.g. MALA considers a proposal
\begin{equation*}
	\textstyle
	x^\star_k=x_k+ \gamma \nabla \log \hat\pi_k(x_k)+ \sqrt{2\gamma} \epsilon,\quad \epsilon \sim \mathcal{N}(0,\mathrm{I}) ,
\end{equation*}
for a step size $\gamma$. This proposal is accepted with probability
\begin{equation*}
	\min\biggl\{1,\frac{\hat \pi_k(x_k^\star)\mathcal{N}(x_k;x^\star_k+\gamma \nabla \log \hat \pi_k(x^\star_k),2 \gamma \mathrm{I})}{\hat\pi_k(x_k)\mathcal{N}(x^\star_k;x_k+\gamma \nabla \log \hat\pi_k(x_k),2\gamma \mathrm{I})} \biggr\}.
\end{equation*}

\subsection{Theoretical results}
\label{sec:theory}
\paragraph{Fixed number of discretization steps.} We show below that the estimates $\hmc^N_0$ and $\pi^N f=\tfrac{1}{N}\sum_{i=1}^N f(X^i_0)$ of Algorithm~\ref{algo:smc_dds} satisfy a central limit theorem. This follows from standard SMC theory \citep{del2004feynman,webber2019unifying}.
\begin{proposition}
	\label{prop:standard_smc}
	Assume that $\E[w_k(X_k, X_{k+1})^2] < \infty$ where the expectation is w.r.t. $\hat \pi(x_{k+1}) \hat\pi(x_k|x_{k+1})$ and that $\int \pi_k(x) (g_k/\hat g_k)(x) \rmd x < \infty$. Then $\hmc^N_0$ is an unbiased estimate of $\mathcal Z$ and has finite variance. If multinomial resampling is used at every step, $\sqrt N(\hmc^N_0 /\mathcal Z - 1)$ is asymptotically normal with asymptotic variance
	\begin{align*}
		 & \sigma_K^2 = \chisquared{\pi_K}{\hat\pi_K} +                                                          \\
		 & \sum_{k=0}^{K-1} \chisquared{\pi_k(x_k)\pi(x_k|x_{k+1})}{\hat \pi_{k}(x_{k}) \hat \pi(x_k|x_{k+1})} ,
	\end{align*}
	with $\chisquared{\cdot}{\cdot}$ the chi-squared divergence between two distributions. Moreover, for any bounded function $f$, we also have asymptotic normality of $\sqrt N(\hat{\pi}^N f - \pi f)$.
\end{proposition}

The finiteness assumptions on $\E[w_k(X_k, X_{k+1})^2]$ and $\int \pi_k(x) (g_k/\hat g_k)(x) \rmd x$ require that $\hat g_k$ is not too far from $g_k$ but still allow enough freedom in the choice of $\hat g_k$. The expression for $\sigma_K^2$ suggests that choosing $\hat g_k$ close to $g_k$ will reduce the asymptotic variance as \eqref{eq:smc_proposal} will better approximate~\eqref{eq:approximate_proposal_derivation}.

\paragraph{Infinitely fine discretization limit.} We now investigate the performance of PDDS as the number of discretization time steps goes to infinity. Even when all the weights are equal, multinomial resampling still kills over a third of particles on average \citep{chopin2020book}. Hence a fine discretization with repeated applications of multinomial resampling leads to the total collapse of the particle approximation. This has been formalized in the continuous-time setting \citep{chopin2022resampling}. In our case, it can be readily checked that $\sigma_K^2 \to \infty$ as $K \to \infty$ in general. This justifies using a more sophisticated resampling strategy. Indeed, when using the sorted stratified resampling strategy of \citet{gerber2019negative}, the following results show that our particle approximations remain well behaved as $K \rightarrow \infty$.
\newcommand{\ssK}{\sigma_{K*}}
\begin{proposition}
	\label{prop:smc_sorted}
	Consider the setting of Proposition~\ref{prop:standard_smc} with sorted stratified resampling. Then there exists a sequence of sets $(B_N)$ such that $\P(B_N) \to 1$ and
	\begin{equation*}
		\limsup_{N \to \infty} \E[N(\hmc^N_0/\mathcal Z - 1)^2 \mathbbm{1}_{B_N}] \leq \zeta^2_K %
	\end{equation*}
	with
	\begin{align*}
		 & \zeta^2_K:= \chisquared{\pi_K}{\hat \pi_K}+                                                                                                \\
		 & \sum_{k=0}^{K-1} \int\frac{\pi_{k+1}(x_{k+1})^2}{\hat\pi_{k+1}(x_{k+1})} \chisquared{\pi(x_k|x_{k+1})}{\hat \pi(x_k|x_{k+1})} \dd x_{k+1}.
	\end{align*}
\end{proposition}
The next result bounds the limit when $K \to\infty$, i.e. $\delta=T/K \rightarrow 0$
\newcommand{\barhat}[1]{\bar{\hat #1}}
\begin{proposition}
	\label{prop:error_K_infinite}
	Under the setting of Proposition~\ref{prop:smc_sorted}, assume that $\beta_t \equiv 1$ and that the target distribution satisfies the regularity conditions in Appendix~\ref{apx:regularity_conditions}. Let $\bar g_t:= g_t/\mathcal Z$ and $\tilde g_t:= \hat g_t/\int p_0(x) \hat g_t(x) \dd x$. Assume further that the approximations $\hat g_t, \bar g_t, \tilde g_t$ satisfy $C_1^{-1} \leq \bar g_t/\tilde g_t \leq C_1$ and $\norm{ \gradlog \hat g_t(x_t)} \leq C_2(1 +\norm{x_t})$ for some $C_1, C_2 \geq 0$. Then
	\begin{align*}
		 & \textstyle \limsup_{K\to\infty} \zeta^2_K \leq \chisquared{\pi_T}{\hat \pi_T} +
		\\ & \textstyle 2\int_0^T \E_{X_t \sim \pi_t}\ps{\frac{\bar g_t}{\tilde g_t}(X_t)\norm{\gradlog g_t(X_t) - \gradlog \hat g_t(X_t)}^2} \dd t.
	\end{align*}
\end{proposition}
\newcommand{\pihatg}{\pi^{\hat g}}
This result highlights precisely how the asymptotic error depends on the quality of the estimates of $g_t$ and its logarithmic gradient. Let $\pihatg(\dd x_{[0,T]})$ be the path measure induced by running~\eqref{eq:conditioned_diffusion} with some approximation $\hat g_t$, i.e. the distribution of $(Y_t)_{t \in [0,T]}$ given by \eqref{eq:conditioned_diffusion}. If importance sampling were directly used to correct between $\pihatg(\dd x_{[0,T]})$ and $\pi(\dd x_{[0,T]})$, the asymptotic error for $\hmc^N_0/\mathcal Z$ would be equal to $\chisquared{\pi}{\pihatg}$ which is greater than $\exp{\KLLigne{\pi}{\pihatg}} - 1$. In contrast, ignoring the negligible first term $\chisquared{\pi_T}{\hat \pi_T}$, the error in this proposition is upper bounded by $2C_1 \KLLigne{\pi}{\pihatg}$. This shows how PDDS helps reduce the error of naive importance sampling.

\section{Learning Potentials via Score Matching}\label{sec:IterativeParticle}

The previous theoretical results establish the consistency of PDDS estimates for any reasonable approximation $\hat g_k$ of $g_k$ but also show that better approximations lead to lower Monte Carlo errors. We show here how to use the approximation of $\pi$ outputted by PDDS to learn a better neural network (NN) approximation $\hat g_\theta(k, x_k)$ of the potential functions $g_k$ by leveraging score matching ideas. Once we have learned those approximations, we can then run again PDDS (Algorithm \ref{algo:smc_dds}) with the new learned potentials. %

\subsection{Different score identities}
\label{sec:losses}
We follow the notation outlined at the beginning of Section~\ref{sec:Particle}.
From \eqref{eq:twisted_densityandgrad_guidance}, we have the identity
\begin{equation}\label{sec:potentialscore}
	\nabla \log g_k(x_k)=x_k+\nabla \log \pi_k(x_k).
\end{equation}
Hence, if we obtain an approximation $\nabla \log \pi_\theta(k,x_k)$ of $\nabla \log \pi_k(x_k)$, then we get an approximation $\log \hat g_\theta(k,x_k)= \tfrac{1}{2}||x_k||^2 +\log \hat{\pi}_\theta(k,x_k)$ of $\log g_k(x_k)$.

To learn the score, we rely on the following result.
\begin{proposition}\label{prop:scorematching}
	The score satisfies the standard Denoising Score Matching (DSM) identity
	\begin{equation}\label{eq:DSM}
		\textstyle  \gradlog \pi_k(x_k) = \int \nabla \log p(x_k|x_0)~ \pi(x_0|x_k) \dd x_0.
	\end{equation}
	Moreover, if $\int\norm{\nabla \pi(x_0)} e^{-\eta\norm{x_0}^2} \dd x_0 < \infty, \forall \eta>0$, the Novel Score Matching (NSM) identity
	\begin{equation}\label{eq:VDB}
		\textstyle  \gradlog \pi_k(x_k) =\kappa_k \int \gradlog  g_0(x_0)~ \pi(x_0|x_k) \dd x_0 -x_k
	\end{equation}
	holds for $\kappa_k=\sqrt{1-\lambda_k}$ and $\gradlog  g_0(x_0)=\nabla \log \pi_0(x_0)+x_0$.
	Hence, we can approximate the score by minimizing one of the two following loss functions
	\begin{align*}
		\textstyle & \ell_{\textup{DSM}}(\theta)=\sum_{k\in[K]} \E\norm{\nabla \log \hat \pi_{\theta}(k, X_k) - \nabla \log p(X_k|X_0)}^2, \\
		\textstyle & \ell_{\textup{NSM}}(\theta)=\sum_{k\in[K]} \E||\nabla \log \hat \pi_{\theta}(k, X_k) + X_k                            \\
		           & \hspace{2.5cm} -\kappa_k \gradlog g_0(X_0)||^2,
	\end{align*}
	where the first loss is applicable if $\pi$ has finite second moment and the second loss is applicable if additionally $\E_{\pi}[\norm{\gradlog \pi(X)}^2] < \infty$.
	All the expectations are taken w.r.t.\ $\pi(x_0)p(x_k|x_0)$. For expressive neural networks, both $\ell_{\textup{DSM}}(\theta)$ and $\ell_{\textup{NSM}}(\theta)$ are such that $\nabla \log \hat \pi_{\theta}(k, x_k)=\nabla \log \pi_k(x_k)$ at the minimizer.
\end{proposition}

The benefit of this novel loss is that it is much better behaved compared to the standard denoising score matching loss when $\delta \ll 1$ as established below, this is due to the fact that the variance of the terms appearing in $\ell_{\textup{DSM}}$ for $k$ close to zero become very large. This loss is not applicable to generative modeling as $\pi$ is only available through samples.

\newcommand{\ggrad}{\mathfrak g}
\newcommand{\IA}{I_{\mathrm{A}}}
\newcommand{\IB}{I_{\mathrm{B}}}
\newcommand{\gga}{\ggrad^{\mathrm{A}}}
\newcommand{\ggb}{\ggrad^{\mathrm{B}}}
\subsection{Benefits of alternative score matching identity}
We establish here formally the benefits of NSM over DSM. NSM score matching prevents variance blow up as $k$ is close to time $0$ and $\delta \rightarrow 0$. This is more elegantly formalized in continuous-time. In this case, the renormalized loss functions $\ell_{\textup{DSM}}$ and $\ell_{\textup{NSM}}$ introduced in Proposition \ref{prop:scorematching} become for $s_\theta(t,x_t)=\nabla \log \hat \pi_\theta(t,x_t)$
\begin{align*}
	 & \textstyle \ell_{\textup{DSM}}(\theta)= \textstyle \int^T_0 \E\norm{s_\theta(t, X_t) - \gradlog p(X_t|X_0)}^2 \dd t,          \\
	 & \textstyle \ell_{\textup{NSM}}(\theta)= \textstyle \int^T_0 \E||s_\theta(t, X_t)+X_t - \kappa_t \gradlog g_0(X_0) ||^2 \dd t,
\end{align*}
where the expectations are w.r.t. $\pi(x_0)p(x_t|x_0)$.
Unbiased estimates of the gradient of these losses $\hat {\nabla \ell}_{\textup{DSM}}(\theta)$ and $\hat {\nabla \ell}_{\textup{NSM}}(\theta)$ are obtained by computing the gradient w.r.t. $\theta$ of the argument within the expectation at a sample $\tau \sim \operatorname{Unif}[0, T]$, $X_0 \sim \pi$ and $X_\tau \sim p(x_\tau|X_0)$.

The following result clarifies the advantage of NSM score matching.
\begin{proposition}
	\label{prop:guidance_loss_is_good}
	Let $d=1$, $\beta_t \equiv 1$, and $\pi(x_0) = \mathcal N(x_0;\mu, \sigma^2)$. Suppose that the score network $s_\theta(t, x_t)$ is a continuously differentiable function jointly on its three variables. Moreover, assume that $\abs{s} + \norm{\nabla_\theta s} \leq C(1 + |\theta|+ |t| + |x_t|)^{\alpha}$ for some $C, \alpha > 0$; and that $\E_\pi[\norm{\nabla_\theta s_\theta(0, X_0)}^2] > 0$ for all $\theta$. Then, for all $\theta$, $\ell_{\textup{DSM}}(\theta)$ and $\E[\hat {||\nabla \ell}_{\textup{DSM}}(\theta)||^2]$ are infinite whereas $\ell_{\textup{NSM}}(\theta)$ and $\E[||\hat {\nabla \ell}_{\textup{NSM}}(\theta)||^2]$ are finite.
\end{proposition}

\newcommand{\hgtheta}{\hat g_\theta}

\subsection{Neural network parametrization}
\label{sec:nn_param}
Contrary to standard practice for generative modeling \cite{ho2020denoising,song2020score}, we parameterize a function and not a vector field. This is necessary to run PDDS (Algorithm \ref{algo:smc_dds}) as it relies on potentials to weight particles.

We will use a parameterization of the form $\log \hat \pi_\theta(k,x_k)=\log \hat g_\theta(k,x_k)-\tfrac{1}{2}||x||^2$ where we parametrize $\hat g_\theta$ using two neural networks $r_\eta$ and $N_\gamma$ such that $\theta = (\eta, \gamma)$. The network $r_\eta$ returns a scalar whereas $N_\gamma$ returns a vector in $\mathbb R^d$. The precise expression is given by
\begin{multline*}
	\log \hat g_\theta(k, x_k) = \left[r_\eta(k) - r_\eta(0) \right] \langle N_\gamma(k, x_k), x_k \rangle +\\+ \left[ 1 - r_\eta(k) + r_\eta(0) \right] \log g_0(\sqrt{1 - \lambda_k} x_k).
\end{multline*}
This parametrization takes advantage of the simple approximation~\eqref{eq:naive_approximation} with which it coincides at time $0$. \citet{zhangyongxinchen2021path} used a similar parametrization to incorporate gradient information in their control policy. Although $\log \hat g_\theta(k, x_k)$ is a scalar, we deliberately let $N_\gamma(k, x_k)$ return a vector which is then scalar-multiplied with $x_k$. This is usually done in the literature when the scalar potential instead of just its gradient is learned (see e.g. \citealp{salimans2021should}) and helps improve model expressiveness.

\subsection{Training the neural network}
\label{subsect:iterative}

Both $\ell_{\textup{DSM}}(\theta)$ and $\ell_{\textup{NSM}}(\theta)$ can be written as $\sum_{k\in[K]} \E\ps{\ell(\theta,k, X_0, X_k)}$ for appropriately defined local loss functions $\ell$. Algorithm~\ref{algo:learn_theta} describes the gradient updates according to the batch size $B$.

\newcommand{\Nupdate}{N_{\mathrm{up}}}
\begin{algorithm}[tb]
	\caption{Potential neural network training}
	\label{algo:learn_theta}
	\begin{algorithmic}
		\REQUIRE Particle approx. $\hat{\pi}^N$ outputted by Algorithm~\ref{algo:smc_dds}; Potential NN $\hat g_k(\theta, x_k)$; Initialization $\theta_0$; Local loss functions $\ell(\theta, k, x_0, x_k)$; Batch size $B$; Number of gradient updates $\Nupdate$; Learning rate $\eta>0$.
		\FOR{$i = 1, 2, \ldots, \Nupdate$}
		\FOR{$b = 1, 2, \ldots, B$}
		\STATE Sample $X^{b}_0 \sim \hat{\pi}^N(\cdot)$; $k_b \sim \operatorname{Unif}[K]$;
		\STATE Sample $X_{k_b} \sim p_{k_b, 0}(\cdot|X_0^{b})$
		\ENDFOR
		\STATE $\theta_i := \theta_{i-1} -  \tfrac{\eta}{B}\sum_{b\in[B]} \nabla_{\theta}\ell(\theta_{i-1},k_b, X_0^{b}, X_{k_b})$
		\ENDFOR
		\ENSURE Potential $\hat g_{\theta_{\Nupdate}}(k, x_k)$
	\end{algorithmic}
\end{algorithm}
In practice, we first run Algorithm~\ref{algo:smc_dds} with a simple approximation such as \eqref{eq:naive_approximation}. We then use Algorithm~\ref{algo:learn_theta} to learn $\hat g_\theta(k, x_k)$ and can then execute Algorithm~\ref{algo:smc_dds} again with $\hat g_k(x_k)=\hat g_\theta(k, x_k)$ to obtain lower variance estimates of $\pi$ and $\mathcal{Z}$. We can further refine the approximation of $g_k$ using Algorithm~\ref{algo:learn_theta}. In practice we found that one or two iterations with larger $N_{\text{up}}$ are sufficient, although more frequent iterations  with smaller $N_{\text{up}}$ can give better performance under a limited budget of target density evaluations.

\subsection{Mechanisms behind potential improvement}
\label{sec:potential_iterations}
It could be unclear at first sight how the iterative procedure described in Section~\ref{subsect:iterative} leads to an improvement of $\hat g_k(x_k)$. Superficially, it seems like we try to improve a poor potential approximation using particles produced by the poor approximation itself. However, results in Section~\ref{sec:theory} imply that the SMC mechanism provides a consistent estimate of the target for \textit{any} approximation of the potential. Thus, we expect that for $N$ large enough, the output of Algorithm~\ref{algo:smc_dds} would have higher quality than the particles used to learn the current $\hat g_k$. This mechanism has been studied in a different setting in \citet{heng2020controlled} and it would be interesting to extend their results to our case.

Moreover, the training process uses not only the output of Algorithm~\ref{algo:smc_dds}, but also further information about the target injected via a variety of mechanisms (the guidance loss described in Section~\ref{sec:losses} and the neural network parametrization described in Section~\ref{sec:nn_param}). Quantifying the gain from these techniques is an open question. We provide ablation studies in \cref{app:additional_ablations}.

\section{Related Work}\label{sec:Related}
SMC samplers \citep{del2006sequential,dai2022invitation} are a general methodology to sample sequentially from a sequence of distributions. They rely on the notion of forward and backward kernels in order to move from one distribution to another. PDDS can be cast in this framework where the forward kernel is chosen to be the forward noising diffusion and the backward kernel the approximate time-reversal. The novelty of our work is the exploitation of the special structure induced by such a choice to come up with efficient backward kernel estimates. Other standard methods include Annealed Importance Sampling \citep{neal2001annealed} and Parallel Tempering \cite{geyer1991markov}. Unlike our work, all these methods approximate a sequence of tempered versions of the target. While they are standard, it is also well-known that tempering strategies can exhibit poor performance for multimodal targets \citep{woodard2009sufficient,tawn2020weight,syed2022non} as tempering can change dramatically the masses of distribution modes depending on their widths. Adding noise does not suffer from this issue \citep{mate2023learning}. It only perturbs the distribution locally, preserving the weights of the modes; see Figure~\ref{fig:tempering_vs_noising} for an illustration.

\begin{figure}
	\centering
	\includegraphics[width=\linewidth]{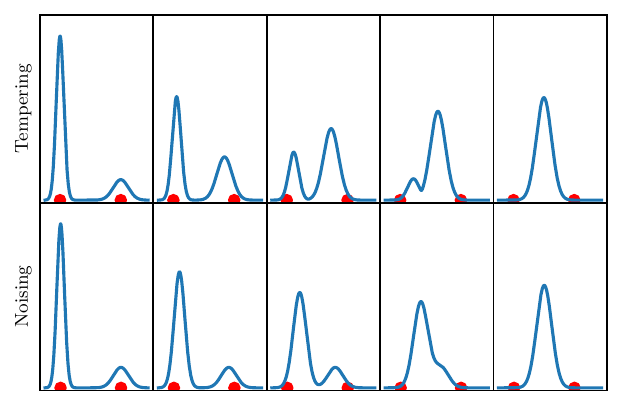}

	\caption{Tempered (top) and noised (bottom) sequences of distributions for the target $\pi(x) = 0.8 \mathcal N(x; -4, 0.5^2) + 0.2 \mathcal N(x; 4, 1)$. The tempered sequence follows $\pi_t(x) \propto \pi(x)^{(1-\eta_t)} \phi(x)^{\eta_t}$ where $\phi$ is the standard normal and $\eta_t$ increases from $0$ to $1$. The noising sequence follows the forward diffusion in \cref{eq:forward_diffusion}. Red dots indicate the position of modes in the original target. The tempered sequence suffers from mode switching, i.e. the low mass large width mode becomes dominant across the tempered path. The noised sequence does not suffer from this problem.}
	\label{fig:tempering_vs_noising}
\end{figure}

Our sampler can also be interpreted as sampling from a sequence of distributions using so-called twisted proposals. The general framework for approximating these twisted proposals using SMC has been considered in \citet{guarniero2017iterated,heng2020controlled,lawson2022sixo}. In particular, \citet{wu2023practical2023} and \citet{cardoso2023diffusion} recently apply such ideas to conditional simulation in generative modeling given access to a pretrained score network. \citet{wu2023practical2023} rely on the simple approximation (\ref{eq:naive_approximation}) and do not quantify its error, while \citet{cardoso2023diffusion} use a different proposal kernel which leverages the structure of a linear inverse problem. Our setup is here different as we consider general Monte Carlo sampling problems. We do not rely on any pretrained network and refine our potential approximations using an original loss function. We additionally provide theoretical results in particular when the discretization time step goes to zero.

\section{Experimental Results}
\label{sec:experimental}

\subsection{Normalizing constant estimation}

\begin{figure}
	\centering
	\includegraphics[width=\linewidth]{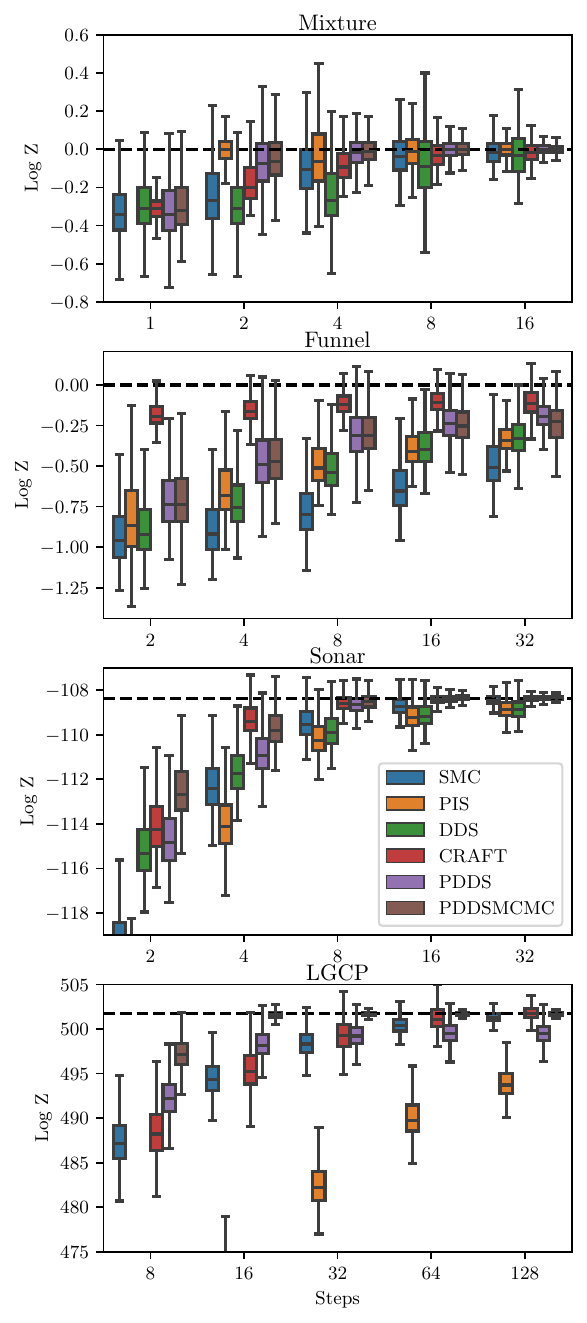}
	\caption{$\log \hmc^N_0$ for our method (PDDS and PDDS-MCMC), compared with SMC, CRAFT, DDS and PIS. Dotted black represents analytic ground truth where available, otherwise long-run SMC. Variation is displayed over both training and sampling seeds (2000 total). The y-axes on Sonar and LGCP have been cropped and outliers (present in all methods) removed for clarity. Uncurated samples are presented in \cref{app:additional_results}.}
	\label{fig:logZ_boxplots}
\end{figure}

\begin{figure}
	\centering
	\includegraphics[width=\linewidth]{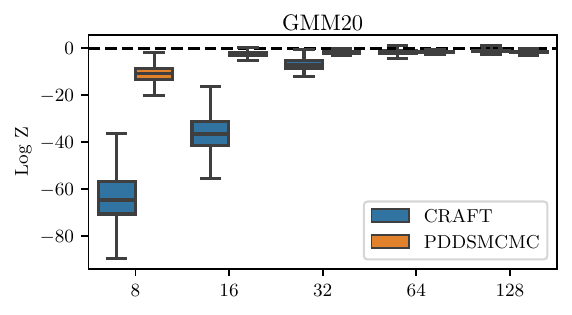}
	\caption{$\log \hmc^N_0$ for CRAFT and PDDS-MCMC on the GMM task in 20 dimensions. Variation displayed over training and sampling seeds (1000 total).}
	\label{fig:gmm_logZ}
\end{figure}

We evaluate the quality of normalizing constant estimates produced by PDDS on a variety of sampling tasks. We consider two synthetic target densities and two posterior distributions, with results on additional targets included in \cref{app:additional_results}. The synthetic targets are \verb|Mixture|, a 2-dimensional mixture of Gaussian distributions with separated modes \citep{arbel2021annealed} and \verb|Funnel|, a 10-dimensional target displaying challenging variance structure \citep{Neal:2003}. The posterior distributions are \verb|Sonar|, a logistic regression posterior fitted to the Sonar (61-dimensional) dataset and \verb|LGCP|, a log Gaussian Cox process \citep{Moller:1998} modelling the rate parameter of a Poisson point process on a $40\times40=1600$-point grid, fitted to the Pines dataset. Precise specification of these targets can be found in \cref{apx:exp_tasks}.

We compare our performance to a selection of strong baselines. We consider two annealing algorithms, firstly an SMC sampler \citep{del2006sequential} with HMC kernels and secondly CRAFT \citep{Matthews2022}, which uses normalizing flows to transport particles at each step of an SMC algorithm. We also consider two diffusion-based sampling methods, Path Integral Sampler (PIS) \citep{zhangyongxinchen2021path} and Denoising Diffusion Sampler (DDS) \citep{vargasDDSampler2023}. As mentioned in \cref{sec:algo_settings}, we reparameterize the target using a variational approximation for all methods. We include hyperparameter and optimizer settings, run times, and experimental procedures in \cref{apx:exp_settings}.

We present the normalizing constant estimation results in \cref{fig:logZ_boxplots}. PDDS uses the same training budget as PIS and DDS. We also include PDDS with optional MCMC steps (PDDS-MCMC). Considering the posterior sampling tasks, PDDS-MCMC is the best performing method in terms of estimation bias and variance everywhere, except for 4 steps on the \verb|Sonar| task where CRAFT performs the best. PDDS without MCMC steps performs on par with CRAFT on average, specifically PDDS outperforms CRAFT for larger step regimes on \verb|Sonar| and low step regimes on \verb|LGCP| while CRAFT outperforms PDDS in the opposite regimes. Both PDDS methods uniformly outperform the diffusion-based approaches (PIS and DDS). Considering the synthetic target densities, CRAFT is the best performing method on the synthetic \verb|Funnel| task while PDDS and PDDS-MCMC are the best performing methods on \verb|Mixture|. Our approach again outperforms both of the diffusion-based approaches.

While CRAFT performs competitively with our method on certain tasks, we note that CRAFT cannot be easily refined. Indeed, if we want more intermediate distributions between the reference and the target, we can simply decrease the discretization step size for PDDS, but would need to relearn the flows for CRAFT. In addition, choosing the flow structure in CRAFT can be challenging and is problem-specific. In contrast, we used the same simple MLP structure for all tasks for PDDS. Finally, CRAFT relies on MCMC moves to inject noise into the system and prevent particle degeneracy. On the other hand PDDS can produce competitive results without MCMC, with the option of boosting performance with MCMC steps if the computational budget allows.

We further compared PDDS-MCMC with CRAFT on a challenging Gaussian Mixture Model (\verb|GMM|) with 40 highly separated modes \citep{midgley2022flow} in a range of dimensions. We present the normalising constant estimates for the task in 20 dimensions in \cref{fig:gmm_logZ}. We observe that PDDS-MCMC significantly outperforms CRAFT in bias and variance, particularly in low step regimes. Furthermore, PIS and DDS failed to produce competitive results. Results in additional dimensions are included in \cref{app:additional_results}.

\subsection{Sample quality}

We also visually assess the quality of samples from each method in \cref{fig:samples}. We choose the multi-modal \verb|Mixture| task in 2-dimensions for ease of visualization. Unsurprisingly we find that both PIS and DDS do not capture all 6 modes of the distribution due to the mode-seeking behaviour of the reverse KL objective. Samples from our approach appear of similar quality to those of SMC and CRAFT.

We further quantitatively compared the sample quality of our method versus CRAFT by evaluating the entropy-regularized Wasserstein-2 distance between samples from the model and the target in the challenging \verb|GMM| task. We present the results in \cref{fig:gmm_sample_quality} for the task in 20 dimensions. Here PDDS-MCMC clearly outperforms CRAFT producing samples with significantly lower transport cost to the target distribution. From the sample visualization we observe that PDDS-MCMC is able to correctly recover a far greater proportion of the target modes than our CRAFT implementation.

\begin{figure}[t]
	\centering
	\includegraphics[width=\linewidth]{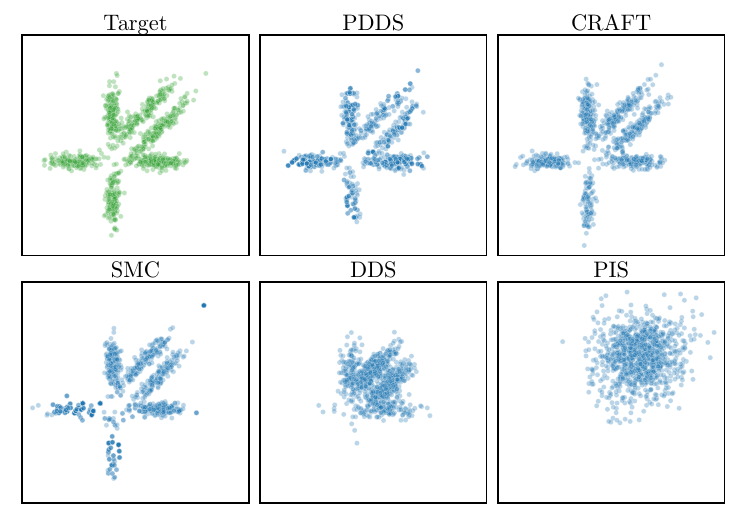}
	\caption{Samples from each method on the Gaussian Mixture task with 4 steps.}
	\label{fig:samples}
\end{figure}

\begin{figure}[h]
	\centering
	\includegraphics[width=0.8\linewidth]{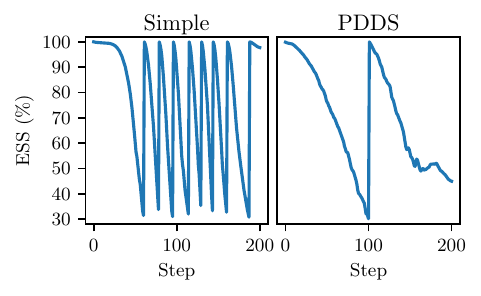}
	\caption{ESS curves on the Sonar task with 200 steps. Left: PDDS with approximation \cref{eq:naive_approximation}, right: PDDS with learnt potential approximation.}
	\label{fig:ess_curves}
\end{figure}

\begin{figure}[t]
	\centering
	\includegraphics[width=\linewidth]{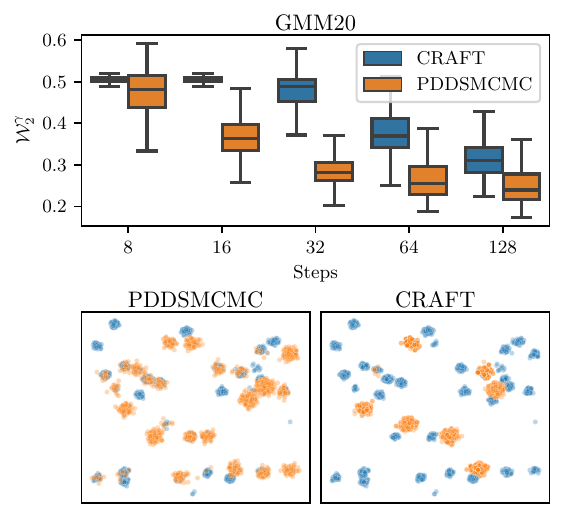}
	\caption{Top: $\mathcal{W}_2^\gamma$ distance between samples from CRAFT and PDDS-MCMC and from GMM20. Variation displayed over training and sampling seeds (200 total). Bottom: samples (first two dimensions) from PDDS-MCMC and CRAFT (orange) using 32 steps versus GMM20 target (blue).}
	\label{fig:gmm_sample_quality}
\end{figure}

\begin{figure}[h!]
	\centering
	\includegraphics[width=\linewidth]{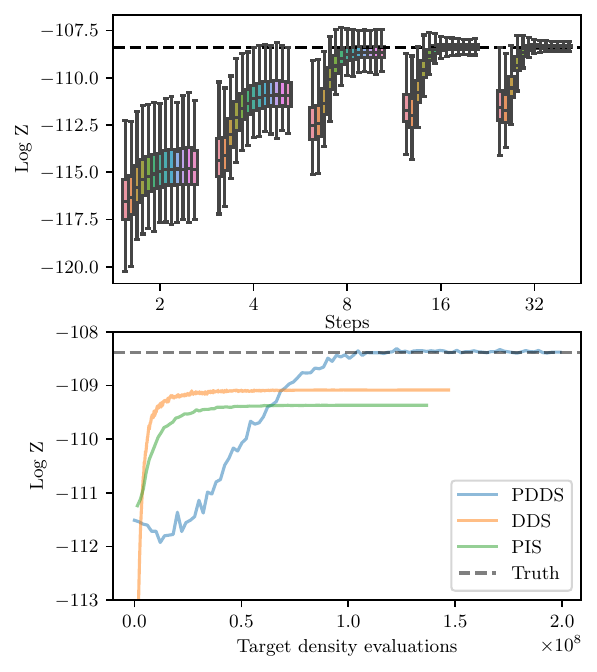}
	\caption{Top: $\log \hmc^N_0$ every 2 iterations of \cref{algo:learn_theta} on the Sonar task. Left-most bar of each group shows PDDS with the simple approximation. Bottom: $\log \hmc^N_0$ during training, one realization on the Sonar task with 16 steps.}
	\label{fig:pdds_iters}
\end{figure}

\subsection{Iterations of potential approximation}

Here we demonstrate the improvement in normalizing constant estimates due to our iterative potential approximation scheme. \cref{fig:ess_curves} shows the improvement in ESS and reduction in number of resampling steps. In the top pane of \cref{fig:pdds_iters}, we see that the simple potential approximation (\ref{eq:naive_approximation}) provides very poor normalizing constant estimates. Iterations of \cref{algo:learn_theta} considerably improve performance. %

In the second pane of \cref{fig:pdds_iters}, we show the evolution of $\log \hmc^N_0$ during training for each of the diffusion-based approaches. While PDDS initially falls below PIS and DDS due to the simple initial approximation, we exceed each of these methods after around $7\times10^7$ density evaluations. %

\section{Discussion}
This paper contributes to the growing literature on the use of denoising diffusion ideas for Monte Carlo sampling \citep{berner2022optimal,mcdonald2022proposal,vargasDDSampler2023,huang2023monte,richter2023improved}. It proposes an original iterative SMC algorithm which provides an unbiased estimate of the normalizing constant for any finite number of particles by leveraging an original score matching technique. This algorithm also provides asymptotically consistent estimates of the normalizing constant and of expectations with respect to the target. One limitation of PDDS is that it practically relies on $g_0(x)$ being a well-behaved potential function. While our approach using a variational approximation to guide a reparameterization of the target has been effective in our examples, more sophisticated techniques might have to be implemented \citep{hoffman2019neutra}.

\clearpage

\section*{Impact Statement}
Sampling is a ubiquitous problem. Therefore, PDDS can be applied to a wide range of applications. While we have obtained limit theorems for this scheme, it is important to exercise caution if the output were to be used for decision making as we practically only use a finite number of particles.

\section*{Acknowledgements}
We would like to thank Pierre Del Moral, Christian A. Naesseth, Sam Power, Saifuddin Syed, Brian Trippe, Francisco Vargas, and Luhuan Wu for helpful discussions. AP is funded by the EPSRC CDT in Modern Statistics and Statistical Machine Learning through grant EP/S023151/1. HD is funded by a postdoctoral position through the CoSInES grant EP/R034710/1. AD acknowledges support from EPSRC grants EP/R034710/1 and EP/R018561/1.
\newpage
\bibliography{refDDS}

\begin{thebibliography}{57}
\providecommand{\natexlab}[1]{#1}
\providecommand{\url}[1]{\texttt{#1}}
\expandafter\ifx\csname urlstyle\endcsname\relax
  \providecommand{\doi}[1]{doi: #1}\else
  \providecommand{\doi}{doi: \begingroup \urlstyle{rm}\Url}\fi

\bibitem[Anastasiou et~al.(2023)Anastasiou, Barp, Briol, Ebner, Gaunt,
  Ghaderinezhad, Gorham, Gretton, Ley, Liu, Mackey, Oates, Reinert, and
  Swan]{anastasiou2023steinmethod}
Anastasiou, A., Barp, A., Briol, F.-X., Ebner, B., Gaunt, R.~E., Ghaderinezhad,
  F., Gorham, J., Gretton, A., Ley, C., Liu, Q., Mackey, L., Oates, C.~J.,
  Reinert, G., and Swan, Y.
\newblock Stein's method meets computational statistics: a review of some
  recent developments.
\newblock \emph{Statistical Science}, 38\penalty0 (1):\penalty0 120--139, 2023.

\bibitem[Arbel et~al.(2021)Arbel, Matthews, and Doucet]{arbel2021annealed}
Arbel, M., Matthews, A., and Doucet, A.
\newblock Annealed flow transport {M}onte {C}arlo.
\newblock In \emph{International Conference on Machine Learning}, 2021.

\bibitem[Assaraf \& Caffarel(1999)Assaraf and
  Caffarel]{assaraf1999zerovariance}
Assaraf, R. and Caffarel, M.
\newblock Zero-variance principle for {M}onte {C}arlo algorithms.
\newblock \emph{Physical Review Letters}, 83:\penalty0 4682--4685, 1999.

\bibitem[Berner et~al.(2022)Berner, Richter, and Ullrich]{berner2022optimal}
Berner, J., Richter, L., and Ullrich, K.
\newblock An optimal control perspective on diffusion-based generative
  modeling.
\newblock In \emph{NeurIPS Workshop on Score-Based Methods}, 2022.

\bibitem[Bradbury et~al.(2018)Bradbury, Frostig, Hawkins, Johnson, Leary,
  Maclaurin, Necula, Paszke, Vander{P}las, Wanderman-{M}ilne, and
  Zhang]{Bradbury:2018}
Bradbury, J., Frostig, R., Hawkins, P., Johnson, M.~J., Leary, C., Maclaurin,
  D., Necula, G., Paszke, A., Vander{P}las, J., Wanderman-{M}ilne, S., and
  Zhang, Q.
\newblock {JAX}: composable transformations of {P}ython+{N}um{P}y programs,
  2018.

\bibitem[Cardoso et~al.(2024)Cardoso, Idrissi, Corff, and
  Moulines]{cardoso2023diffusion}
Cardoso, G., Idrissi, Y. J.~E., Corff, S.~L., and Moulines, E.
\newblock Monte {C}arlo guided diffusion for {B}ayesian linear inverse
  problems.
\newblock In \emph{International Conference on Learning Representations}, 2024.

\bibitem[Cattiaux et~al.(2023)Cattiaux, Conforti, Gentil, and
  L{\'e}onard]{cattiaux2021time}
Cattiaux, P., Conforti, G., Gentil, I., and L{\'e}onard, C.
\newblock Time reversal of diffusion processes under a finite entropy
  condition.
\newblock \emph{Annales de l'Institut Henri Poincar{\'e} (B) Probabilites et
  statistiques}, 59\penalty0 (4):\penalty0 1844--1881, 2023.

\bibitem[Chatterjee \& Diaconis(2018)Chatterjee and Diaconis]{chatterjee2018}
Chatterjee, S. and Diaconis, P.
\newblock The sample size required in importance sampling.
\newblock \emph{The Annals of Applied Probability}, 28\penalty0 (2):\penalty0
  1099--1135, 2018.

\bibitem[Chopin \& Papaspiliopoulos(2020)Chopin and
  Papaspiliopoulos]{chopin2020book}
Chopin, N. and Papaspiliopoulos, O.
\newblock \emph{An Introduction to Sequential {Monte} {Carlo}}.
\newblock Springer Ser. Stat. Springer, 2020.

\bibitem[Chopin et~al.(2022)Chopin, Singh, Soto, and
  Vihola]{chopin2022resampling}
Chopin, N., Singh, S.~S., Soto, T., and Vihola, M.
\newblock On resampling schemes for particle filters with weakly informative
  observations.
\newblock \emph{The Annals of Statistics}, 50\penalty0 (6):\penalty0
  3197--3222, 2022.

\bibitem[Chung et~al.(2023)Chung, Kim, Mccann, Klasky, and
  Ye]{chung2023diffusion}
Chung, H., Kim, J., Mccann, M.~T., Klasky, M.~L., and Ye, J.~C.
\newblock Diffusion posterior sampling for general noisy inverse problems.
\newblock In \emph{International Conference on Learning Representations}, 2023.

\bibitem[Corso et~al.(2023)Corso, Xu, De~Bortoli, Barzilay, and
  Jaakkola]{corso2023particle}
Corso, G., Xu, Y., De~Bortoli, V., Barzilay, R., and Jaakkola, T.
\newblock Particle guidance: non-iid diverse sampling with diffusion models.
\newblock In \emph{International Conference on Learning Representations}, 2023.

\bibitem[Dai et~al.(2022)Dai, Heng, Jacob, and Whiteley]{dai2022invitation}
Dai, C., Heng, J., Jacob, P.~E., and Whiteley, N.
\newblock An invitation to sequential {M}onte {C}arlo samplers.
\newblock \emph{Journal of the American Statistical Association}, 117\penalty0
  (539):\penalty0 1587--1600, 2022.

\bibitem[De~Bortoli et~al.(2021)De~Bortoli, Thornton, Heng, and
  Doucet]{debortoli2021diffusion}
De~Bortoli, V., Thornton, J., Heng, J., and Doucet, A.
\newblock Diffusion {S}chr{\"o}dinger bridge with applications to score-based
  generative modeling.
\newblock In \emph{Advances in Neural Information Processing Systems}, 2021.

\bibitem[Del~Moral(2004)]{del2004feynman}
Del~Moral, P.
\newblock \emph{Feynman-{K}ac Formulae: Genealogical and Interacting Particle
  Approximations}.
\newblock Springer, 2004.

\bibitem[Del~Moral et~al.(2006)Del~Moral, Doucet, and Jasra]{del2006sequential}
Del~Moral, P., Doucet, A., and Jasra, A.
\newblock Sequential {M}onte {C}arlo samplers.
\newblock \emph{Journal of the Royal Statistical Society: Series B (Statistical
  Methodology)}, 68\penalty0 (3):\penalty0 411--436, 2006.

\bibitem[Del~Moral et~al.(2012)Del~Moral, Doucet, and Jasra]{Del-Moral:2012}
Del~Moral, P., Doucet, A., and Jasra, A.
\newblock On adaptive resampling strategies for sequential {M}onte {C}arlo
  methods.
\newblock \emph{Bernoulli}, 18\penalty0 (1):\penalty0 252--278, 2012.

\bibitem[Douc \& Capp{\'e}(2005)Douc and Capp{\'e}]{douc2005comparison}
Douc, R. and Capp{\'e}, O.
\newblock Comparison of resampling schemes for particle filtering.
\newblock In \emph{Proceedings of the 4th International Symposium on Image and
  Signal Processing and Analysis}, pp.\  64--69. IEEE, 2005.

\bibitem[Doucet et~al.(2001)Doucet, De~Freitas, and Gordon]{Doucet:2001}
Doucet, A., De~Freitas, N., and Gordon, N.~J.
\newblock \emph{Sequential {Monte} {Carlo} {Methods} in {Practice}}.
\newblock Information {Science} and {Statistics}. New York, NY: Springer, New
  York, 2001.

\bibitem[Gerber et~al.(2019)Gerber, Chopin, and Whiteley]{gerber2019negative}
Gerber, M., Chopin, N., and Whiteley, N.
\newblock Negative association, ordering and convergence of resampling methods.
\newblock \emph{The Annals of Statistics}, 47\penalty0 (4):\penalty0
  2236--2260, 2019.

\bibitem[Geyer(1991)]{geyer1991markov}
Geyer, C.
\newblock {Markov chain Monte Carlo maximum likelihood}.
\newblock In \emph{Computing science and statistics: Proceedings of 23rd
  Symposium on the Interface Interface Foundation, Fairfax Station, 1991}, pp.\
   156--163, 1991.

\bibitem[Guarniero et~al.(2017)Guarniero, Johansen, and
  Lee]{guarniero2017iterated}
Guarniero, P., Johansen, A.~M., and Lee, A.
\newblock The iterated auxiliary particle filter.
\newblock \emph{Journal of the American Statistical Association}, 112\penalty0
  (520):\penalty0 1636--1647, 2017.

\bibitem[Haussmann \& Pardoux(1986)Haussmann and Pardoux]{haussmann1986time}
Haussmann, U.~G. and Pardoux, E.
\newblock Time reversal of diffusions.
\newblock \emph{The Annals of Probability}, 14\penalty0 (3):\penalty0
  1188--1205, 1986.

\bibitem[Heng et~al.(2020)Heng, Bishop, Deligiannidis, and
  Doucet]{heng2020controlled}
Heng, J., Bishop, A.~N., Deligiannidis, G., and Doucet, A.
\newblock Controlled sequential {M}onte {C}arlo.
\newblock \emph{The Annals of Statistics}, 48\penalty0 (5):\penalty0
  2904--2929, 2020.

\bibitem[Ho et~al.(2020)Ho, Jain, and Abbeel]{ho2020denoising}
Ho, J., Jain, A., and Abbeel, P.
\newblock Denoising diffusion probabilistic models.
\newblock In \emph{Advances in Neural Information Processing Systems}, 2020.

\bibitem[Hoffman et~al.(2019)Hoffman, Sountsov, Dillon, Langmore, Tran, and
  Vasudevan]{hoffman2019neutra}
Hoffman, M., Sountsov, P., Dillon, J.~V., Langmore, I., Tran, D., and
  Vasudevan, S.
\newblock Neu{T}ra-lizing bad geometry in {H}amiltonian {M}onte {C}arlo using
  neural transport.
\newblock \emph{arXiv preprint arXiv:1903.03704}, 2019.

\bibitem[Huang et~al.(2024)Huang, Dong, Hao, Ma, and Zhang]{huang2023monte}
Huang, X., Dong, H., Hao, Y., Ma, Y., and Zhang, T.
\newblock Reverse diffusion {M}onte {C}arlo.
\newblock In \emph{International Conference on Learning Representations}, 2024.

\bibitem[Hyv{\"a}rinen(2005)]{Hyvarinen:2005a}
Hyv{\"a}rinen, A.
\newblock Estimation of non-normalized statistical models by score matching.
\newblock \emph{The Journal of Machine Learning Research}, 6:\penalty0
  695--709, 2005.

\bibitem[Kingma \& Ba(2015)Kingma and Ba]{adam}
Kingma, D. and Ba, J.
\newblock Adam: A method for stochastic optimization.
\newblock In \emph{International Conference on Learning Representations}, 2015.

\bibitem[Kloeden \& Platen(1992)Kloeden and Platen]{Kloeden1992numerical}
Kloeden, P.~E. and Platen, E.
\newblock \emph{Numerical Solution of Stochastic Differential Equations},
  volume~23 of \emph{Appl. Math. (N. Y.)}.
\newblock Berlin: Springer-Verlag, 1992.

\bibitem[Lai et~al.(2022)Lai, Takida, Murata, Uesaka, Mitsufuji, and
  Ermon]{lai2022regularizing}
Lai, C.-H., Takida, Y., Murata, N., Uesaka, T., Mitsufuji, Y., and Ermon, S.
\newblock Regularizing score-based models with score {F}okker-{P}lanck
  equations.
\newblock In \emph{NeurIPS Workshop on Score-Based Methods}, 2022.

\bibitem[Lawson et~al.(2022)Lawson, Ravent{\'o}s, and
  Linderman]{lawson2022sixo}
Lawson, D., Ravent{\'o}s, A., and Linderman, S.
\newblock {SIXO}: Smoothing inference with twisted objectives.
\newblock \emph{Advances in Neural Information Processing Systems}, 2022.

\bibitem[Liptser \& Shiryayev(1977)Liptser and Shiryayev]{liptser1977book}
Liptser, R.~S. and Shiryayev, A.~N.
\newblock \emph{Statistics of Random Processes. {I}. {General} theory.
  {Translated} by {A}. {B}. {Aries}}, volume~5 of \emph{Appl. Math. (N. Y.)}.
\newblock Springer, New York, 1977.

\bibitem[M{\'a}t{\'e} \& Fleuret(2023)M{\'a}t{\'e} and
  Fleuret]{mate2023learning}
M{\'a}t{\'e}, B. and Fleuret, F.
\newblock Learning deformation trajectories of {B}oltzmann densities.
\newblock \emph{Transactions on Machine Learning Research}, 2023.

\bibitem[Matthews et~al.(2022)Matthews, Arbel, Rezende, and
  Doucet]{Matthews2022}
Matthews, A. G. D.~G., Arbel, M., Rezende, D.~J., and Doucet, A.
\newblock Continual repeated annealed flow transport {M}onte {C}arlo.
\newblock In \emph{International Conference on Machine Learning}, 2022.

\bibitem[McDonald \& Barron(2022)McDonald and Barron]{mcdonald2022proposal}
McDonald, C.~J. and Barron, A.~R.
\newblock Proposal of a score based approach to sampling using {M}onte {C}arlo
  estimation of score and oracle access to target density.
\newblock In \emph{NeurIPS Workshop on Score-Based Methods}, 2022.

\bibitem[Midgley et~al.(2023)Midgley, Stimper, Simm, Sch{\"o}lk~opf, and
  Hern{\'a}ndez-Lobato]{midgley2022flow}
Midgley, L.~I., Stimper, V., Simm, G.~N., Sch{\"o}lk~opf, B., and
  Hern{\'a}ndez-Lobato, J.~M.
\newblock Flow annealed importance sampling bootstrap.
\newblock In \emph{International Conference on Learning Representations}, 2023.

\bibitem[Mira et~al.(2013)Mira, Solgi, and Imparato]{mira2013zerovariance}
Mira, A., Solgi, R., and Imparato, D.
\newblock Zero variance {Markov} chain {Monte} {Carlo} for {Bayesian}
  estimators.
\newblock \emph{Statistics and Computing}, 23\penalty0 (5):\penalty0 653--662,
  2013.

\bibitem[M{\o}ller et~al.(1998)M{\o}ller, Syversveen, and
  Waagepetersen]{Moller:1998}
M{\o}ller, J., Syversveen, A.~R., and Waagepetersen, R.~P.
\newblock Log {G}aussian {C}ox processes.
\newblock \emph{Scandinavian Journal of Statistics}, 25\penalty0 (3):\penalty0
  451--482, 1998.

\bibitem[Neal(2001)]{neal2001annealed}
Neal, R.~M.
\newblock Annealed importance sampling.
\newblock \emph{Statistics and Computing}, 11\penalty0 (2):\penalty0 125--139,
  2001.

\bibitem[Neal(2003)]{Neal:2003}
Neal, R.~M.
\newblock Slice sampling.
\newblock \emph{The Annals of Statistics}, 31:\penalty0 705--767, 06 2003.

\bibitem[Nichol \& Dhariwal(2021)Nichol and Dhariwal]{nichol2021improved}
Nichol, A.~Q. and Dhariwal, P.
\newblock Improved denoising diffusion probabilistic models.
\newblock In \emph{International Conference on Machine Learning}, 2021.

\bibitem[Richter et~al.(2024)Richter, Berner, and Liu]{richter2023improved}
Richter, L., Berner, J., and Liu, G.-H.
\newblock Improved sampling via learned diffusions.
\newblock In \emph{International Conference on Learning Representations}, 2024.

\bibitem[Salimans \& Ho(2021)Salimans and Ho]{salimans2021should}
Salimans, T. and Ho, J.
\newblock Should {EBM}s model the energy or the score?
\newblock In \emph{Energy Based Models Workshop-ICLR 2021}, 2021.

\bibitem[Song et~al.(2023)Song, Vahdat, Mardani, and
  Kautz]{song2023pseudoinverseguided}
Song, J., Vahdat, A., Mardani, M., and Kautz, J.
\newblock Pseudoinverse-guided diffusion models for inverse problems.
\newblock In \emph{International Conference on Learning Representations}, 2023.

\bibitem[Song \& Ermon(2019)Song and Ermon]{song2019generative}
Song, Y. and Ermon, S.
\newblock Generative modeling by estimating gradients of the data distribution.
\newblock \emph{Advances in Neural Information Processing Systems}, 2019.

\bibitem[Song et~al.(2021)Song, Sohl-Dickstein, Kingma, Kumar, Ermon, and
  Poole]{song2020score}
Song, Y., Sohl-Dickstein, J., Kingma, D.~P., Kumar, A., Ermon, S., and Poole,
  B.
\newblock Score-based generative modeling through stochastic differential
  equations.
\newblock In \emph{International Conference on Learning Representations}, 2021.

\bibitem[Sountsov et~al.(2020)Sountsov, Radul, and
  contributors]{inferencegym2020}
Sountsov, P., Radul, A., and contributors.
\newblock Inference gym, 2020.
\newblock URL \url{https://pypi.org/project/inference_gym}.

\bibitem[Syed et~al.(2022)Syed, Bouchard-C{\^o}t{\'e}, Deligiannidis, and
  Doucet]{syed2022non}
Syed, S., Bouchard-C{\^o}t{\'e}, A., Deligiannidis, G., and Doucet, A.
\newblock Non-reversible parallel tempering: A scalable highly parallel {MCMC}
  scheme.
\newblock \emph{Journal of the Royal Statistical Society Series B}, 84\penalty0
  (2):\penalty0 321--350, 2022.

\bibitem[Tawn et~al.(2020)Tawn, Roberts, and Rosenthal]{tawn2020weight}
Tawn, N.~G., Roberts, G.~O., and Rosenthal, J.~S.
\newblock Weight-preserving simulated tempering.
\newblock \emph{Statistics and Computing}, 30\penalty0 (1):\penalty0 27--41,
  2020.

\bibitem[Vargas et~al.(2023)Vargas, Grathwohl, and Doucet]{vargasDDSampler2023}
Vargas, F., Grathwohl, W., and Doucet, A.
\newblock Denoising diffusion samplers.
\newblock In \emph{International Conference on Learning Representations}, 2023.

\bibitem[Vincent(2011)]{vincent2011connection}
Vincent, P.
\newblock A connection between score matching and denoising autoencoders.
\newblock \emph{Neural Computation}, 23\penalty0 (7):\penalty0 1661--1674,
  2011.

\bibitem[Webber(2019)]{webber2019unifying}
Webber, R.~J.
\newblock Unifying sequential {M}onte {C}arlo with resampling matrices.
\newblock \emph{arXiv preprint arXiv:1903.12583}, 2019.

\bibitem[Woodard et~al.(2009)Woodard, Schmidler, and
  Huber]{woodard2009sufficient}
Woodard, D.~B., Schmidler, S.~C., and Huber, M.
\newblock Sufficient conditions for torpid mixing of parallel and simulated
  tempering.
\newblock \emph{Electronic Journal of Probability}, 14:\penalty0 780--804,
  2009.

\bibitem[Wu et~al.(2023)Wu, Trippe, Naesseth, Blei, and
  Cunningham]{wu2023practical2023}
Wu, L., Trippe, B.~L., Naesseth, C.~A., Blei, D., and Cunningham, J.~P.
\newblock Practical and asymptotically exact conditional sampling in diffusion
  models.
\newblock In \emph{Advances in Neural Information Processing Systems}, 2023.

\bibitem[Zhang et~al.(2024)Zhang, Chen, Liu, Courville, and
  Bengio]{zhang2023diffusion}
Zhang, D., Chen, R.~T., Liu, C.-H., Courville, A., and Bengio, Y.
\newblock Diffusion generative flow samplers: Improving learning signals
  through partial trajectory optimization.
\newblock In \emph{International Conference on Learning Representations}, 2024.

\bibitem[Zhang \& Chen(2022)Zhang and Chen]{zhangyongxinchen2021path}
Zhang, Q. and Chen, Y.
\newblock Path integral sampler: a stochastic control approach for sampling.
\newblock In \emph{International Conference on Learning Representations}, 2022.

\end{thebibliography}
\bibliographystyle{icml2024}

\newpage
\appendix
\onecolumn

\section*{Appendix}
The Appendix is organized as follows. In \Cref{apx:adaptive_resampling} we detail the PDDS algorithm when adaptive resampling is used. In Appendix~\ref{sec:webber_presentation} we reformulate the results of \citet{webber2019unifying} in terms of chi-squared divergences. In \Cref{apx:proofs}, we expose our proofs. Finally, in \Cref{apx:experimental} we present details and additional work relating to our experiments.

\section{Particle Denoising Diffusion Sampler with Adaptive Resampling}
\label{apx:adaptive_resampling}

\begin{algorithm}[H]
	\caption{Particle Denoising Diffusion Sampler with Adaptive Resampling}
	\label{algo:smc_dds_adaptive}
	\begin{algorithmic}
		\REQUIRE Schedule $(\beta_t)_{t \in [0,T]}$ as in~\eqref{eq:forward_diffusion}; Approximations $(\hat g_k)_{k=0}^K$ s.t. $\hat g_0=g_0, \hat g_K = 1$;
		Number of particles $N$; ESS resampling threshold $\alpha$
		\STATE Sample $X_K^i \overset{\mathrm{iid}}{\sim} \mathcal N(0, \mathrm{I})$ for $i \in [N]$
		\STATE Set $\hat {\mathcal Z}_K \gets 1$ and $\omega_K^i \gets 1/N$ for $i \in [N]$
		\FOR{$k = K-1, \ldots, 0$}
		\STATE \underline{Move}. Sample $\tilde X_k^i \sim \hat \pi(\cdot|X_{k+1}^i)$ for $i \in [N]$ (see~\eqref{eq:smc_proposal})
		\STATE \underline{Weight}. $\omega_k^i \gets \omega_k^i~ \omega_k(\tilde X_k^i, X_{k+1}^i)$ for $i \in [N]$ (see~\eqref{eq:importanceweights})
		\STATE Set $\hmc_k \gets \hmc_{k+1} \times \frac 1N \sum_{i\in[N]} \omega_k^i$
		\STATE Normalize $\omega_k^i \gets \omega_k^i/\sum_{j\in[N]} \omega_k^j$
		\STATE \underline{Resample and MCMC}.
		\IF{$(\sum_{i \in [N]} (\omega_k^i)^2)^{-1} < \alpha N$}
		\STATE $X_k^{1:N} \gets \operatorname{resample}(\tilde X_k^{1:N}, \omega_k^{1:N})$ (see Section~\ref{sec:algo_settings})
		\STATE (Optional). Sample $X_k^i \gets \mathfrak M_k(X_k^i, \cdot)$ for $i\in[N]$ using a $\hat \pi_k$-invariant MCMC kernel.
		\STATE Reset $\omega_k^{1:N} \gets 1/N$
		\ELSE
		\STATE $X_k^{1:N} \gets \tilde{X}_k^{1:N}$
		\ENDIF
		\ENDFOR
		\ENSURE Particle estimates $\hat \pi^N = \frac 1N \sum_{i\in[N]} \delta_{X_0^i}$ of $\pi$ and $\hmc^N_0$ of $\mathcal Z$
	\end{algorithmic}
\end{algorithm}

\section{Asymptotic Error Formulae for SMC}
\label{sec:webber_presentation}
The results of \citet{webber2019unifying} play a fundamental role in our work. Compared to traditional SMC literature, they put more stress on the normalizing constant estimate, have results for different resampling schemes, and do not require weight boundedness.

In this section, we present these results using the language of chi-squared divergence. This gives alternative expressions which are easier to manipulate in our context.
\subsection{Chi-squared divergence}
For two probability distributions $p$ and $q$ such that $q \ll p$, the chi-squared divergence of $q$ with respect to $p$ is defined as
\begin{equation*}
	\chisquared qp = \int_{\mathcal X} p(\dd x) \ps{\pr{\frac{\dd q}{\dd p}}^2(x) - 1}.
\end{equation*}
We state without proof two simple properties of this divergence.
\begin{lemma}
	\label{lem:chisqvar}
	Let $g$ be a nonnegative function from $\mathcal X$ to $\mathbb R$ such that $p(g):=\E_p[g(X)] < \infty$. Define the probability distribution $\pi$ by $\pi(\dd x) \propto p(x)g(x)$. Then
	\begin{equation*}
		\Var_p(g) = \pr{p(g)}^2 \chisquared \pi p.
	\end{equation*}
\end{lemma}

\begin{lemma}
	\label{lem:chisq_chainrule}
	The chi-squared divergence between two probability distributions $p$ and $q$ on $\mathcal X \times \mathcal Y$ satisfies the decomposition
	\begin{equation*}
		\chisquared{q(\dd x, \dd y)}{p(\dd x, \dd y)} = \chisquared{q(\dd x)}{p(\dd x)} + \int p(\dd x) \pr{\frac{\dd q}{\dd q}}^2(x) \chisquared{q(\dd y|x)}{p(\dd y|x)}.
	\end{equation*}
\end{lemma}

\subsection{Generic Feynman-Kac formula and the associated SMC algorithm}
\label{sec:fk_general}
\newcommand{\M}{\mathbb M}
Let
\begin{equation*}
	\M(\dd x_{0:T}) = \M(\dd x_0) \M(\dd x_1|x_0) \ldots \M(\dd x_T|x_{T-1})
\end{equation*} be a Markov measure defined on $\mathcal X_0 \times \ldots \times \mathcal X_T$. Let $G_0(x_0)$, $G_1(x_0, x_1)$, $\ldots$, $G_T(x_{T-1}, x_T)$ be strictly positive functions. For any $t<T$, assume that there exists $Z_t>0$ such that
\begin{equation*}
	\Q_t(\dd x_{0:T}) = \frac{1}{Z_t} \M_0(\dd x_0) G_0(x_0) \M_1(\dd x_1|x_0) G_1(x_0, x_1) \ldots \M_t(\dd x_t|x_{t-1}) G_t(x_t) \M(\dd x_{t+1:T}|x_t)
\end{equation*}
is a probability measure. Then $\Q_t(\dd x_{0:t})$ is called the Feynman-Kac model associated with the Markov measure $\M(\dd x_{0:t})$ and the weight functions $(G_s)_{s\leq t}$. Given a number of particles $N$, Algorithm~\ref{algo:smc_generic} approximates $\Q_T(\dd x_T)$ and the normalizing constant $Z_T$. It is called the SMC algorithm associated to the Feynman-Kac model $\Q_T$.

\begin{algorithm}
	\begin{algorithmic}
		\caption{Generic SMC algorithm}
		\label{algo:smc_generic}
		\REQUIRE Markov kernels $\M(\dd x_t|x_{t-1})$; Functions $G_t(x_{t-1}, x_t)$; Number of particles $N$
		\STATE Sample $X_0^{1:N} \simiid \M_0(\dd x_0)$
		\STATE Set $\omega_0^n = G_0(X_0^n)$ and $Z_0^N = \frac 1N \sum \omega_0^n$
		\STATE Normalize $\omega_0^n \gets \omega_0^n/\sum_m \omega_0^m$
		\FOR{$t=1, \ldots, T$}
		\STATE Resample particles $\tilde X_{t-1}^{1:N}$ among particles $X_{t-1}^{1:N}$ with weights $\omega_{t-1}^{1:N}$
		\STATE Sample $X_t^n \sim \M_t(\dd x_t|\tilde X_{t-1}^n)$
		\STATE Set $\omega_t^n = G_t(\tilde X_{t-1}^n, X_t^n)$ and $Z_t^n = Z_{t-1}^n \frac 1N \sum \omega_t^n$
		\STATE Normalize $\omega_t^n \gets \omega_t^n/\sum \omega_t^m$
		\ENDFOR
		\ENSURE Empirical measure $\sum \omega_T^n \delta_{X_T^n}$ approximating $\Q_T(\dd x_T)$ and estimate $Z_T^N$ approximating $Z_T$
	\end{algorithmic}
\end{algorithm}

\subsection{Sorted stratified resampling schemes}
The resampling step of Algorithm~\ref{algo:smc_generic} can be done in a number of different ways. It is well known that multinomial resampling should be avoided. Practitioners often rely on alternative schemes, in particular systematic resampling. However, the theoretical properties of these schemes are less well-studied.

To investigate theoretically the asymptotic error of PDDS when the discretization step tends to zero, we consider the stratified resampling scheme where particles are sorted by a certain coordinate $\theta:\mathbb R^d \to \mathbb R$, see Algorithm~\ref{algo:sorted_stratified}.

\begin{algorithm}
	\begin{algorithmic}
		\caption{Generic sorted stratified resampling}
		\label{algo:sorted_stratified}
		\REQUIRE Particles $X^{1:N}$ in $\mathbb R^d$ with weights $W^{1:N}$; Sorting function $\theta: \mathbb R^d \to \mathbb R$
		\STATE Sort the particles $X^{1:N}$ such that $\theta(X^{s_1}) \leq \ldots \leq \theta(X^{s_N})$ for a permutation $s_{1:N}$ of $\px{1, 2, \ldots, N}$
		\FOR{$n = 1, \ldots, N$}
		\STATE Simulate $U^n \sim \operatorname{Uniform}[0, 1]$
		\STATE Find the index $k_n$ such that $\sum_{i=1}^{k_n-1} W^{s_i} \leq \frac{n-1+U^n}{N} < \sum_{i=1}^{k_n} W^{s_i}$
		\STATE Set $\tilde X_n \gets X^{s_{k_n}}$
		\ENDFOR
		\ENSURE Resampled particles $\tilde X^{1:N}$
	\end{algorithmic}
\end{algorithm}

\newcommand{\vqmo}{\Var_{\Q_{t-1}}}
\newcommand{\eqmo}{\E_{\Q_{t-1}}}
Depending on the chosen coordinate $\theta$, sorted stratified resampling can significantly reduce the asymptotic error of the particle filter. Precise formulations is given in \citet[Theorem 3.2]{webber2019unifying}. In a nutshell, the variance of $Z_T^N/Z_T$ comes from two sources: the mismatch between the PF proposal $\M_t(\dd x_t|x_{t-1})$ and the target law $\Q_T(\dd x_t|x_{t-1})$; and the error caused by the resampling step. The first source is common to all resampling methods. The magnitude of the second source for multinomial resampling is $\sum_{t=1}^T \vqmo\pr{\bar h_{t-1}(X_{t-1})}$ where
\[ \bar h_{t-1}(x_{t-1}) = \frac{\Q_T(\dd x_{t-1})}{\Q_{t-1}(\dd x_{t-1})} = \frac{\int
	\prod_{s=t}^T \M_s(\dd x_s|x_{s-1}) G_s(x_{s-1}, x_s) \dd x_{t:T}
	}{\int
	\Q_{t-1}(\dd x'_{t-1})\prod_{s=t}^T \M_s(\dd x_s|x_{s-1}) G_s(x_{s-1}, x_s)\dd x'_{t-1} \dd x_{t:T}
	}. \]
Write the resampling error for multinomial resampling as
\[ \sum_{t=1}^T \vqmo\pr{\bar h_{t-1}(X_{t-1})} = \sum_{t=1}^T \eqmo\ps{\vqmo\pr{ \bar h_{t-1}(X_{t-1}) | \theta(X_{t-1}) }}   + \sum_{t=1}^T \vqmo\pr{\eqmo\ps{ \bar h_{t-1}(X_{t-1}) | \theta(X_{t-1}) }}. \]
Then the error for stratified resampling when particles are sorted by $\theta: \mathbb R^d \to \mathbb R$ contains only the first term and not the second one. Thus ideally we would like to choose $\theta = \bar h_{t-1}$ so that
\begin{equation}
	\label{eq:zero_rs_variance}
	\Var\pr{\bar h_{t-1}(X_{t-1}) | \theta(X_{t-1})} = 0.
\end{equation}
While the ideal function $\bar h_{t-1}$ is intractable, there are more generic choices of $\theta$ which guarantee~\eqref{eq:zero_rs_variance}. If $\theta:\mathbb R^d \to \mathbb R$ is an injective map then~\eqref{eq:zero_rs_variance} automatically holds. Such maps usually arise as pseudo-inverses of space-filling curves.

Following~\citet{gerber2019negative}, we take $\theta$ to be the pseudo-inverse of the Hilbert curve, of which the existence is given by their Proposition 2. More precisely, that proposition gives an injective map from $[0,1]^d$ to $[0,1]$; so to get an injection from $\mathbb R^d$ to $[0,1]$ one might first apply any injection from $\mathbb R^d$ to $[0,1]^d$. In practice, numerical implementations are available to sort particles with the Hilbert curve, for instance the \verb+hilbert+ module of the Python package \verb+particles+ .

We point out that we consider this particular resampling strategy mainly for theoretical convenience. In our experiments, adaptive systematic resampling works well, but as we mentioned earlier, very little is known about their theoretical properties.

\subsection{Asymptotic error}
We recall the following definition from \citet{webber2019unifying}.
\begin{definition}
	The notation $|Y_n - c| \lesssim U_n$ means that there exists a sequence of sets $(B_n)$ such that $\P(B_n) \to 1$ and
	\begin{equation*}
		\limsup_{n \to \infty} \E\ps{\mathbbm{1}_{B_n} \abs{\frac{Y_n-c}{U_n}}^2} \leq 1.
	\end{equation*}
\end{definition}
We are now ready to restate parts of Theorem 3.2 and Example 3.4 of \citet{webber2019unifying} in terms of chi-squared divergences.

\newcommand{\sigmasqmult}{\sigma^2_{\mathrm{mult}}}
\newcommand{\sigmasqsort}{\sigma^2_{\mathrm{sort}}}
\newcommand{\lcfirstt}{\chisquared{\Q_T(\dd x_0)}{\M_0(\dd x_0)}}
\begin{theorem}
	\label{thm:webber_restated}
	Given the Feynman-Kac model defined in Section~\ref{sec:fk_general}, assume that $\chisquared{\Q_T(\dd x_0)}{\M_0(\dd x_0)} < \infty$ and $\chisquared{\Q_T(\dd x_t)}{\Q_t(\dd x_t)} < \infty$, $\forall t$. Then $\sqrt N(Z_T^N/Z_T-1)$ is asymptotically normal with variance $\sigmasqmult$ if multinomial resampling is used; $\abs{\sqrt N (Z_T^N/Z_T-1)}^2 \lesssim \sigmasqsort$ if sorted resampling \citep{gerber2019negative} is used; with
	\begin{equation*}
		\begin{split}
			\sigmasqmult &= \lcfirstt + \sum_{t=1}^T \chisquared{\Q_T(\dd x_{t-1}, \dd x_t)}{\Q_{t-1}(\dd x_{t-1}) M_t(x_{t-1}, \dd x_t)} \\
			&= \underbrace{\lcfirstt + \sum_{t=1}^T \int \Q_{t-1}(\dd x_{t-1}) \pr{\frac{\Q_T(\dd x_{t-1})}{\Q_{t-1}(\dd x_{t-1})}}^2 \chisquared{\Q_T(\dd x_t|x_{t-1})}{\M_t(\dd x_t|x_{t-1})}}_{\sigmasqsort} +\\
			&+ \sum_{t=1}^T \chisquared{\Q_T(\dd x_{t-1})}{\Q_{t-1}(\dd x_{t-1})}
		\end{split}
	\end{equation*}
	where we have decomposed $\sigmasqmult$ using the chain rule (Lemma~\ref{lem:chisq_chainrule}).
\end{theorem}
\begin{proof}
	The original formulation of Theorem 3.2 \citep{webber2019unifying} is written in terms of the following quantities
	\begin{align*}
		\tilde G_t & := \E_\M\ps{\prod_{s=0}^{t-1} G_s} G_t/\E\ps{\prod_{s=0}^t G_s}, \\
		h_t(x_t)   & := \CE{\prod_{s=t+1}^T G_s}{X_t=x_t}.
	\end{align*}
	To translate these notations into our case, note that $\tilde G_t = \Q_t(\dd x_{0:T})/\Q_{t-1}(\dd x_{0:T})$
	and thus
	\begin{equation*}
		\begin{split}
			\min_{c\in \mathbb R} \E_\M\ps{\prod_{s=0}^t \tilde G_s|h_t-c|^2} &= \min_{c\in \mathbb R} \E_{\Q_t}\ps{|h_t-c|^2} = \Var_{\Q_t}(h_t) = \ps{\Q_t(h_t)}^2 \chisquared{\Q_T(\dd x_t)}{\Q_t(\dd x_t)}
		\end{split}
	\end{equation*}
	using Lemma~\ref{lem:chisqvar}.
	Moreover,
	\begin{equation*}
		\Var_\M(G_{t+1}(X_t, X_{t+1})h_{t+1}(X_{t+1})|X_t) = h_t^2(X_t) \chisquared{\Q_T(\dd x_{t+1}|x_t)}{\M_T(\dd x_{t+1}|x_t)}
	\end{equation*}
	and thus, using $h_t(x_t) = \frac{\Q_T(\dd x_t)}{\Q_t(\dd x_t)} \Q_t(h_t)$ we get
	\begin{equation*}
		\E\ps{\prod_{s=0}^t G_s \Var(G_{t+1}h_{t+1}|X_t)} = Z_t \Q_t(h_t)^2 \E_{\Q_t}\ps{\px{\frac{\Q_T(\dd x_t)}{\Q_t(\dd x_t)}(X_t)}^2 \chisquared{\Q_T(\dd x_{t+1}|X_t)}{\M_T(\dd x_{t+1}|X_t}}.
	\end{equation*}
	The identity $Z_t \Q_t(h_t) = Z_T$ helps conclude the proof.
\end{proof}

\section{Proofs}
\label{apx:proofs}

\subsection{Proof of Proposition~\ref{lemma:score}}
\label{proof_lemma_score}
\begin{proof}
	Write
	\newcommand{\inverseZ}{\frac{1}{\mathcal Z}}
	\begin{equation}
		\begin{split}
			\pi_t(x_t) &= \int \pi_0(x_0) p(x_t|x_0) \dd x_0 = \inverseZ \int g_0(x_0) p_0(x_0) p(x_t|x_0) \dd x_0 \\
			&= \inverseZ \int g_0(x_0) p_0(x_t) p(x_0|x_t) \dd x_0 \\
			&= \inverseZ p_0(x_t) g_t(x_t).
		\end{split}
	\end{equation}
	The proof is concluded by taking the log gradient of the obtained identity with respect to $x_t$.
\end{proof}

\subsection{Proof of Lemma~\ref{lem:error_guidance}}
\label{proof_lem_error_guidance}
We first state the following elementary lemma on the solution of a linear SDE \citep[Chapter 4.2]{Kloeden1992numerical}.
\begin{lemma}
	\label{lem:sde_exact}
	Let $a: [0,T] \to \mathbb R$ and $c:[0, T] \to \mathbb R$ be two continuous functions. Put $\Phi_t= \exp\int_0^t a_s \dd s$, $\forall t \in [0, T]$. Then the solution to
	\begin{equation}
		\label{eq:sde_general}
		\dd Z_t = (a_t Z_t + c_t) \dd t+ \sqrt 2 \dd W_t
	\end{equation}
	is
	\begin{equation}
		Z_t = \Phi_t\pr{Z_0 + \int_0^t \Phi_s^{-1} c_s \dd s + \int_0^t \Phi_s^{-1} \sqrt 2 \dd W_s}.
	\end{equation}
\end{lemma}
Using It\^o isometry, we get the following straightforward corollary.
\begin{corollary}
	\label{cor:exact_mean_cov}
	Under the setting of Lemma~\ref{lem:sde_exact}, if $Z_0 \sim \mathcal N(0, 1)$, then
	\begin{align}
		\E[Z_t] = \Phi_t \int_0^t \Phi_s^{-1} c_s \dd s,\quad \operatorname{Var}(Z_t) = \Phi_t^2 + 2\Phi_t^2\int_0^t \Phi_s^{-2} \dd s.
	\end{align}
\end{corollary}
We are now ready to give the proof of Proposition~\ref{lem:error_guidance}.
\newcommand{\hatbar}[1]{\hat {\bar #1}}
\newcommand{\powertp}[1]{{#1}^{(T)}}
\begin{proof}[Proof of Proposition~\ref{lem:error_guidance}]
	Without loss of generality, we can assume that $\beta_t \equiv 1$. Indeed, putting $\hatbar g_t(x_t):= g_0(e^{-t} x_t)$ and defining $\powertp{\bar Z_t}$ by
	\begin{equation}
		\dd \powertp{\bar Z_t} = \ps{-\powertp{\bar Z_t} + 2 \gradlog \hatbar g_{T-t}(\powertp{\bar Z_t})} \dd t+ \sqrt 2 \dd \tilde B_t, \quad \powertp{\bar Z_0} \sim \mathcal N(0, 1),
	\end{equation}
	we see that $Z_{0:t}^{(T)}$ has the same law as $\bar Z^{\int_0^T \beta_s \dd s}_{\int_{T-t}^T \beta_s \dd s}$.

	We only consider the case $\sigma \neq 1$ here. The case $\sigma = 1$ can be treated using similar but simpler calculations; it can also be recovered by letting $\sigma \to 1$ in the expressions below.

	Since $\gradlog g_0(x_0) = \gradlog \pi(x_0) - \gradlog p_0(x_0) = -(x-\mu)/\sigma^2 + x$, Equation~\eqref{eq:naive_approximation_in_action} has the form~\eqref{eq:sde_general} with
	\begin{equation}
		a_t = -\ps{1 + 2e^{2(t-T)}(1/\sigma^2 - 1)}, \quad c_t = 2e^{t-T}\mu/\sigma^2.
	\end{equation}
	Tedious but standard calculations yield
	\begin{equation}
		\Phi_t = \exp{-t - e^{-2T}(1/\sigma^2 - 1)(e^{2t}-1)}
	\end{equation}
	and the integral of $t$-dependent terms in $\int_0^T \Phi_t^{-1} c_t \dd t$ is
	\begin{equation*}
		\int_0^T \exp{2t+(1/\sigma^2-1)e^{2(t-T)}} \dd t = \frac{1}{2e^{-2T}(1/\sigma^2 - 1)} \ps{\exp{1/\sigma^2 - 1} - \exp{e^{-2T}(1/\sigma^2 - 1)}}.
	\end{equation*}
	Using Corollary~\ref{cor:exact_mean_cov}, we have
	\begin{equation}
		\E[Z_T^{(T)}] = \frac{\mu}{1-\sigma^2} \ps{1 - \exp{(e^{-2T} - 1)(1/\sigma^2 - 1)}}
	\end{equation}
	and $\operatorname{Var}(Z_T^{(T)}) = v_{T, 1} + v_{T, 2}$, where
	\begin{align}
		v_{T, 1} & = \frac{1}{2(1/\sigma^2 - 1)} \ps{1 - \exp{-2(1/\sigma^2 - 1)(1 - e^{-2T})}}, \\
		v_{T, 2} & = \exp{-2T-2(1/\sigma^2 - 1)(1-e^{-2T})}.
	\end{align}
	The lemma is proved.
\end{proof}
\subsection{Proof of Propositions~\ref{prop:standard_smc} and~\ref{prop:smc_sorted}}
The propositions follow from a direct application of Theorem~\ref{thm:webber_restated} to the Feynman-Kac model
\begin{equation*}
	\Q_K(y_{0:K}) = \frac{1}{\mathcal Z} \M_K(y_{0:K}) G_0(y_0) \prod_{k=1}^K G_k(y_{k-1}, y_k)
\end{equation*}
where we have the correspondence $y_k = x_{K-k}$ and
\begin{align*}
	\M_K(y_{0:K})     & = \mathcal N(x_K|0, \operatorname{Id}) \hat \pi(x_{K-1}|x_K) \ldots \hat \pi(x_0|x_1) \\
	G_0(y_0)          & = \hat g_K(x_K)                                                                       \\
	G_k(y_{k-1}, y_k) & = \omega_{K-k}(x_{K-k}, x_{K-k+1}).
\end{align*}

\subsection{Proof of Proposition~\ref{prop:error_K_infinite}}
\label{sec:proof_error_K_inf}
We start by providing some intuition for the result. Repeat that $X_k$ is a shorthand for $X_{k\delta}$ where $\delta$ is the discretization step size. This convention only applies for $X_k$.

When $K$ is big ($\delta$ is small), we have, as established in ~\eqref{eq:approximate_proposal_derivation},
\begin{align*}
	\pi(x_k|x_{k+1})     & \approx \mathcal N(x_k; \sqrt{1-\alpha_{k+1}} x_{k+1} + \alpha_{k+1} \gradlog g_{k+1}(x_{k+1}), \alpha_{k+1}I), \\
	\hat\pi(x_k|x_{k+1}) & = \mathcal N(x_k;\sqrt{1-\alpha_{k+1}} x_{k+1} + \alpha_{k+1} \gradlog \hat g_{k+1}(x_{k+1}), \alpha_{k+1}I).
\end{align*}
Using the analytic formula for the chi-squared divergence between two Gaussians we get
\begin{align*}
	\chisquared{\pi(x_k|x_{k+1})}{\hat\pi(x_k|x_{k+1})} \approx e^{\alpha_{k+1} \norm{\gradlog g_{k+1} - \gradlog \hat g_{k+1}}^2(x_{k+1})} - 1 \approx 2\delta \norm{\gradlog g_{k+1} - \gradlog \hat g_{k+1}}^2(x_{k+1}).
\end{align*}
Thus
\begin{align*}
	\zeta_{K}^2 & = \chisquared{\pi_K}{\mathcal N(0, \operatorname{Id})} + \sum_k \int \frac{\pi_{k+1}(x_{k+1})^2}{\hat \pi_{k+1}(x_{k+1})} \chisquared{\pi(x_k|x_{k+1})}{\hat\pi(x_k|x_{k+1})} \dd x_{k+1}                      \\
	            & \approx \chisquared{\pi_K}{\mathcal N(0, \operatorname{Id})} + \sum_k \int \frac{\pi_{k+1}(x_{k+1})^2}{\hat \pi_{k+1}(x_{k+1})} 2\delta \norm{\gradlog g_{k+1} - \gradlog \hat g_{k+1}}^2(x_{k+1}) \dd x_{k+1} \\
	            & \approx \chisquared{\pi_T}{\mathcal N(0, \operatorname{Id})} + 2\int_0^T \int_{\mathcal X} \frac{\pi_t(x)^2}{\hat \pi_t(x)} \norm{\gradlog g_t(x) - \gradlog \hat g_t(x)}^2 \dd x \dd t.
\end{align*}
\subsubsection{Regularity conditions for Proposition~\ref{prop:error_K_infinite}}
\label{apx:regularity_conditions}
We assume that the sequence of distributions $\pi_t(\cdot)$ satisfy the following properties.
\begin{assumption}
	\label{asp:order1}
	There exists $M_1 > 0$ such that $\norm{\gradlog \pi_t(x_t)} \leq M_1(1 + \norm{x_t})$.
\end{assumption}
\begin{assumption}
	\label{asp:order23}
	There exist $M_2 > 0$, $M_3 > 0$, $\alpha_2 \geq 1$, and $\alpha_3 \geq 1$ such that $\norm{\nabla^2\log \pi_t(x_t)} \leq M_2(1+\norm{x_t})^{\alpha_2}$ and $\norm{\nabla^3 \log\pi_t(x_t)} \leq M_3(1 + \norm{x_t})^{\alpha_3}$.
\end{assumption}
\begin{assumption}
	\label{asp:tail}
	There exist $\vartheta>0$ and $M_\infty < \infty$ such that $\int \pi_t(x_t) e^{\vartheta \norm{x_t}^2} \dd x_t < M_\infty$, $\forall t$.
\end{assumption}

These assumptions are satisfied, for example, when the target distribution $\pi_0$ is Gaussian.

\newcommand{\opnorm}[1]{\norm{#1}_{\operatorname{op}}}
\begin{remark}
	Let $\varphi(x_t) = \gradlog \pi_t(x_t)$. Then the notation $\nabla^3\log\pi_t(x_t)$ refers to the second order differential $\varphi''(x_t)$ which is a bilinear mapping from $\mathbb R^d \times \mathbb R^d$ to $\mathbb R^d$. For any multilinear operator $H: \mathcal X_1 \times \ldots \times \mathcal X_n \to \mathcal Y$, we define
	\begin{equation*}
		\opnorm{H}:= \sup_{x_1, \ldots, x_n \neq 0} \frac{\norm{H(x_1, \ldots, x_n)}}{\norm{x_1} \ldots \norm{x_n}}.
	\end{equation*}
	By writing $\norm{H}$ we implicitly refer to $\opnorm{H}$. In fact, the space of such operators is of finite dimensions in our cases of interests, hence any two norms are bounded by a constant factor of each other.
\end{remark}
The above assumptions only concern the differential of $\gradlog \pi_t(x_t)$ with respect to $x$. The following lemma derives a bound with respect to $t$.
\begin{lemma}
	\label{lem:partial_t_gradlog}
	Under the above assumptions, there exist constants $\bar M_1$ and $\bar \alpha_1$ such that for all $t$
	\begin{equation*}
		\norm{\frac{\partial}{\partial t} \gradlog \pi_t(x_t)} \leq \bar M_1 (1+\norm{x_t})^{\bar \alpha_1}.
	\end{equation*}
\end{lemma}
\begin{proof}
	Using the Fokker--Planck equation for the score \citep{lai2022regularizing}, we write
	\begin{equation*}
		\partial_t \log g_t(x_t) = \operatorname{div}_x \gradlog g_t(x_t) + \gradlog g_t(x_t) \circ (\gradlog g_t(x_t)-x_t).
	\end{equation*}
	For a fixed $t$, put $\varphi(x_t) = \gradlog g_t(x_t)$ and $\psi(x_t)=\Tr(\varphi'(x_t)) + \varphi(x_t) \circ (\varphi(x_t) - x_t)$. Then $\partial_t \gradlog g_t(x_t) = \nabla\psi(x_t)$. Viewing $\psi'(x_t)$, $\varphi'(x_t)$, and $\varphi''(x_t)$ as elements of $\mathcal L(\mathbb R^d, \mathbb R)$, $\mathcal L(\mathbb R^d, \mathbb R^d)$, and $\mathcal L(\mathbb R^d, \mathcal L(\mathbb R^d, \mathbb R^d))$ respectively; where $\mathcal L(A,B)$ is the space of linear operators from $A$ to $B$; we can write
	\begin{equation*}
		\psi'(x_t)h = \Tr(\varphi''(x_t)h) + (\varphi'(x_t)h)\circ(\varphi(x_t)-x_t) + \varphi(x_t)\circ(\varphi'(x_t)h-h), \forall h \in \mathbb R^d
	\end{equation*}
	where $\circ$ stands for the usual scalar product between two vectors in $\R^d$.
	There is a constant $C$ depending on the dimension such that $\Tr(L) \leq C\opnorm{L}$ for endomorphisms $L$. Thus
	\begin{align*}
		\opnorm{\psi'(x_t)h} & \leq C\opnorm{\varphi''(x_t)} \norm{h} + \opnorm{\varphi'(x_t)} \norm{\varphi(x_t)-x_t} \norm{h} + \norm{\varphi(x_t)} \opnorm{\varphi'(x_t)-\operatorname{Id}} \norm{h} \\
		                     & \leq \bar M_1(1+\norm{x_t})^{\bar \alpha_1} \norm{h}
	\end{align*}
	for some $\bar M_1$ and $\bar \alpha_1$ by Assumptions~\ref{asp:order1} and~\ref{asp:order23}. This entails $\norm{\partial_t \gradlog g_t(x_t)} = \opnorm{\psi'(x_t)} \leq \bar M_1(1+\norm{x_t})^{\bar \alpha_1}$.
\end{proof}

\subsubsection{Formal proof}
To formalize the error of the heuristic approximations presented at the beginning of Section~\ref{sec:proof_error_K_inf} we need the following technical lemma.
\newcommand{\xxa}{X^{\mathrm{A}}}
\newcommand{\xxb}{X^{\mathrm{B}}}
\newcommand{\wwa}{W^{\mathrm{A}}}
\newcommand{\wwb}{W^{\mathrm{B}}}
\newcommand{\wtm}{\widetilde M}
\newcommand{\ppa}{\P_{\mathrm{A}}}
\newcommand{\ppb}{\P_{\mathrm{B}}}
\newcommand{\xbzt}{\ppb(\dd x_{[0:t]})}
\newcommand{\xazt}{\ppa(\dd x_{[0:t]})}
\begin{lemma}
	\label{lem:chisq_fundamental_bound}
	Let $x_0$ and $v$ be two vectors in $\mathbb R^d$ and suppose that $\xxa_t$ and $\xxb_t$ are respectively the unique solutions of the SDEs:
	\begin{align*}
		 & \mathrm{(A):} \quad \dd \xxa_t = (-\xxa_t + 2v)\dd t + \sqrt 2 \dd \wwa_t, \quad X_0 = x_0          \\
		 & \mathrm{(B):} \quad \dd \xxb_t = (-\xxb_t+2f_t(\xxb_t)) \dd t + \sqrt 2 \dd \wwb_t, \quad X_0 = x_0
	\end{align*}
	where there exist strictly positive constants $M_1$, $\bar M_1$, $M_2$, and $M_3$; and strictly greater than $1$ constants $\bar \alpha_1$, $\alpha_2$, and $\alpha_3$; such that the function $f_t(x_t)$ satisfies $\norm{f_t(x_t)} \leq M_1(1+\norm{x_t})$, $\norm{\partial_t f_t(x_t)} \leq \bar M_1(1+\norm{x_t})^{\bar \alpha_1}$, and $\norm{\nabla^i f_t(x_t)} \leq M_{i+1}(1+\norm{x_t})^{\alpha_{i+1}}$ for $i \in \px{1,2}$. (The notation $\nabla$ refers implicitly to the gradient with respect to $x$.) Denote $\ppa(\dd x_{[0,T]})$ and $\ppb(\dd x_{[0:T]})$ respectively the path measures associated with the solutions of (A) and (B). Then, there exist a parameter $\widetilde M$ depending on all the aforementioned constants and a parameter $\widetilde M_1$ depending only on $M_1$ such that for any $t \leq 1/\widetilde M_1$, the chi-squared divergence of $\ppb(\dd x_{[0,t]})$ with respect to $\ppa(\dd x_{[0,t]})$ is finite, and
	\begin{equation*}
		\abs{\chisquared{\xbzt}{\xazt}- 2t\norm{f_0(x_0) -v}^2} \leq \widetilde Mt^2 e^{t\widetilde M_1(\norm{x_0}^2 + \norm{v}^2)} \pr{1 + \norm{x_0} + \norm{v}}^{4(1+\bar\alpha_1 + \alpha_2 + \alpha_3)}
	\end{equation*}
\end{lemma}
\newcommand{\spito}[2]{\langle #1, #2 \rangle}
\begin{proof}
	Put $\Delta_t(x_t) = 2f_t(x_t) - 2v$. By an application of Girsanov's theorem \citep[Example 3, Section 6.2.3]{liptser1977book}, we have
	\begin{equation*}
		D_t := \frac{\dd \ppb}{\dd \ppa}(X_{[0,t]}) = \exp{\int_0^t \frac{\spito{\Delta_s(X_s)} {\dd \wwa_s}}{\sqrt 2} - \frac 14 \int_0^t \norm{\Delta_s(X_s)}^2 \dd s}
	\end{equation*}
	where, for a vector-valued process $V_t$, the notation $\int_0^t \spito{V_s}{\dd W_s} := \sum_{i=1}^d \int_0^t V_s^i \dd W_s^i$.
	As a preliminary step, we would like to bound $\E_{\operatorname{A}}[D_t^{\alpha} (1+\norm{X_t} + \norm{v})^n]$ for some $\alpha, n \geq 1$. Here, $c(\cdot)$ denotes a constant whose value might change from line to line and depends on the variables inside the round bracket. We also drop the subscript/superscript A from $\E_{\operatorname{A}}$ and $\wwa_t$ whenever there is no risk of confusion. We have
	\newcommand{\ideltasws}{\int_0^t \spito{\Delta_s}{\dd W_s}}
	\newcommand{\ideltaswsbis}{\int_0^t \frac{\spito{\Delta_s}{\dd W_s}}{\sqrt 2}}
	\newcommand{\inormd}{\int_0^t \norm{\Delta_s}^2 \dd s}
	\newcommand{\inormx}{\int_0^t \norm{X_s}^2 \dd s}
	\newcommand{\polynormterm}{(1 + \norm{X_t} + \norm{v})^n}
	\begin{equation*}
		\begin{split}
			&\E[D_t^\alpha (1+\norm{X_t} + \norm{v})^n] = \E\ps{\exp{\alpha \ideltaswsbis - \frac \alpha 4 \inormd}\polynormterm} \\
			&= \E\ps{\exp{\alpha \ideltaswsbis - \frac{\alpha^2}{2} \inormd} \exp{\pr{
						\frac{\alpha^2}{2} - \frac \alpha 4
					} \inormd}\polynormterm} \\
			&\leq e^{t(c(M_1, \alpha) + c(\alpha) \norm{v}^2)} \E\bigg[\exp{\alpha \ideltaswsbis - \frac{\alpha^2}{2} \inormd} \exp{c(M_1, \alpha) \inormx} \times \\ &\times \polynormterm\bigg] \text{ using } \norm{\Delta_s}^2 \leq c(M_1)(1 + \norm{x_s}^2) + 8\norm{v}^2 \\
			&\leq e^{t(c(M_1, \alpha) + c(\alpha) \norm{v}^2)} \E^{1/2}\ps{\exp{\sqrt 2 \alpha \ideltasws - \alpha^2 \inormd}} \times \\ &\times \E^{1/4}\ps{\exp{4c(M_1, \alpha) \inormx}} \E^{1/4}\ps{\polynormterm}
		\end{split}
	\end{equation*}
	using double Cauchy-Schwarz $\E[XYZ] \leq \E^{1/2}[X^2] \E^{1/4}[Y^4] \E^{1/4}[Z^4]$.

	In the last line of the above display, the first expectation is equal to $1$ by the same Girsanov argument as \citet[Example 3, Section 6.2.3]{liptser1977book}. To bound the second expectation, we note that under $\P_{\operatorname{A}}$, we have $X_s \sim \mathcal N(\sqrt{1-\lambda_s} (x_0-2v)+2v, \lambda_s)$. Elementary calculations yield the bound
	\begin{equation*}
		\E[e^{k\norm{X_s}^2}] = \pr{\frac{1}{\sqrt{1-2k\lambda_s}}}^d \exp{\frac{k\norm{
					\sqrt{1-\lambda_s}(x_0-2v) + 2v
				}^2}{1-2k\lambda_s}} \leq (\sqrt 2)^d e^{16k(\norm{x_0}^2 + \norm{v}^2)}
	\end{equation*}
	for $0<k<1/4$. Write
	\begin{equation*}
		\begin{split}
			\E\ps{\exp{4c(M_1, \alpha) \inormx}} \leq \frac 1t \int_0^t \E\ps{e^{4tc(M_1, \alpha) \norm{X_s}^2}} \dd s \leq c(1) e^{64tc(M_1, \alpha) (\norm{x_0}^2 + \norm{v}^2)}
		\end{split}
	\end{equation*}
	if $t \leq \frac{1}{16c(M_1, \alpha)}$, using Jensen's inequality and the above bound. The third expectation can be bounded by $c(n)(1 + \norm{x_0}^{4n} + \norm{v_0}^{4n})$.

	Putting everything together, we establish that there exist constants $c(n)$ and $c(M_1, \alpha)$ such that
	\begin{equation}
		\label{eq:fundamental_bound}
		\E[D_t^\alpha \polynormterm] \leq c(n)e^{tc(M_1, \alpha) (\norm{x_0}^2 + \norm{v}^2)} (1 + \norm{x_0}^{4n} + \norm{v}^{4n}), \forall t \leq \frac{1}{c(M_1, \alpha)}.
	\end{equation}
	Now to study $\chisquared{\xbzt}{\xazt}$, we apply Ito's formula to $D_t^2$ and get, under $\P_{\operatorname{A}}$
	\begin{equation}
		\label{eq:ito1}
		D_t^2 = 1 + \sqrt 2 \int_0^t D_s^2 \spito{\Delta_s}{\dd W_s} + \int_0^t \frac{D_s^2 \norm{\Delta_s}^2}{2} \dd s.
	\end{equation}
	Putting $\eta_t(x_t) = \norm{\Delta_t(x_t)}^2$ and $\tilde f_t(x_t) = - x_t + 2v$, we have
	\begin{multline}
		\label{eq:ito2}
		D_t^2 \eta_t = \eta_0 + \sqrt 2 \int_0^t D_s^2\spito{\eta_s \Delta_s + \nabla \eta_s}{\dd W_s} + \\
		+ \int_0^t D_s^2 \px{\eta_s \frac{\norm{\Delta_s}^2}{2} + \nabla \eta_s(\tilde f(X_s) + 2 \Delta_s) + \frac{\partial \eta}{\partial s} + \Tr(\nabla^2 \eta_s)} \dd s.
	\end{multline}
	To study \eqref{eq:ito1} and \eqref{eq:ito2}, we use \eqref{eq:fundamental_bound} together with the following bounds which are consequences of the lemma's assumptions:
	\begin{align*}
		\norm{\Delta_s(x_s)}                           & \leq c(M_1)(1 + \norm{x_s} + \norm{v})                                       \\
		\eta_s(x_s)                                    & \leq c(M_1) (1 + \norm{x_s} + \norm{v})^2                                    \\
		\norm{\nabla \eta_s(x_s)}                      & \leq c(M_1, M_2) (1 + \norm{x_s} + \norm{v})^{1 + \alpha_2}                  \\
		\norm{\frac{\partial\eta}{\partial s}(s, X_s)} & \leq c(M_1, \bar M_1)(1 + \norm{x_s} + \norm{v})^{1 + \bar \alpha_1}         \\
		\norm{\Tr(\nabla^2 \eta_s)}                    & \leq c(M_1, M_2, M_3) (1 + \norm{x_s} + \norm{v})^{1 + \alpha_2 + \alpha_3}.
	\end{align*}
	The last inequality is justified by the fact that $\Tr(\nabla^2 \eta_s) \lesssim \norm{\nabla^2 \eta_s}$, where by considering $\nabla^2 \eta_s(x_s)$ as the second differential of $\eta_s$ at $x_s$ (i.e. a bilinear form from $\mathbb R^d \times \mathbb R^d$ to $\mathbb R$), we have
	\begin{equation*}
		\frac{\partial^2\eta_s}{\partial x^2}(x)[h, k] = 2\ps{\Delta_s(x)\circ \frac{\partial^2\Delta}{\partial x^2}(x)[h,k] + (\frac{\partial\Delta}{\partial x}(x)h) \circ (\frac{\partial\Delta}{\partial x}(x)k)}, \forall (h, k) \in \mathbb R^d \times \mathbb R^d
	\end{equation*}
	where $\circ$ stands for the usual scalar product between two vectors in $\R^d$.
	These bounds show that the stochastic integrals (w.r.t. $\dd W_s$) in~\eqref{eq:ito1} and~\eqref{eq:ito2} are true martingales (as opposed to merely local martingales). Moreover, there exist a constant $\wtm$ depending on $M_1$, $M_2$, $M_3$, $\bar M_1$, $\bar \alpha_1$, $\alpha_2$, and $\alpha_3$; and a constant $\wtm_1$ depending on $M_1$ only such that
	\newcommand{\ssnxonv}{\norm{x_0}^2 + \norm{v}^2}
	\newcommand{\vtailterm}{(1+\norm{x_0} + \norm{v})^{4(1+\bar\alpha_1+\alpha_2+\alpha_3)}}
	\begin{multline*}
		\E\ps{D_s^2 \abs{\eta_s \frac{\norm{\Delta_s}^2}{2} + \nabla \eta_s(\tilde f(X_s) + 2 \Delta_s) + \frac{\partial \eta}{\partial s} + \Tr(\nabla^2 \eta_s)}} \leq 4\wtm e^{t\wtm_1(\ssnxonv)} \times \\ \times \vtailterm, \forall s\leq t\leq \frac{1}{\wtm_1}.
	\end{multline*}
	Taking expectation of both sides of~\eqref{eq:ito2} and rearranging yields
	\begin{equation*}
		\abs{\E[D_t^2 \eta_t] - \eta_0} \leq 4\mathbf{t}\wtm e^{t\wtm_1(\ssnxonv)}  \vtailterm, \forall t\leq \frac{1}{\wtm_1}.
	\end{equation*}
	Then we have, by taking expectation of both sides of~\eqref{eq:ito1}:
	\begin{equation*}
		\begin{split}
			&\abs{\chisquared{\xbzt}{\xazt} - t\frac{\eta_0}{2}} = \abs{\E(D_t^2) - 1 - t\frac{\eta_0}{2}} = \abs{\int_0^t \pr{\E\ps{\frac{D_s^2 \eta_s}{2}} - \frac{\eta_0}{2}} \dd s} \leq \int_0^t \abs{\E\ps{\frac{D_s^2\eta_s}{2}} - \frac{\eta_0}{2}} \dd s \\
			&\leq \int_0^t 2s\wtm e^{t\wtm_1(\ssnxonv)} \vtailterm \dd s \\
			&= t^2 \wtm e^{t\wtm_1(\ssnxonv)} \vtailterm.
		\end{split}
	\end{equation*}
	The proof is completed.
\end{proof}
We are now ready to give the proof of Proposition~\ref{prop:error_K_infinite}.
\begin{proof}
	To make the arguments clearer, we shall assume that the proposal distribution is
	\begin{equation*}
		\hat\pi(x_k|x_{k+1}) = \mathcal N(x_k|\sqrt{1-\alpha_{k+1}} x_{k+1} + 2(1-\sqrt{1-\alpha_{k+1}}) \gradlog \hat g_{k+1}(x_{k+1}), \alpha_{k+1}I)
	\end{equation*}
	which is slightly different from~\eqref{eq:smc_proposal}. As we will see, the proof also applies to the original discretization with minimal changes. For $0<s<u<T$, the distributions $\hat \pi(x_s|x_u)$ and $\pi(x_s|x_u)$ can be obtained by respectively solving the following SDEs between times $T-u$ and $T-s$:
	\newcommand{\yya}{Y^{\mathrm{A}}}
	\newcommand{\yyb}{Y^{\mathrm{B}}}
	\begin{align*}
		 & \hat\pi(x_s|x_u):\quad \dd \yya_t = (-\yya_t + 2\gradlog \hat g_{T-\mathbf u}(\yya_t)) \dd t + \sqrt 2 \dd \wwa_t,\quad \yya_{T-u}=x_u \\
		 & \pi(x_s|x_u):\quad \dd \yyb_t = (-\yyb_t + 2\gradlog g_{T-t}(\yyb_t)) \dd t + \sqrt 2 \dd \wwb_t,\quad \yyb_{T-u}=x_u.
	\end{align*}
	The assumptions in Section~\ref{apx:regularity_conditions} and Lemma~\ref{lem:partial_t_gradlog} show that the conditions of Lemma~\ref{lem:chisq_fundamental_bound} are satisfied for this pair of SDEs. Thus
	\begin{multline*}
		\abs{\chisquared{
				\ppb(\dd y_{[T-u, T-s]})
			}{
				\ppa(\dd y_{[T-u, T-s]})
			} - 2(u-s)\norm{\gradlog g_u(x_u) - \gradlog \hat g_u(x_u)}^2} \leq \wtm(u-s)^2 \times \\ \times e^{(u-s)\wtm_1 (\norm{x_u}^2 + \norm{\gradlog \hat g_u(x_u)}^2)} (1+\norm{x_u}+\norm{\gradlog \hat g_u(x_u)})^{\alpha_+}
	\end{multline*}
	for $\alpha_+ = 4(1+\bar\alpha_1+\alpha_2+\alpha_3)$ and $u-s\leq \frac{1}{\wtm_1}$. This, together with the data processing inequality and the assumption on $|\gradlog \hat g_t|$, implies
	\begin{multline*}
		\chisquared{\pi(x_k|x_{k+1})}{\hat \pi(x_k|x_{k+1})} \leq 2\delta\norm{\gradlog g_{k+1}-\gradlog \hat g_{k+1}}^2(x_{k+1}) + \wtm \delta^2 e^{\delta\wtm_1(1+2C_2^2)(1+\norm{x_{k+1}})^2} \times \\ \times (1+C_2)^{\alpha_+}(1+\norm{x_{k+1}})^{\alpha_+}, \forall \delta \leq \frac{1}{\wtm_1}.
	\end{multline*}
	Thus, for a sufficiently fine discretization,
	\begin{multline}
		\label{eq:last_bound}
		\sum_k \int \frac{\pi_{k+1}(x_{k+1})^2}{\hat \pi_{k+1}(x_{k+1})} \chisquared{\pi(x_k|x_{k+1})}{\hat\pi(x_k|x_{k+1})} \dd x_{k+1} \leq \\ \leq \sum_k \delta \int \frac{\pi_{k+1}(x_{k+1})^2}{\hat\pi_{k+1}(x_{k+1})} 2 \norm{\gradlog g_{k+1}-\gradlog \hat g_{k+1}}^2(x_{k+1}) \dd x_{k+1} + \\ + \sum_k\int C_1 \pi_{k+1}(x_{k+1})\wtm \delta^2 e^{\delta \wtm_1 (1+2C_2^2)(1+\norm{x_{k+1}}^2)} (1+C_2)^{\alpha_+} (1+\norm{x_{k+1}})^{\alpha_+} \dd x_{k+1}.
	\end{multline}
	The first term is a Riemann sum and converges to $\int_0^T \int \frac{\pi_t(x_t)^2}{\hat \pi_t(x_t)} 2 \norm{\gradlog g_t - \gradlog\hat g_t}^2(x_t) \dd x_t \dd t$. To bound the second term, first note that
	\newcommand{\epikpo}{\E_{\pi_{k+1}}}
	\begin{align*}
		\epikpo[(1 + \norm{X})^{2\alpha_+}] & \leq 2^{2\alpha_+-1} \E\ps{1 + \norm{X}^{2\alpha_+}} \leq 2^{2\alpha_+-1} \E\ps{1 +
			e^{\vartheta \norm{X}^2} \max\pr{\frac{\lceil \alpha_+\rceil !}{\vartheta^{\lceil \alpha_+ \rceil}}, \frac{\lfloor \alpha_+ -1\rfloor !}{\vartheta^{\lfloor \alpha_+ -1 \rfloor }}}
		}                                                                                                                         \\ &\leq 2^{2\alpha_+-1} \pr{1 + M_\infty\max\pr{\frac{\lceil \alpha_+\rceil !}{\vartheta^{\lceil \alpha_+ \rceil}}, \frac{\lfloor \alpha_+ -1\rfloor !}{\vartheta^{\lfloor \alpha_+ -1 \rfloor }}}}
	\end{align*}
	where $M_\theta$ and $\vartheta$ appear in Assumption~\ref{asp:tail}. Thus, as long as $\delta \wtm_1 (1+2C_2^2) \leq \vartheta/2$, it holds that
	\begin{align*}
		\int \pi_{k+1}(x) e^{\delta \wtm_1(1+2C_2^2) \norm{x}^2}(1+\norm{x})^{\alpha_+} \dd x & \leq \epikpo\ps{e^{\vartheta \norm{X}^2/2} (1+\norm{X})^{\alpha_+}}                                                      \\
		                                                                                      & \leq \E^{1/2}\ps{e^{\vartheta \norm{X}^2}} \E^{1/2}\ps{(1 + \norm{X})^{2\alpha_+}} \leq C(M_\infty, \alpha_+, \vartheta)
	\end{align*}
	for some constant $C(M_\infty, \alpha_+)$ depending on $M_\infty$, $\alpha_+$, and $\vartheta$. Hence the second sum of~\eqref{eq:last_bound} tends to $0$ when $\delta\to0$. The proof is finished.
\end{proof}

\subsection{Proof of Proposition~\ref{prop:scorematching}}
\label{proof_log_g_res}
The denoising score matching identity is standard and recalled here for convenience. We have
\begin{align*}
	\pi_k(x_k)=\int p(x_k|x_0) \pi_0(x_0)  \rmd x_0
\end{align*}
so by using the log derivative we obtain
\begin{align}
	\nabla \log \pi_k(x_k) & =\int \nabla \log p(x_k|x_0) \frac{\pi_0(x_0) p(x_k|x_0)}{\pi_k(x_k)} \rmd x_0 \\
	                       & =\int \nabla \log p(x_k|x_0) \pi(x_0|x_k) \rmd x_0. \label{eq:standard_score}
\end{align}
It can be easily verified that the interchange of differentiation and integration here does not require any regularity assumption on $\pi_0(x_0)$ apart from differentiability.

To prove the novel score identity, we first note that, under the condition $\int \norm{\nabla \pi(x_0)} e^{-\eta\norm{x_0}^2} \dd x_0 < \infty, \forall \eta > 0$, we have
\begin{equation}
	\label{eq:stein}
	\int \nabla_{x_0}\log \pi(x_0|x_k) \pi(x_0|x_k) \dd x_0 = 0
\end{equation}
according to Lemma~\ref{lem:stein}. Combining this identity with~\eqref{eq:standard_score}, we have, for any $\alpha \in \mathbb R$,
\begin{equation*}
	\gradlog \pi(x_k) = \int \ps{\nabla_{x_k} \log p(x_k|x_0) + \alpha\nabla_{x_0} \log \pi(x_0|x_k)} \pi(x_0|x_k) \dd x_0.
\end{equation*}
In particular:
\begin{itemize}
	\item For $\alpha = \frac{1}{\sqrt{1-\lambda_k}}$, we get
	      \begin{equation*}
		      \gradlog \pi(x_k) = \frac{1}{\sqrt{1-\lambda_k}} \int \gradlog \pi(x_0) \pi(x_0|x_k) \dd x_0
	      \end{equation*}
	      which is the identity presented in Appendix C.1.3 \citep{debortoli2021diffusion};
	\item For $\alpha = \sqrt{1-\lambda_k}$, we get
	      \begin{equation*}
		      \gradlog \pi(x_k)=  \int\pr{\sqrt{1-\lambda_k} \gradlog g_0(x_0) - x_k} \pi(x_0|x_k) \dd x_0
	      \end{equation*}
	      which is the identity we wanted to prove.
\end{itemize}
The verifications are straightforward by remarking that $\nabla_{x_0} \log \pi(x_0|x_k) = \gradlog g_0(x_0) + \nabla_{x_0} \log p(x_0|x_k)$. We also note that choosing $\alpha = 0$ brings us back to the classical score matching loss. Therefore, different values of $\alpha$ give losses with different properties.

We finish this section with a technical lemma giving conditions for~\eqref{eq:stein} to hold. The identity is a particular case of what is known in the literature as zero-variance control variates and Stein's control variates \citep{assaraf1999zerovariance,mira2013zerovariance,anastasiou2023steinmethod}.
\begin{lemma}
	\label{lem:stein}
	Let $f: \mathbb R^d \to \mathbb R$ be a probability density, i.e. $f \geq 0$ and $\int f(x) \dd x = 1$. Suppose that $f$ is continuously differentiable and $\int \norm{\nabla f(x)} \dd x < \infty$. Then $\int \nabla f(x) \dd x = 0$.
\end{lemma}
\begin{remark}
	The condition $\int \norm{\nabla f(x)} \dd x < \infty$ is clearly the minimum necessary for $\int \nabla f(x) \dd x =0$ to make sense. On the other hand, we do \textit{not} explicitly require that $f$ or $\nabla f$ vanishes at infinity.
\end{remark}
\begin{proof}
	Without loss of generality, we only prove that $\int \partial_1f(x) \dd x =0$. Put $g(x_1):= \int f(x_1, x_{2:d}) \dd x_{2:d}$. Fubini's theorem then implies that $\int_{\mathbb R} g(x_1)\dd x_1 = 1$. We have
	\begin{equation}
		\label{eq:intgrad}
		\begin{split}
			\int_{\mathbb{R}^d} \partial_1f(x)\dd x &= \lim_{M\to\infty} \int_{\mathbb R^{d-1}} \int_{-M}^M \partial_1f(x_1, x_{2:d}) \dd x_1 \dd x_{2:d} = \lim_{M\to\infty} \int_{\R^{d-1}} f(M, x_{2:d}) - f(-M, x_{2:d}) \dd x_{2:d} \\
			&= \lim_{M\to\infty} g(M) - g(-M).
		\end{split}
	\end{equation}
	Put $I(M):= \int_0^\infty \abs{g(M+x) - g(-M-x)} \dd x$. We have
	\begin{equation*}
		I(M) \leq \int_0^\infty \abs{g(M+x)} + \abs{g(-M-x)} \dd x,
	\end{equation*}
	thus $\lim_{M\to\infty} I(M) = 0$ by the integrability of $g$. Combining this with Fatou's lemma yields
	\begin{equation*}
		0=\liminf_{M\to\infty} I(M) \geq\int_0^\infty \liminf_{M\to\infty} \abs{g(M+x) - g(-M-x)} \dd x = \int_0^\infty \abs{\int_{\mathbb R^d} \partial_1f(y)\dd y} \dd x
	\end{equation*}
	by~\eqref{eq:intgrad}. This means that $\int_{\mathbb R^d} \partial_1f(y)\dd y = 0$.
\end{proof}

\newcommand{\ceh}[1]{\CE{#1}{X_t}}
\newcommand{\cehbis}[1]{\CE{#1}{X_t, \tau=t}}
\newcommand{\veh}[1]{\VarE{#1}{X_t}}
\subsection{Proof of Proposition~\ref{prop:guidance_loss_is_good}}
\begin{proof}
	Since we are in the Gaussian case with $d=1$, we have $\gradlog g_0(x_0) = ax_0 + b$ for some $a,b\in \mathbb R$. Therefore $\ell_{\textup{NSM}}(\theta)$ and $\operatorname{Var}(\hat \ell_{\textup{DSM}}(\theta))$ are trivially bounded. To study $\ell_{\textup{DSM}}(\theta)$, we first note that
	\begin{equation*}
		\operatorname{Var}(X_0|X_t) = \frac{\lambda_t \sigma^2}{\lambda_t + (1-\lambda_t) \sigma^2} =: \rho_t.
	\end{equation*}
	Write
	\begin{equation*}
		\begin{split}
			\ceh{\norm{s_\theta(t, X_t) - \gradlog p(X_t|X_0)}^2} &\geq \veh{s_\theta(t, X_t) - \gradlog p(X_t|X_0)} = \veh{-\frac{X_t-\sqrt{1-\lambda_t} X_0}{\lambda_t}} \\
			&= \frac{1-\lambda_t}{\lambda_t^2} \rho_t
		\end{split}
	\end{equation*}
	so
	\begin{equation*}
		\ell_{\textup{DSM}}(\theta) \geq \int_0^T \frac{1-\lambda_t}{\lambda_t^2} \rho_t \dd t = \infty
	\end{equation*}
	since $\rho_t \sim 2t$ as $ t \to 0$.
	Concerning $\hat \nabla \ell_{\textup{DSM}}(\theta) = 2T(s_\theta(\tau, X_\tau) - \gradlog p(X_\tau|X_0))\nabla_\theta s_\theta(\tau, X_\tau)$, we have
	\begin{equation*}
		\begin{split}
			\cehbis{\norm{\hat \nabla \ell_{\textup{DSM}}(\theta)}^2} &= 4T^2\norm{\nabla_\theta s_\theta(t, X_t)}^2 \cehbis{\norm{s_\theta(t, X_t) - \gradlog p(X_t|X_0)}^2} \\
			&\geq 4T^2\norm{\nabla_\theta s_\theta(t, X_t)}^2 \frac{1-\lambda_t}{\lambda_t^2} \rho_t
		\end{split}
	\end{equation*}
	so
	\begin{equation*}
		\begin{split}
			\E\ps{\norm{\hat \nabla \ell_{\textup{DSM}}(\theta)}^2} &= \frac 1T \int_0^T \CE{\norm{\hat \nabla \ell_{\textup{DSM}}(\theta)}^2}{\tau=t} \dd t \geq 4T \int_0^T \E\ps{\norm{\nabla_\theta s_\theta(t, X_t)}^2 \frac{1-\lambda_t}{\lambda_t^2} \rho_t} \dd t \\
			&= 4T \E\ps{\int_0^T \norm{\nabla_\theta s_\theta(t, X_t)}^2 \frac{1-\lambda_t}{\lambda_t^2} \rho_t \dd t}.
		\end{split}
	\end{equation*}
	The integral inside the last expectation is infinite whenever the event $\norm{\nabla_\theta s_\theta(0, X_0)} \neq 0$ holds, thanks to the continuity of $\nabla_\theta s$ and the path $X_{[0,T]}$. Since $\E\norm{\nabla_\theta s_\theta(t, X_0)}^2 > 0$ by assumption, that event has non-zero probability, which concludes the proof.
\end{proof}

\section{Experimental Details}
\label{apx:experimental}

In this section we give additional details and ablations relating to our experimental results. We begin by providing details of the sampling tasks we considered. We then provide details of our implementation and the baseline methods. We finally provide additional ablation studies and results which demonstrate the properties of our method.

\subsection{Benchmarking targets}
\label{apx:exp_tasks}

\paragraph{Gaussian}
Here we consider the target $\pi(x) = \mathcal{N}(x; 2.75, 0.25^2)$.

\paragraph{Mixture} This target was used in \citet{arbel2021annealed}. It is an equally weighted mixture of 6 bivariate Gaussian distributions with means $\mu_1=(3.0,0.0), \mu_2=(-2.5, 0.0), \mu_3=(2.0, 3.0), \mu_4=(0.0, 3.0), \mu_5=(0.0, -2.5), \mu_6=(3.0,2.0)$ and covariances $\Sigma_1=\Sigma_2=\big(\begin{smallmatrix}
			0.7 & 0.0\\
			0.0 & 0.05
		\end{smallmatrix}\big), \Sigma_4=\Sigma_5=\big(\begin{smallmatrix}
			0.05 & 0.0\\
			0.0 & 0.07
		\end{smallmatrix}\big), \Sigma_3=\Sigma_6=\big(\begin{smallmatrix}
			1.0 & 0.95\\
			0.95 & 1.0
		\end{smallmatrix}\big)$. This target is symmetric around $y=x$.

\paragraph{Funnel} This target was proposed by \citet{Neal:2003}. Its density follows $x_0 \sim \mathcal{N}(0, \sigma_f^2), x_{1:9}|x_0 \sim \mathcal{N}(0, \exp(x_0)\mathbf{I})$, with $\sigma_f=3$.

\paragraph{Logistic Regression} The Bayesian logistic regression model is defined by the prior distribution $\theta  \sim \mathcal{N}(0, \sigma^2\mathrm{I})$ and likelihood $y|\theta, x \sim \text{Bernoulli}(\sigma(\theta^T x))$ where $\sigma$ is the sigmoid function. We consider sampling the posterior $\theta| y, x$ on the Ionosphere and Sonar datasets, which are of $35$ and $61$ dimensions respectively.

\paragraph{Brownian Motion} In this task, we make noisy observations of a simple Brownian motion over 30 time steps. The model was introduced by \citet{inferencegym2020} and is defined by the prior \(\alpha_{\text{inn}} \sim \text{LogNormal}(0, 2)\), $\alpha_{\text{obs}} \sim \text{LogNormal}(0, 2)$, $x_1 \sim \mathcal{N}(0, \alpha_{\text{inn}}^2)$ and $x_i \sim \mathcal{N}(x_{i-1}, \alpha_{\text{inn}}^2)$ for $i=2, ..., 30$. The observation likelihood is given by $y_i \sim \mathcal{N}(x_i, \alpha_{\text{obs}}^2)$ for $i=1, ..., 30$. The goal is to sample the posterior distribution of $\alpha_\text{inn}, \alpha_{\text{obs}}, x_1, ..., x_{30} | \{y_i\}_{i=1}^{10} \bigcup \{y_i\}_{i=21}^{30}$. This task is in $32$ dimensions.

\paragraph{Log Gaussian Cox Process} The LGCP model \cite{Moller:1998} was developed for the analysis of spatial data. The Poisson rate parameter $\lambda(x)$ is modelled on the grid using an exponentially-transformed Gaussian process, and observations come from a Poisson point process. The unnormalized posterior density is given directly by $\gamma(x) = \mathcal{N}(x; \mu, K)\prod_{i \in [1:M]^2} \exp(x_i y_i - a \exp(x_i))$, where $x_i$ are the points of a regular $M\times M$ grid. In our experiments, we fit this model on the Pines forest dataset where $M=40$, resulting in a problem in $1600$ dimensions.

\paragraph{GMM} The challenging Gaussian Mixture Model used in \cite{midgley2022flow} is an unequally weighted mixture of 40 Gaussian components. The mean of each component is uniformly distributed in the range $[-40,40]^d$, the covariance is $\sigma^2 I_d$ where $\sigma = \log(1+\exp(0.1))$ and the unnormalized weight is uniformly distributed in $[0,1]$. We consider dimensions $d \in \{1,2,5,10,20\}$.

\subsection{Algorithmic details and hyperparameter settings}
\label{apx:exp_settings}

Here we give details of the algorithmic settings used in our experiments. We first describe the considerations taken to ensure a fair comparison between baselines, and then we detail exact hyperparameter settings.

Our method was implemented in Python using the libraries of JAX \cite{Bradbury:2018}, Haiku and Optax. Our implementation is available on Github\footnote{\url{https://github.com/angusphillips/particle_denoising_diffusion_sampler}}. We used the open source code-bases of \citet{arbel2021annealed} to run the SMC and CRAFT baselines and of \citet{vargasDDSampler2023} to run the PIS and DDS benchmarks, both of which are also implemented in JAX.

In all experiments we used $2000$ particles to estimate the normalizing constant.

\paragraph{Variational approximation} We used a variational approximation as the reference distribution for all methods. We found that this was required for numerical stability of the potential function $g_t(x_t)$ in our method. We therefore used the same variational approximation for all methods to ensure a fair comparison. Note that PIS reverses a pinned brownian motion and thus the reference distribution depends on the diffusion time span $T$ and the noise coefficient $\sigma$. Since $\sigma$ affects the performance of the PIS algorithm itself, we tune this parameter independently rather than setting this via the variational approximation. The variational approximation was a mean-field variational distribution i.e. a diagonal Gaussian distribution learnt by optimizing the ELBO. We used $20,000$ optimisation steps ($50,000$ for the \verb|Funnel| and \verb|Brownian| tasks) with the Adam optimizer \cite{adam} and learning rate $1e-3$. We did not use a variational approximation in the \verb|Gaussian| and \verb|GMM| tasks where we used $\mathcal{N}(0, 1)$ and $\mathcal{N}(0, 20^2 I_d)$ respectively.

\paragraph{Network architectures and optimizer settings} For the CRAFT baseline, we followed the flow network architectures and optimizer settings given in \citet{Matthews2022}, which are restated below for completeness. For the diffusion-based methods (PDDS, DDS and PIS) we use the same network architecture and optimizer settings for each method. The neural network follows the PISGRAD network of \citet{zhangyongxinchen2021path} with minor adaptations. We use a sinusoidal embedding of 128 dimensions for the time input. We use a 3-layer MLP with 64 hidden units per layer for the `smoothing' network ($r_\eta(t)$ in our notation and $\mathrm{NN}_2(t)$ in \citet{zhangyongxinchen2021path}). For the main potential/score network ($\mathrm{N}_\gamma$ in our notation and $\mathrm{NN}_1(t, x)$ in \citet{zhangyongxinchen2021path}), we use a 2 layer MLP of 64 hidden units per layer to encode the 128-dimensional time embedding. This is concatenated with the state input $x$ before passing through a 3-layer MLP with 64 hidden units per layer, outputting a vector of dimension $d$. In PDDS, we take the scalar product of this output with the state input $x$ to approximate the potential function, while PIS and DDS use the $d$-dimensional output to approximate the optimal control term. The activation function is GeLU throughout. We train for $10,000$ iterations of the Adam optimizer \cite{adam} with batch size $300$ and a learning rate of $1e-3$, which decays exponentially at a rate of $0.95$ every 50 iterations (with the exception of the \verb|Funnel| task where we did not use any learning rate decay).

The number of trainable parameters for each method and task can be found in \cref{tab:num_params}, along with training time and sampling time (performed on a NVIDIA GeForce GTX 1080 Ti GPU).

\begin{table}[]
	\centering
	\def\arraystretch{2.0}
	\begin{tabular}{lrrrrrrr}
		\toprule
		      & Gaussian (16)     & Mixture (16) & Funnel (32)  & Brownian (16) & Ion (32)     & Sonar (32)   & LGCP (128)   \\
		\midrule
		PDDS  & \makecell{37570 /                                                                                            \\84 / 0.10} & \makecell{37764 /\\84 / 0.13} & \makecell{39316 /\\114 / 0.15} & \makecell{43584 /\\90 / 0.19} & \makecell{44166 /\\97 / 0.12} & \makecell{49210 /\\105 / 0.12} & \makecell{347776 /\\2492 / 1.86} \\
		CRAFT & \makecell{32 /                                                                                               \\4 / 0.02} & \makecell{176608 /\\26 / 0.09} & \makecell{6077440 /\\178 / 0.18} & \makecell{1024 /\\27 / 0.21} & \makecell{2240 /\\38 / 0.09} & \makecell{3904 /\\40 / 0.09} & \makecell{409600 /\\3072 / 14.9} \\
		PIS   & \makecell{37570 /                                                                                            \\129 / 0.00} & \makecell{37764 /\\167 / 0.01} & \makecell{39316 /\\394 / 0.02} & \makecell{43854 /\\176 / 0.01} & \makecell{44166 /\\324 / 0.02} & \makecell{49210 /\\326 / 0.01} & \makecell{347776 /\\3931 / 0.64} \\
		DDS   & \makecell{37570 /                                                                                            \\136 / 0.00} & \makecell{37764 /\\187 / 0.01} & \makecell{39316 /\\381 / 0.02} & \makecell{43854 /\\183 / 0.01}  & \makecell{44166 /\\338 / 0.01} & \makecell{49210 /\\332 / 0.02} & \makecell{347776 /\\3941 / 0.68} \\
		SMC   & 0 / 0 / 0.02      & 0 / 0 / 0.07 & 0 / 0 / 0.17 & 0 / 0 / 0.20  & 0 / 0 / 0.09 & 0 / 0 / 0.09 & 0 / 0 / 14.6 \\
		\bottomrule\
	\end{tabular}
	\caption{Number of trainable parameters / training time total (seconds) / sampling time per 2000 particles (seconds). Timings are averaged over 3 training seeds. }
	\label{tab:num_params}
\end{table}

\paragraph{Annealing and noise schedules} For the annealing based approaches (SMC and CRAFT) we used a geometric annealing schedule with initial distribution the variational approximation as described above. For the diffusion based approaches (PDDS, DDS and PIS) we carefully considered the appropriate noise schedules for each method. Firstly, we fix the diffusion time span at $T=1$ and adapt the discretization step size depending on the number of steps $K$ of the experiment, i.e. $\delta = T/K$. This choice is equivalent to the fixed discretization step and varying diffusion time $T$ as considered by \citet{vargasDDSampler2023}, up to the choice of $\alpha_\text{max}$.

For PIS, the original work of \citet{zhangyongxinchen2021path} used by default a uniform noise schedule controlled by $\sigma$. Further, \citet{vargasDDSampler2023} were unable to find a noise schedule which improved performance above the default uniform noise schedule. As such as we stick with the uniform noise schedule and tune the noise coefficient $\sigma$ by optimizing the ELBO objective over a grid search.%

For DDS, it was observed by \citet{vargasDDSampler2023} that controlling the transition scale $\sqrt{\alpha_k}$ such that it goes smoothly to zero as $k \to 0$ is critical to performance. To achieve this they choose a cosine-based schedule $\alpha_k^{1/2} = \alpha_{\text{max}}^{1/2} \cos^2\big(\frac{\pi}{2} \frac{1-k/K + s}{1+s}\big)$ for $s$ small ($0.008$ following \citet{nichol2021improved}), which we term the DDS cosine schedule. The parameter $\alpha_{\text{max}}$ is tuned such that the noise at the final step in the reverse process is sufficiently small. We found that this scheduler did indeed result in the best performance when compared to a linear noise schedule ($\beta_t = \beta_0 + \beta_T t/T$) or the popular cosine schedule of \citet{nichol2021improved} ($\lambda_t = 1 - \cos^2\big(\frac{\pi}{2} \frac{t/T + s}{1+s}\big)$). As such we use the DDS cosine schedule and tune $\alpha_{\text{max}}$ by optimizing the ELBO objective over a grid search. %

For our method PDDS, we obtained the best performance using the cosine schedule of \citet{nichol2021improved}, which sets $\lambda_t = 1 - \cos^2\big(\frac{\pi}{2} \frac{t/T + s}{1+s}\big)$ where we recall that $\lambda_t = 1 - \exp(-2\int_{0}^{t} \beta_s ds)$, i.e. the variance of the transition from $0$ to $t$. We provide an illustration of the benefits of this schedule in the following section. In particular we found that the alternative DDS cosine schedule did not improve performance and added the complexity of tuning the parameter $\alpha_\text{max}$.

In summary, while each of the diffusion based approaches used different noise schedulers, each was chosen to provide the best performance for the individual approach and thus ensures a fair comparison.

\subsubsection{SMC and CRAFT settings}

We used 1 iteration of an HMC kernel with 10 leapfrog integrator steps as the proposal distribution in the SMC and CRAFT baselines. We tuned the HMC step sizes based on initial runs and obtained the step size schedules given below. We performed simple resampling when the ESS dropped below $30\%$. We trained CRAFT for $500$ iterations ($1000$ on \verb|Funnel|) with a batch size of $2000$ and learning rate schedule detailed below.  We also list the flow architecture in each task. Our parameter settings differ to those in \citet{Matthews2022} since we use the variational approximation, which results in larger MCMC step sizes and smaller learning rates.

\paragraph{Gaussian} Step sizes $[0.7, 0.7, 0.5, 0.4]$ linearly interpolated between times $[0.0, 0.25, 0.5, 1.0]$. CRAFT used a diagonal affine flow, with a learning rate of $1e-2$.
\paragraph{Mixture} Step sizes $[0.5, 0.5, 0.5, 0.3]$ linearly interpolated between times $[0.0, 0.25, 0.5, 1.0]$. CRAFT used a spline inverse autoregressive flow with 10 spline bins and a 3 layer autoregressive MLP of 30 hidden units per layer, with a learning rate of $1e-3$.
\paragraph{Funnel} Step sizes $[1.0, 0.9, 0.8, 0.7, 0.6]$ linearly interpolated between times $[0.0, 0.25, 0.5, 0.75, 1.0]$. CRAFT used an affine inverse autoregressive flow, trained for 4000 iterations with a learning rate of $1e-3$.
\paragraph{Brownian} Step sizes $[0.8,0.8,0.7,0.6,0.5]$ linearly interpolated between times $[0.0, 0.25, 0.5, 0.75, 1.0]$. CRAFT used a diagonal affine flow, with a learning rate of $1e-3$.
\paragraph{Ion} Step sizes $[0.7, 0.7, 0.6, 0.5, 0.4]$ linearly interpolated between times $[0.0, 0.1, 0.25, 0.5, 1.0]$. CRAFT used a diagonal affine flow, with a learning rate of $1e-3$.
\paragraph{Sonar} Step sizes $[0.7, 0.7, 0.6, 0.5, 0.35]$ linearly interpolated between times $[0.0, 0.1, 0.25, 0.5, 1.0]$. CRAFT used a diagonal affine flow, with a learning rate of $1e-3$.
\paragraph{LGCP} Step sizes $[0.35, 0.35, 0.3, 0.2]$ linearly interpolated between times $[0.0, 0.25, 0.5, 1.0]$. CRAFT used a diagonal affine flow, with a learning rate of $1e-4$.
\paragraph{GMM1} Step sizes $[5, 4, 3, 2.8, 2.5]$ linearly interpolated between times $[0.0, 0.3, 0.5, 0.85, 1.0]$. CRAFT used a diagonal affine flow, with a learning rate of $1e-3$.
\paragraph{GMM2} Step sizes $[4, 3, 2.5, 2.1, 2]$ linearly interpolated between times $[0.0, 0.3, 0.5, 0.85, 1.0]$. CRAFT used a diagonal affine flow, with a learning rate of $1e-3$.
\paragraph{GMM5} Step sizes $[5, 3.3, 2.3, 1.8, 1.6]$ linearly interpolated between times $[0.0, 0.25, 0.5, 0.8, 1.0]$. CRAFT used a diagonal affine flow, with a learning rate of $1e-3$.
\paragraph{GMM10} Step sizes $[3, 2, 1.5, 1.5]$ linearly interpolated between times $[0.0, 0.5, 0.85, 1.0]$. CRAFT used a diagonal affine flow, with a learning rate of $1e-3$.
\paragraph{GMM20} Step sizes $[3, 1.8, 1.4, 1.3]$ linearly interpolated between times $[0.0, 0.5, 0.85, 1.0]$. CRAFT used a diagonal affine flow, with a learning rate of $1e-3$.

We used SMC with the above settings for $1000$ steps to estimate the `ground truth' normalizing constant on the Bayesian posterior targets.

\subsubsection{DDS settings}
Optimal settings for $\alpha_\text{max}$ are given in \cref{tab:dds_settings}. Note that we do not tune $\sigma$ as in \citet{vargasDDSampler2023} since we used a variational approximation. We also tuned DDS on the \verb|GMM| tasks but did not present the results as they were not competitive with PDDS and CRAFT.

\begin{table}[]
	\centering
	\begin{tabular}{lrrrrrr}
		\toprule
		         & Base step & 1    & 2    & 4    & 8    & 16   \\
		\midrule
		Gaussian & 1         & 0.86 & 0.86 & 1.00 & 0.96 & 0.82 \\
		Mixture  & 1         & 0.28 & 0.28 & 0.36 & 0.52 & 0.54 \\
		Brownian & 1         & 0.76 & 0.76 & 0.84 & 0.80 & 0.72 \\
		Funnel   & 2         & 0.60 & 0.68 & 0.68 & 0.60 & 0.64 \\
		Ion      & 2         & 0.68 & 0.80 & 0.74 & 0.64 & 0.52 \\
		Sonar    & 2         & 0.68 & 0.82 & 0.78 & 0.64 & 0.50 \\
		LGCP     & 8         & 0.74 & 0.62 & 0.60 & 0.44 & 0.26 \\
		GMM1     & 2         & 0.22 & 0.28 & 0.24 & 0.20 & 0.22 \\
		GMM2     & 2         & 0.14 & 0.18 & 0.18 & 0.16 & 0.16 \\
		GMM5     & 2         & 0.20 & 0.28 & 0.26 & 0.24 & 0.20 \\
		GMM10    & 4         & 0.36 & 0.34 & 0.30 & 0.26 & 0.22 \\
		GMM20    & 8         & 0.40 & 0.32 & 0.32 & 0.26 & 0.18 \\
		\bottomrule
	\end{tabular}
	\caption{Optimal settings for $\alpha_\text{max}$. The number of steps for a given entry is the base steps in the first column multiplied by the step multiplier in the zeroth row.}
	\label{tab:dds_settings}
\end{table}

\subsubsection{PIS settings}
Optimal settings for $\sigma$ are given in \cref{tab:pis_settings}. Note that we were unable to obtain reasonable performance with PIS with only 1 step. We also tuned PIS on the \verb|GMM| tasks but did not present the results as they were not competitive with PDDS and CRAFT.

\begin{table}[]
	\centering
	\begin{tabular}{lrrrrrr}
		\toprule
		         & Steps & 1     & 2     & 4     & 8     & 16    \\
		\midrule
		Gaussian & 1     & NA    & 1.00  & 1.00  & 1.00  & 1.00  \\
		Mixture  & 1     & NA    & 2.40  & 2.20  & 1.92  & 1.88  \\
		Brownian & 1     & NA    & 0.08  & 0.10  & 0.10  & 0.13  \\
		Funnel   & 2     & 1.50  & 1.00  & 1.00  & 1.00  & 1.00  \\
		Ion      & 2     & 0.37  & 0.40  & 0.40  & 0.46  & 0.46  \\
		Sonar    & 2     & 0.25  & 0.31  & 0.40  & 0.46  & 0.49  \\
		LGCP     & 8     & 1.36  & 1.64  & 1.78  & 1.99  & 2.06  \\
		GMM1     & 2     & 15.00 & 14.00 & 15.00 & 15.00 & 16.00 \\
		GMM2     & 2     & 4.40  & 1.30  & 1.30  & 7.30  & 8.70  \\
		GMM5     & 2     & 1.30  & 1.30  & 1.40  & 1.40  & 1.40  \\
		GMM10    & 4     & 1.30  & 1.30  & 1.30  & 1.30  & 1.30  \\
		GMM20    & 8     & 1.30  & 1.30  & 1.30  & 1.20  & 1.20  \\
		\bottomrule
	\end{tabular}
	\caption{Optimal settings for $\sigma$. The number of steps for a given entry is the base steps in the first column multiplied by the step multiplier in the zeroth row. We were unable to tune PIS with one step size.}
	\label{tab:pis_settings}
\end{table}

\subsubsection{PDDS settings}
No tuning of the cosine noise schedule was required. We performed systematic resampling \cite{douc2005comparison} when the ESS dropped below $30\%$. PDDS-MCMC used 10 Metropolis-adjusted Langevin MCMC steps with step sizes tuned based on initial runs with the initial simple approximation, targeting an acceptance rate of approximately 0.6. The resulting step sizes can be found below. We used 20 iterations of PDDS, each trained for 500 steps with a fresh instance of the learning rate schedule ($1e-3$ with exponential decay at a rate of $0.95$ per 50 iterations). At each refinement we initialise the potential approximation at it's previous state, rather than training from scratch at each refinement. We found that for a limited computational budget, better performance was obtained for a fast iteration rate (20 iterations with 500 training steps each). We also tested a slower iteration rate (2 iterations with 10,000 training steps each) which performed equivalently but required a larger overall training budget. The fast iteration schedule uses a lower number of density evaluations but has a larger training time due to requiring more frequent compilation of the PDDS sampler.

\paragraph{Gaussian} Step sizes $[0.1, 0.2, 0.5, 0.6]$ linearly interpolated between times $[0, 0.5, 0.75, 1.0]$.
\paragraph{Mixture} Step sizes $[0.05, 0.15, 0.4, 0.6]$ linearly interpolated between times $[0, 0.5, 0.75, 1.0]$.
\paragraph{Funnel} Step sizes $[0.4, 0.3, 0.5, 0.6, 0.6] $ linearly interpolated between times $[0, 0.2, 0.5, 0.75, 1.0]$. We found that the Langevin MCMC can become unstable on the \verb|Funnel| task due to extreme gradients of the density function, therefore at each iteration of PDDS we reduced the step sizes by 50\% for the first 10 iterations.
\paragraph{Brownian} Step sizes $[0.2, 0.4, 0.5, 0.5]$ linearly interpolated between times $[0, 0.5, 0.75, 1.0]$.
\paragraph{Ion} Step sizes $[0.15, 0.25, 0.5, 0.6]$ linearly interpolated between times $[0, 0.5, 0.75, 1.0]$.
\paragraph{Sonar} Step sizes $[0.1, 0.1, 0.18, 0.32, 0.4]$ linearly interpolated between times $[0, 0.25, 0.5, 0.75, 1.0]$.
\paragraph{LGCP} Step sizes $[0.1, 0.1, 0.15, 0.2] $ linearly interpolated between times $[0, 0.5, 0.75, 1.0]$.
\paragraph{GMM1} Step sizes $[10,11,20,35,100]$ linearly interpolated between times $[0.0,0.25,0.5,0.75,1.0]$.
\paragraph{GMM2} Step sizes $[2,2.5,4,6,12,50]$ linearly interpolated between times $[0.0,0.25,0.5,0.6,0.75,1.0]$.
\paragraph{GMM5} Step sizes $[1.5,1.5,2.5,4,8,30]$ linearly interpolated between times $[0.0,0.25,0.5,0.6,0.75,1.0]$.
\paragraph{GMM10} Step sizes $[1,1.2,1.8,3,6,20]$ linearly interpolated between times $[0.0,0.25,0.5,0.6,0.75,1.0]$.
\paragraph{GMM20} Step sizes $[0.8,1.0,2,3,6,20]$ linearly interpolated between times $[0.0,0.25,0.5,0.6,0.75,1.0]$.
\subsection{Ablation studies}
\label{app:additional_ablations}

\paragraph{Importance of SMC and iterative potential approximations}
Here we study the behaviour of PDDS on a gaussian mixture task where the initial (here referred to as `naive') approximation \cref{eq:naive_approximation} only captures one out of three modes when simulating the reverse SDE \cref{eq:conditioned_diffusion}. \cref{fig:mix3_samples_pdds} shows that our method is able to recover the unknown modes after only two iterations of potential training. Following the discussion in \cref{sec:potential_iterations}, there are two mechanisms at play which allow our method to recover additional modes which are missed by the naive approximation.

Firstly, our asymptotically correct SMC scheme \cref{algo:smc_dds} means that we do not simply re-learn the potential function of the previous iteration during training, but instead we learn an improved potential function since the training samples are `improved' by the SMC scheme. We evidence the improvement in potential functions by plotting the path of distributions induced by \cref{eq:conditioned_diffusion} with the learnt potential at each PDDS iteration in \cref{fig:mix3_paths_smc}. We see that at each iteration of PDDS the sequence of distributions moves closer to the true denoising sequence for the given target. Furthermore we show that our SMC scheme is critical by showing the behaviour of PDDS if the SMC scheme is ignored (i.e. remove all resampling steps). \cref{fig:mix3-samples-no-smc} and \cref{fig:mix3-paths-no-smc} show the resulting samples and sequence of distributions when we simply simulate the reverse SDE \cref{eq:conditioned_diffusion} using the previous potential approximation without applying the SMC correction scheme. We observe that very few particles do still reach the missing modes but this is not compounded in each iteration and after the first iteration no improvements are made.

The second, more subtle, mechanism at play is that the learning of the potential approximation is not based on the sample alone, but also uses information from the target distribution via the way we parametrize our neural network (\cref{sec:nn_param}). This is true for both the original and NSM losses (although the advantage of the NSM loss is to provide a lower variance regression target than the original DSM loss). We note, however, that this second mechanism alone is not sufficient without the help of SMC, again illustrated in \cref{fig:mix3-samples-no-smc} and \cref{fig:mix3-paths-no-smc}.

\begin{figure}
	\centering
	\includegraphics[width=0.6\textwidth]{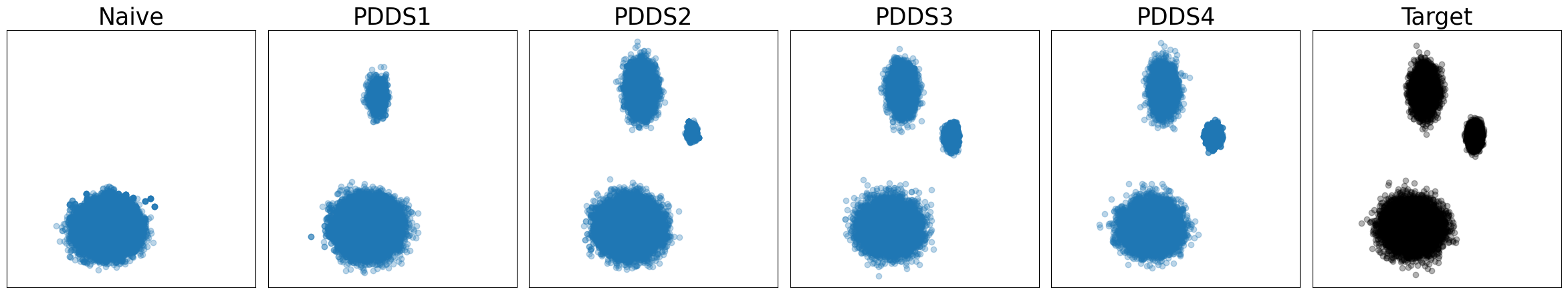}
	\caption{Samples from PDDS on a 3-mode mixture of Gaussians.}
	\label{fig:mix3_samples_pdds}
\end{figure}

\begin{figure}
	\centering
	\includegraphics[width=0.6\textwidth]{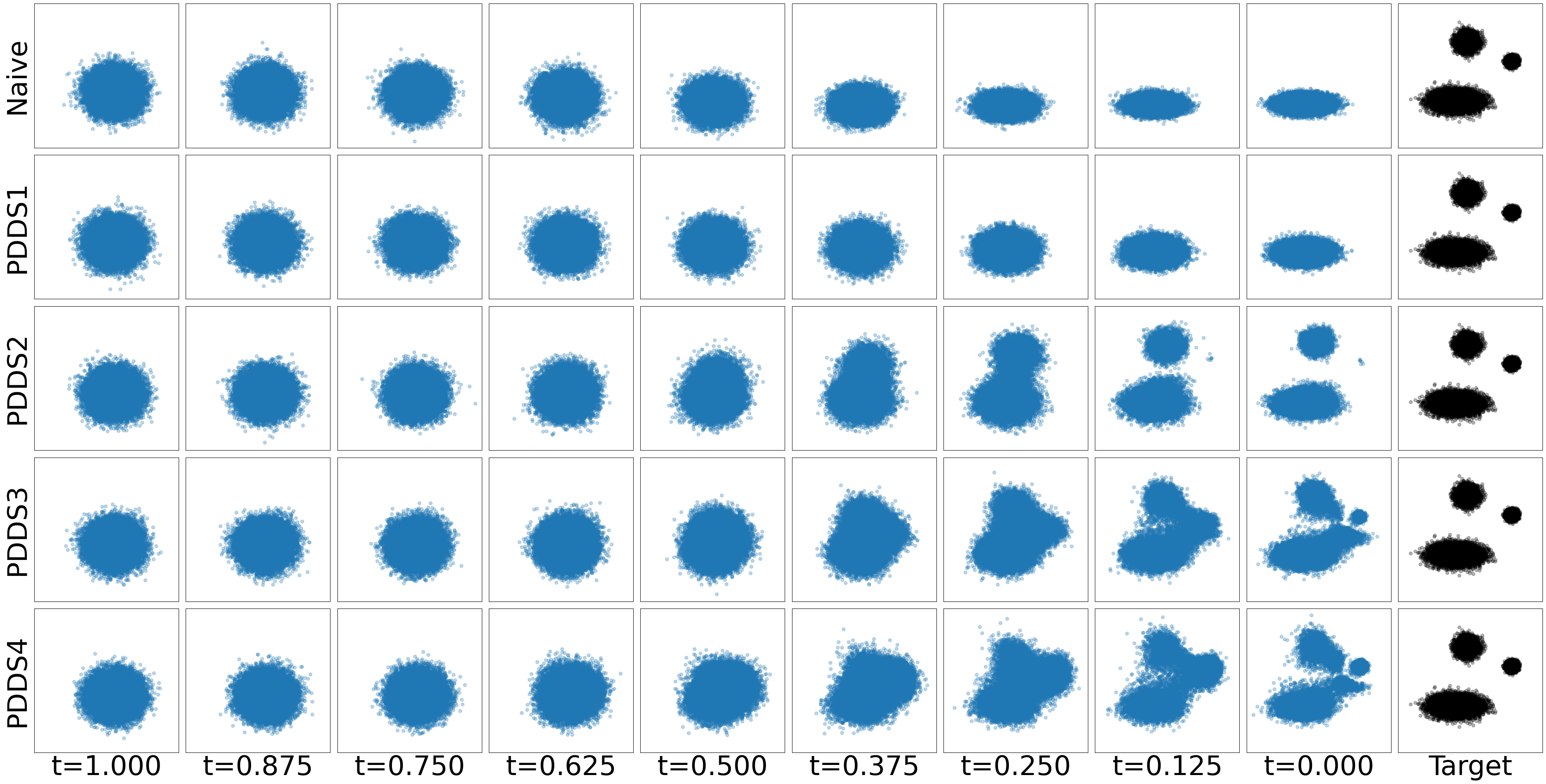}
	\caption{Left to right: marginal distributions of the reverse SDE \cref{eq:conditioned_diffusion} using the potential approximation at each iteration of PDDS.}
	\label{fig:mix3_paths_smc}
\end{figure}

\begin{figure}
	\centering
	\includegraphics[width=0.6\textwidth]{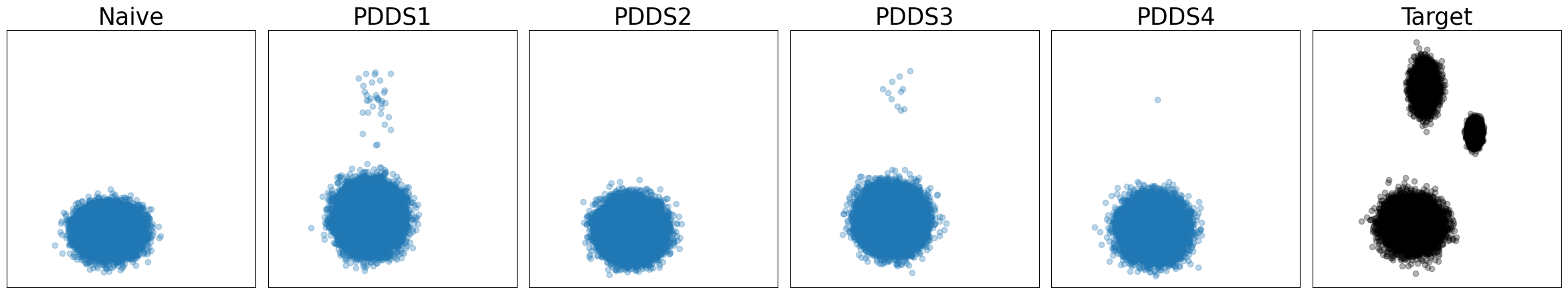}
	\caption{Samples from PDDS without SMC on a 3-mode mixture of Gaussians}
	\label{fig:mix3-samples-no-smc}
\end{figure}

\begin{figure}
	\centering
	\includegraphics[width=0.6\textwidth]{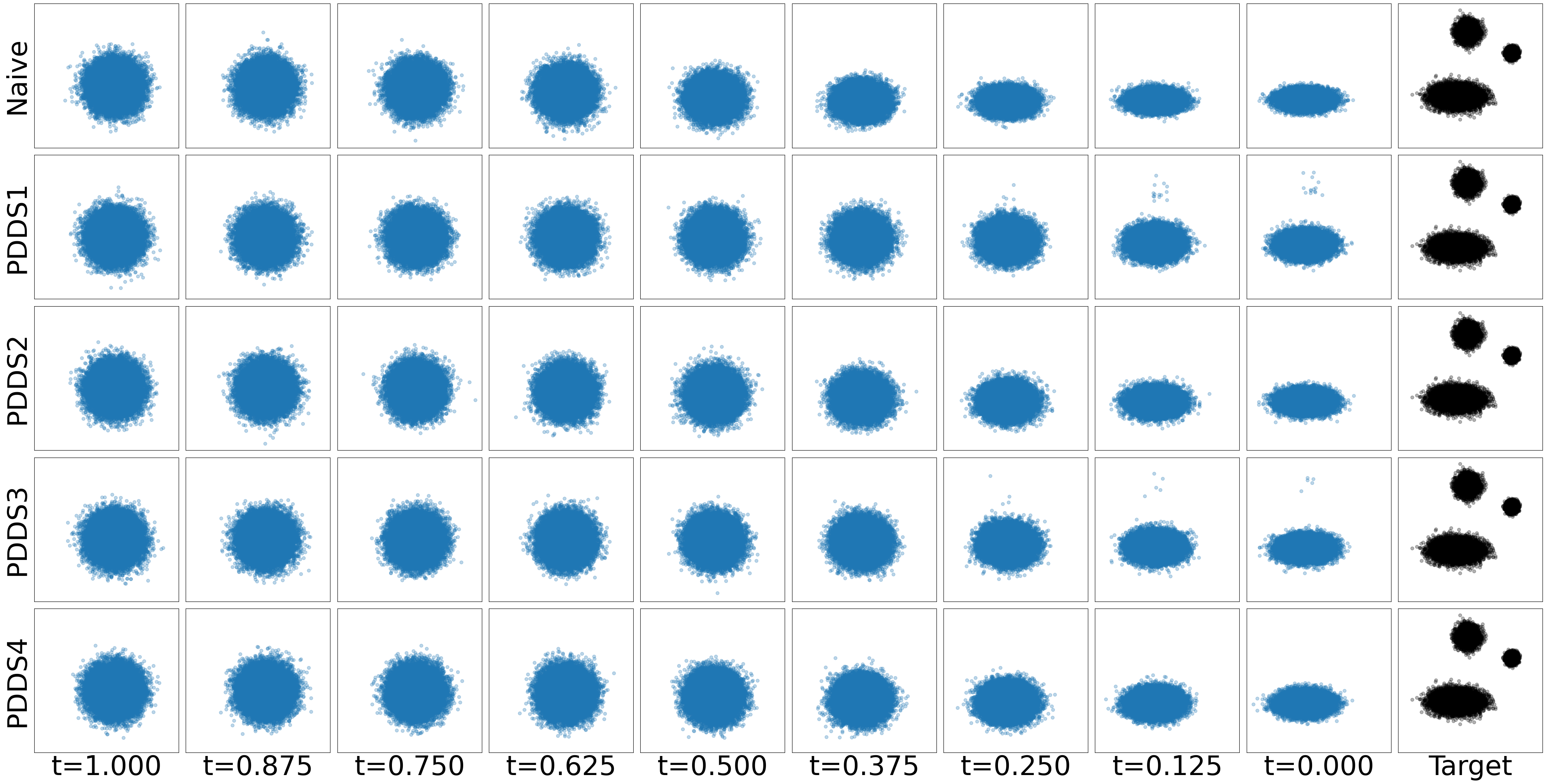}
	\caption{Left to right: marginal distributions of the reverse SDE \cref{eq:conditioned_diffusion} using the potential approximation at each iteration of PDDS without SMC.}
	\label{fig:mix3-paths-no-smc}
\end{figure}

\paragraph{Cosine scheduler}
As stated above, we follow the cosine scheduler introduced by \citet{nichol2021improved}. We found, as demonstrated in \Cref{fig:schedulers} and \Cref{fig:schedulers_ess}, that the cosine schedule was effective in ensuring the forward SDE converges to the target distribution while regularly spacing the ESS drops across the sampling path.

\begin{figure}
	\centering
	\begin{subfigure}[b]{0.32\textwidth}
		\centering
		\includegraphics[width=\textwidth]{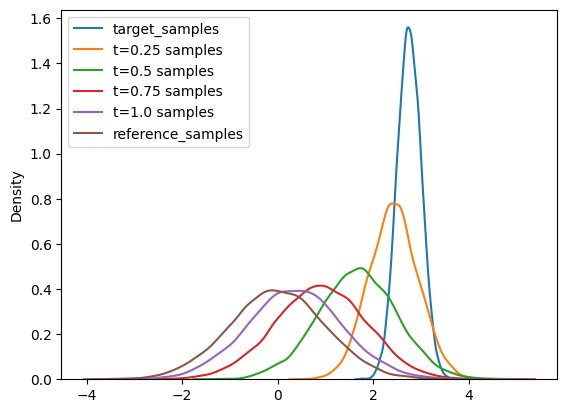}
		\label{fig:bad_lin_schedule}
	\end{subfigure}
	\hfill
	\begin{subfigure}[b]{0.32\textwidth}
		\centering
		\includegraphics[width=\textwidth]{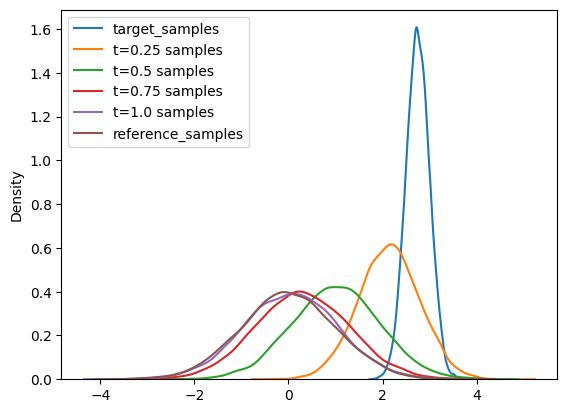}
		\label{fig:lin_schedule}
	\end{subfigure}
	\hfill
	\begin{subfigure}[b]{0.32\textwidth}
		\centering
		\includegraphics[width=\textwidth]{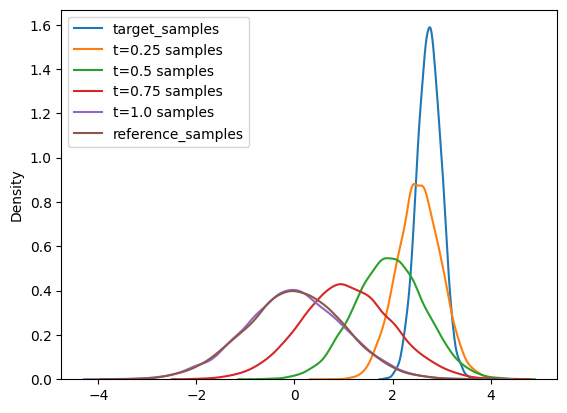}
		\label{fig:cosine_schedule}
	\end{subfigure}
	\caption{Left: linear schedule, $\beta_f=8$. Middle: linear schedule: $\beta_f=12$. Right: cosine schedule.}
	\label{fig:schedulers}
\end{figure}

\begin{figure}
	\centering
	\begin{subfigure}[b]{0.32\textwidth}
		\centering
		\includegraphics[width=\textwidth]{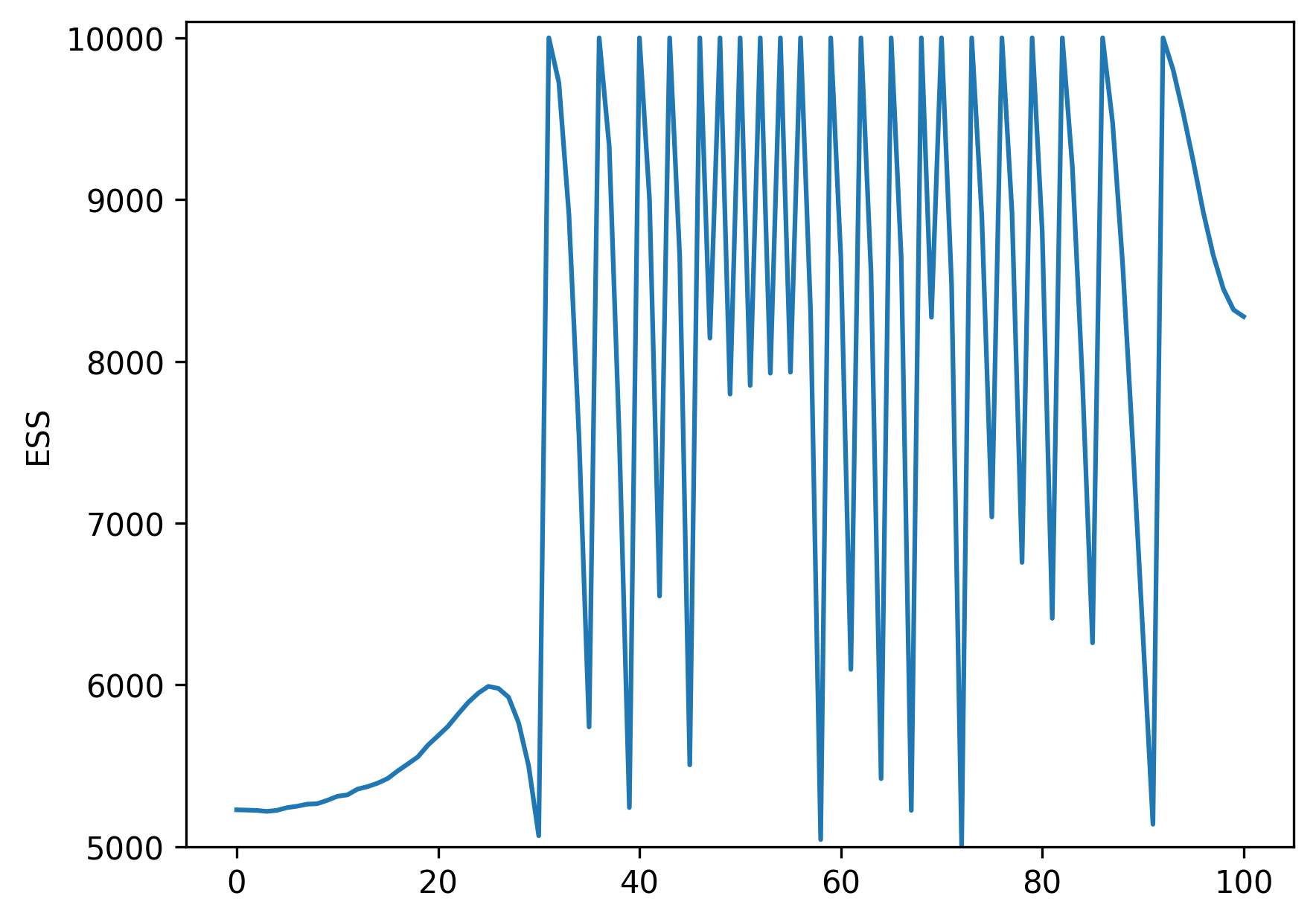}
		\label{fig:bad_lin_schedule_ess}
	\end{subfigure}
	\hfill
	\begin{subfigure}[b]{0.32\textwidth}
		\centering
		\includegraphics[width=\textwidth]{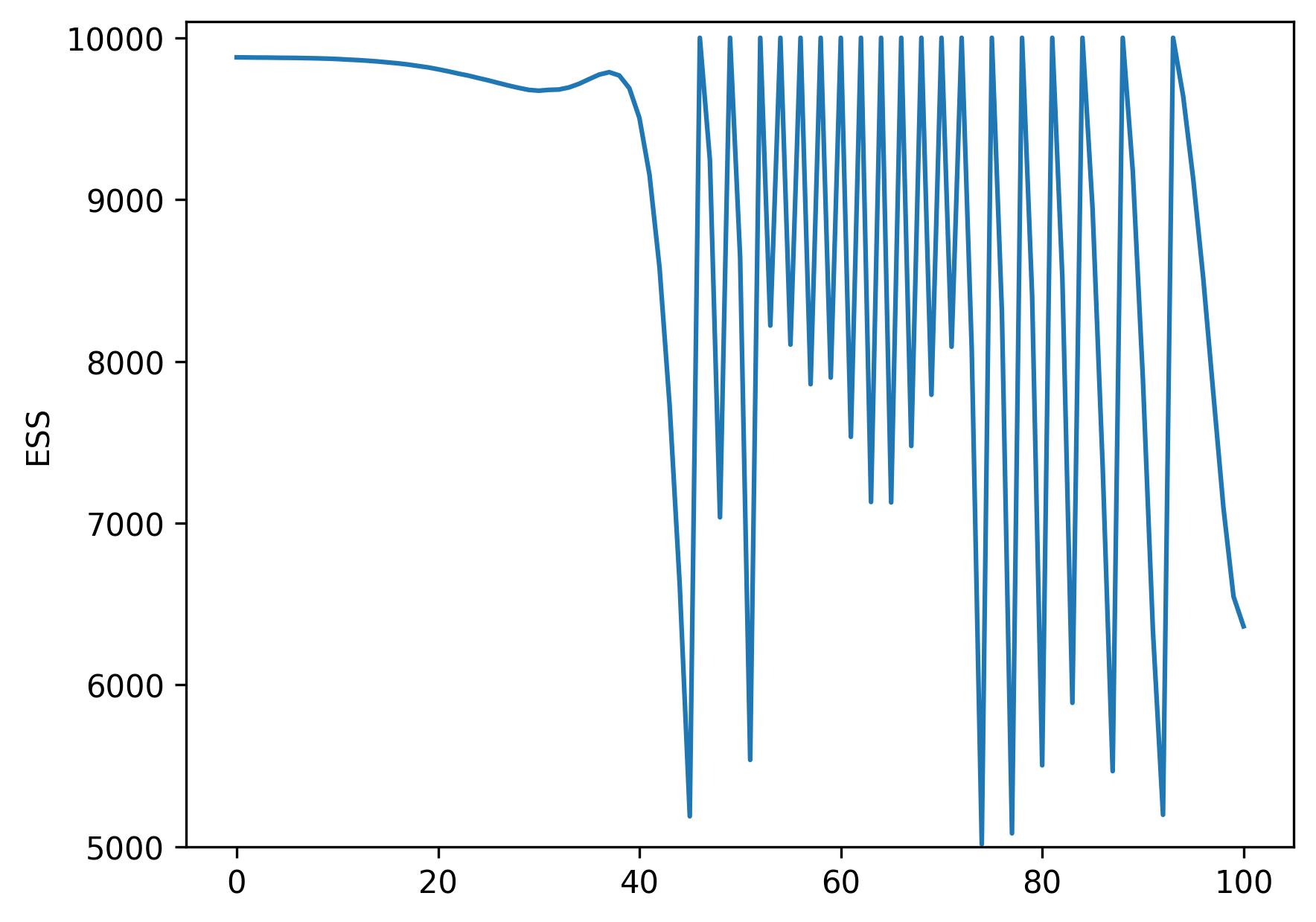}
		\label{fig:lin_schedule_ess}
	\end{subfigure}
	\hfill
	\begin{subfigure}[b]{0.32\textwidth}
		\centering
		\includegraphics[width=\textwidth]{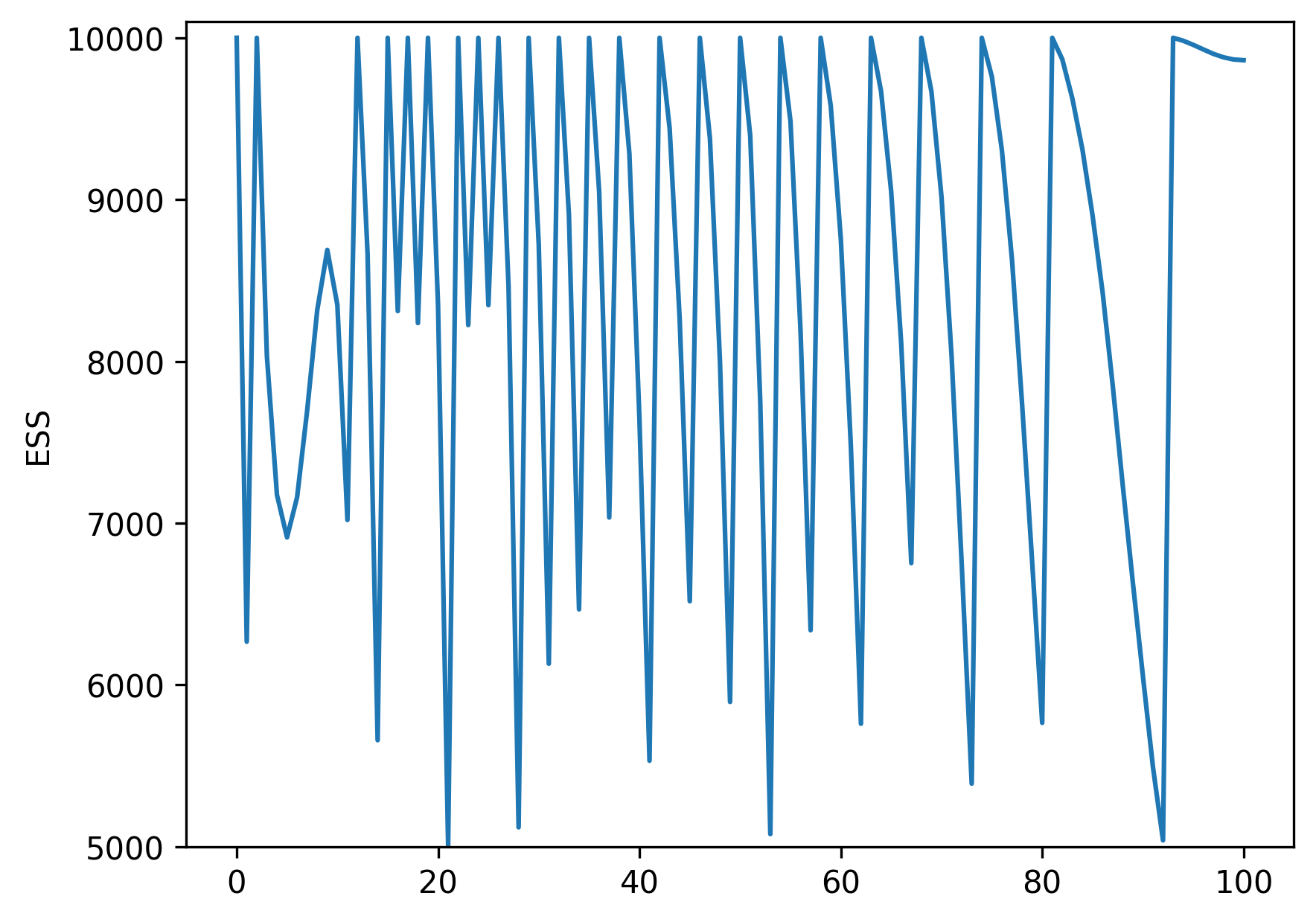}
		\label{fig:cosine_schedule_ess}
	\end{subfigure}
	\caption{Left: linear schedule, $\beta_f=8$. Middle: linear schedule: $\beta_f=12$. Right: cosine schedule.}
	\label{fig:schedulers_ess}
\end{figure}

\paragraph{Number of particles} In \Cref{fig:n_ablation} we show that the normalizing constant estimation error with the naive potential approximation does decrease as the number of particles increases. However, we could not eliminate the error before exceeding the computer memory. We conclude that we must improve our initial potential approximation to obtain feasible results, hence motivating the iterative potential approximation scheme.

\begin{figure}
	\centering
	\begin{subfigure}[b]{0.475\textwidth}
		\centering
		\includegraphics[width=\textwidth]{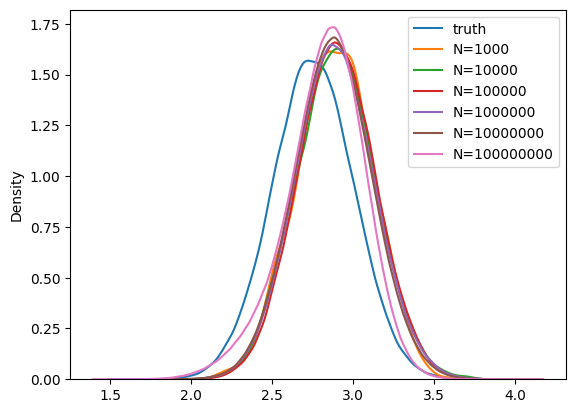}
		\label{fig:n_particles_samples}
	\end{subfigure}
	\hfill
	\begin{subfigure}[b]{0.475\textwidth}
		\centering
		\includegraphics[width=\textwidth]{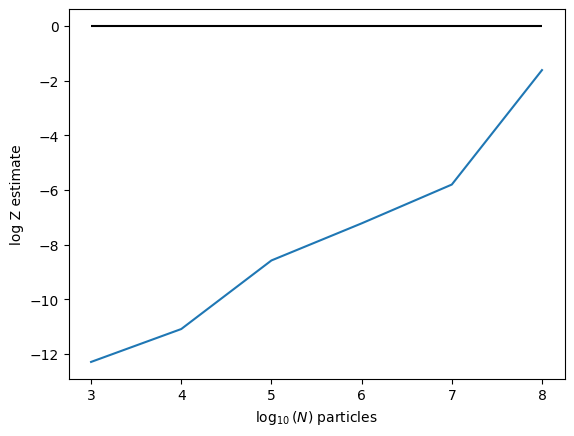}
		\label{fig:n_particles_log_Z}
	\end{subfigure}
	\caption{Samples and estimated normalizing constant from PDDS on $\mathcal{N}(2.75, 0.25^2)$ target.}
	\label{fig:n_ablation}
\end{figure}

\paragraph{Guidance path of NN potential approximation} In \Cref{fig:corrected_paths} we demonstrate that using a neural network to correct the naive approximation has the effect of correcting the path of the guidance SDE. Here we use a linear schedule for $\beta_t$ since the cosine schedule does not allow for analytic roll-out of the SDE.

\begin{figure}
	\centering
	\begin{subfigure}[b]{0.32\textwidth}
		\centering
		\includegraphics[width=\textwidth]{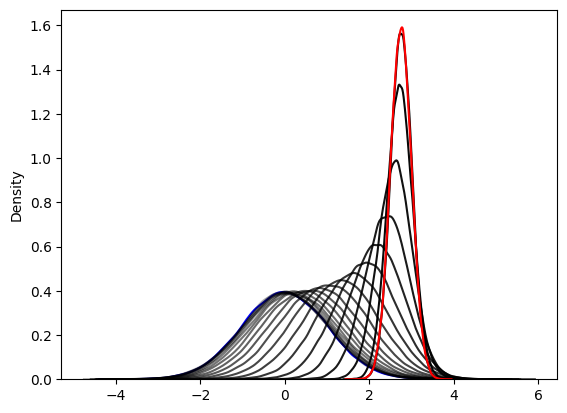}
		\label{fig:exact_path}
	\end{subfigure}
	\hfill
	\begin{subfigure}[b]{0.32\textwidth}
		\centering
		\includegraphics[width=\textwidth]{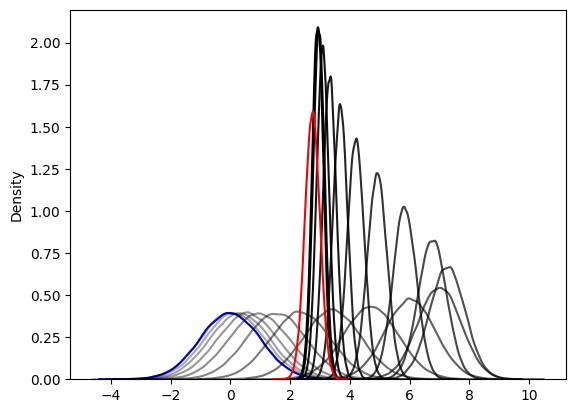}
		\label{fig:naive_path}
	\end{subfigure}
	\hfill
	\begin{subfigure}[b]{0.32\textwidth}
		\centering
		\includegraphics[width=\textwidth]{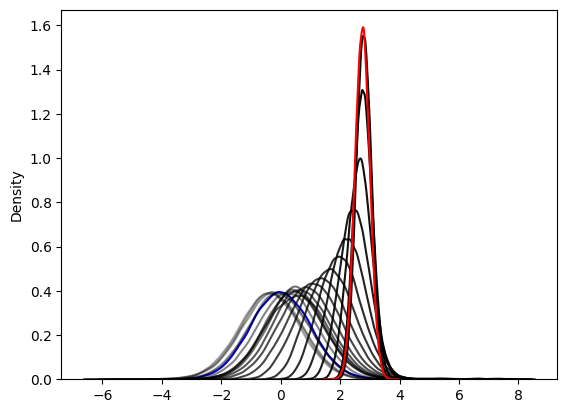}
		\label{fig:nn_path}
	\end{subfigure}
	\caption{Guidance path using analytic potential function (left), naive potential approximation (middle) and neural network potential approximation (right).}
	\label{fig:corrected_paths}
\end{figure}

\subsection{Additional results}
\label{app:additional_results}

In \cref{fig:log_Z_all_tasks} we display the normalising constant estimates on all tasks, including the \verb|Brownian| and \verb|Ion| tasks which were omitted from the main text for space.

\begin{figure}
	\centering
	\includegraphics[width=\textwidth]{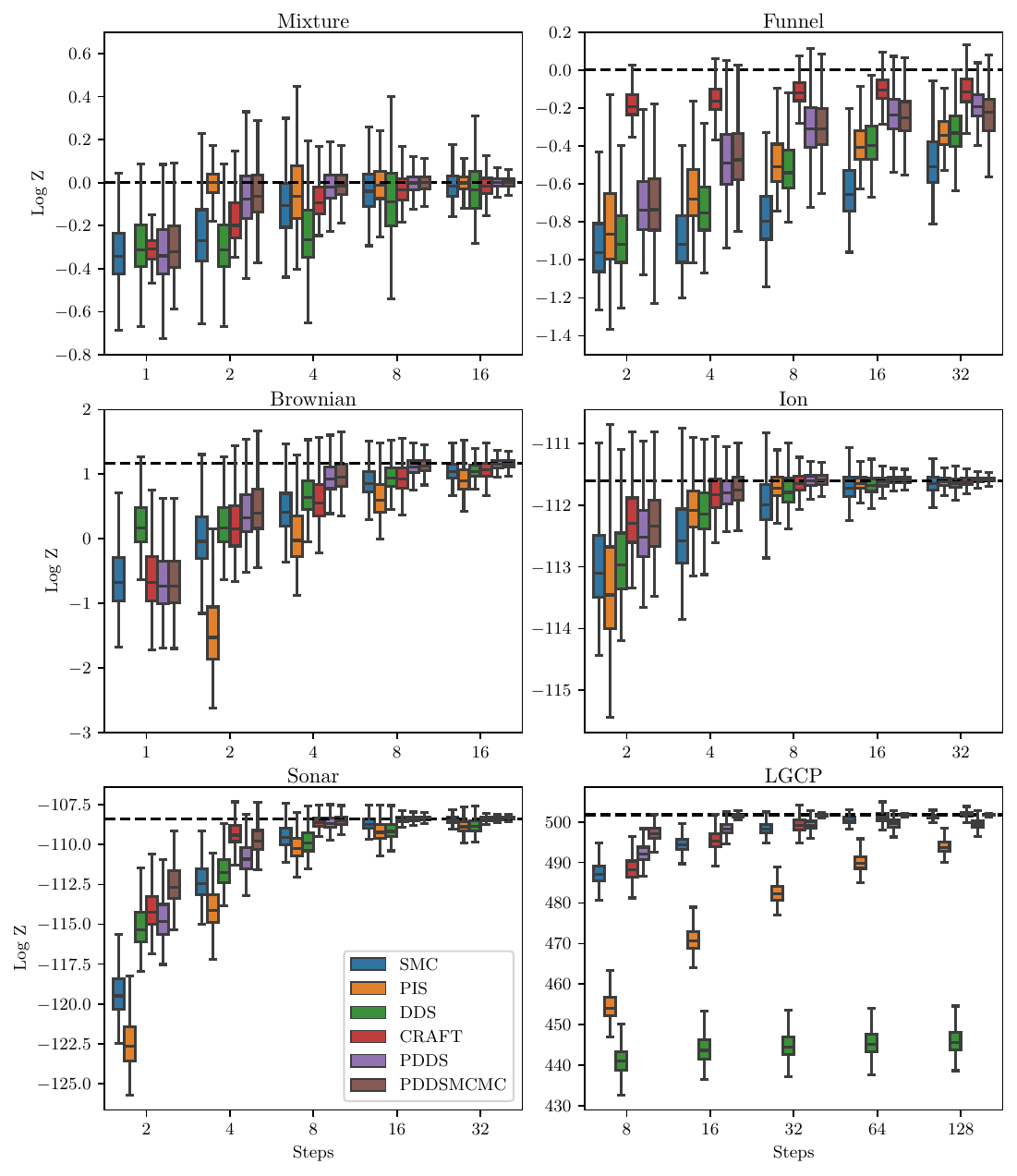}
	\caption{Normalizing constant estimation results on all tasks. Outliers are hidden for clarity. Each box consists of 2000 estimates, coming from 20 training seeds each with 100 evaluation seeds.}
	\label{fig:log_Z_all_tasks}
\end{figure}

In \cref{fig:gmm_all_tasks} we display the normalising constant estimation results and in \cref{fig:gmm_w2_all} the $\mathcal{W}_2^\gamma$ distances for the \verb|GMM| task in 1, 2, 5, 10, 20 dimensions.

\begin{figure}
	\centering
	\includegraphics[width=\textwidth]{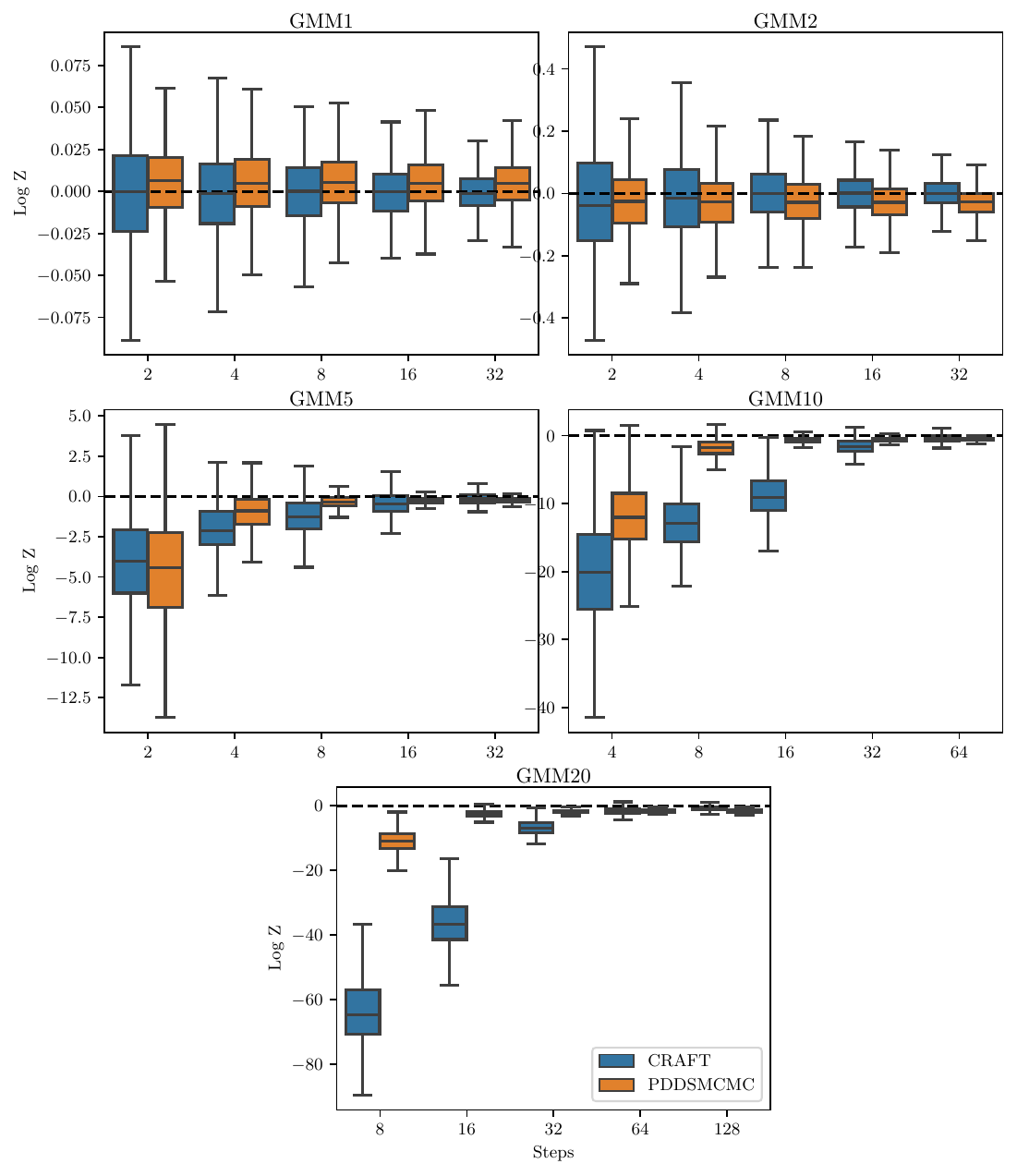}
	\caption{Normalizing constant estimation results on the GMM task in 1, 2, 5, 10 and 20 dimensions. Outliers are hidden for clarity. Each box consists of 1000 estimates coming from 10 training seeds and 100 evaluation seeds.}
	\label{fig:gmm_all_tasks}
\end{figure}

\begin{figure}
	\centering
	\includegraphics[width=\textwidth]{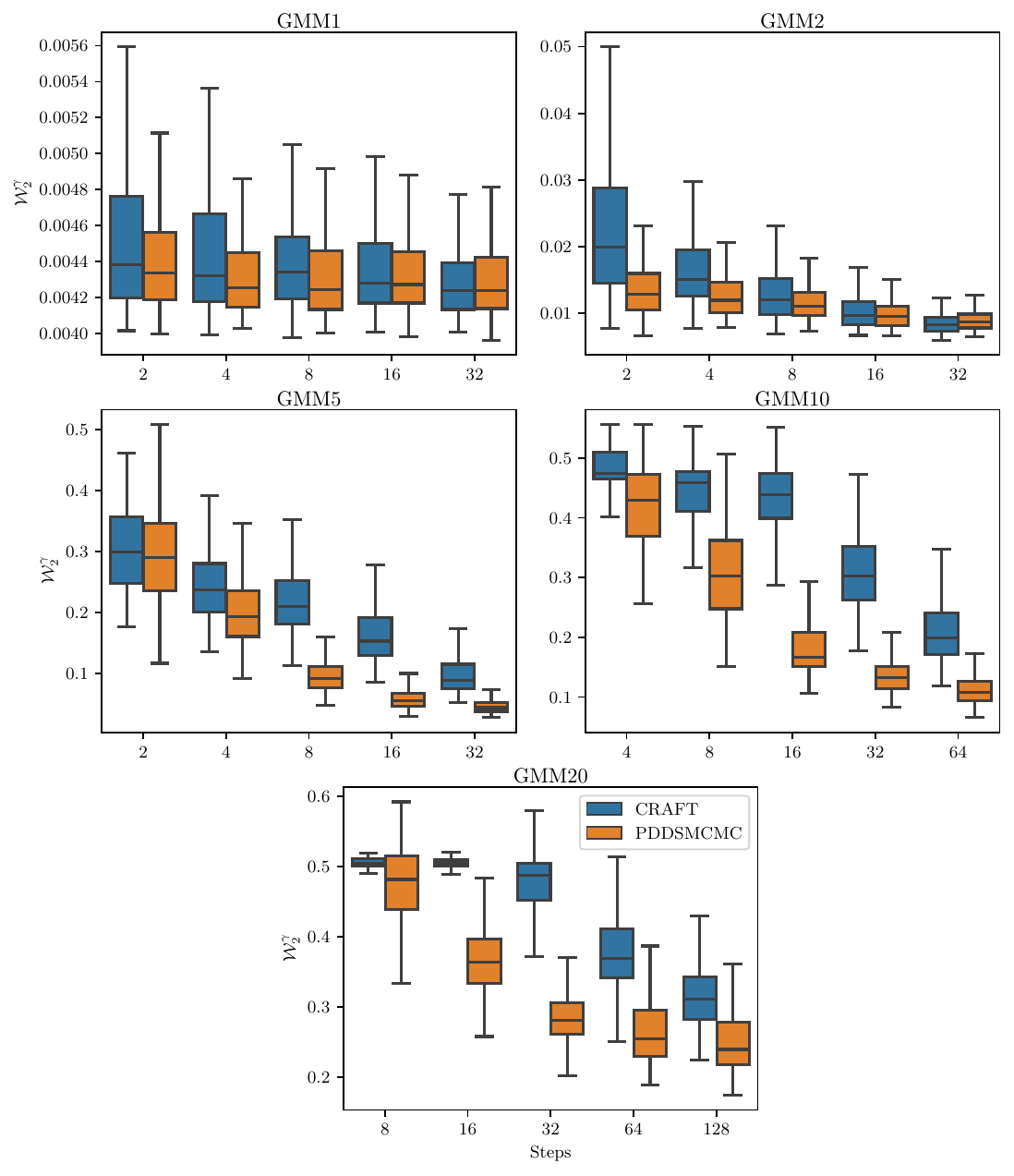}
	\caption{Entropy-regularized Wasserstein-2 distances between samples from the model and target distributions for CRAFT and PDDS-MCMC. Lower is better. Each box consists of 200 values coming from 10 training seeds and 20 evaluation seeds.}
	\label{fig:gmm_w2_all}
\end{figure}

\subsection{Uncurated normalizing constant estimates}
In \cref{fig:log_Z_uncurated} we display the normalising constant estimation results of \cref{fig:log_Z_all_tasks} with outliers present. Note that all methods are susceptible to erroneous over-estimation of the normalising constant due to numerical errors an instability. It appears that all methods are equally susceptible to this issue, with no single method displaying more erroneous overestimation than the others. Results on the Gaussian task are displayed separately in \cref{fig:log_Z_gaussian}.

\begin{figure}
	\centering
	\includegraphics[width=\textwidth]{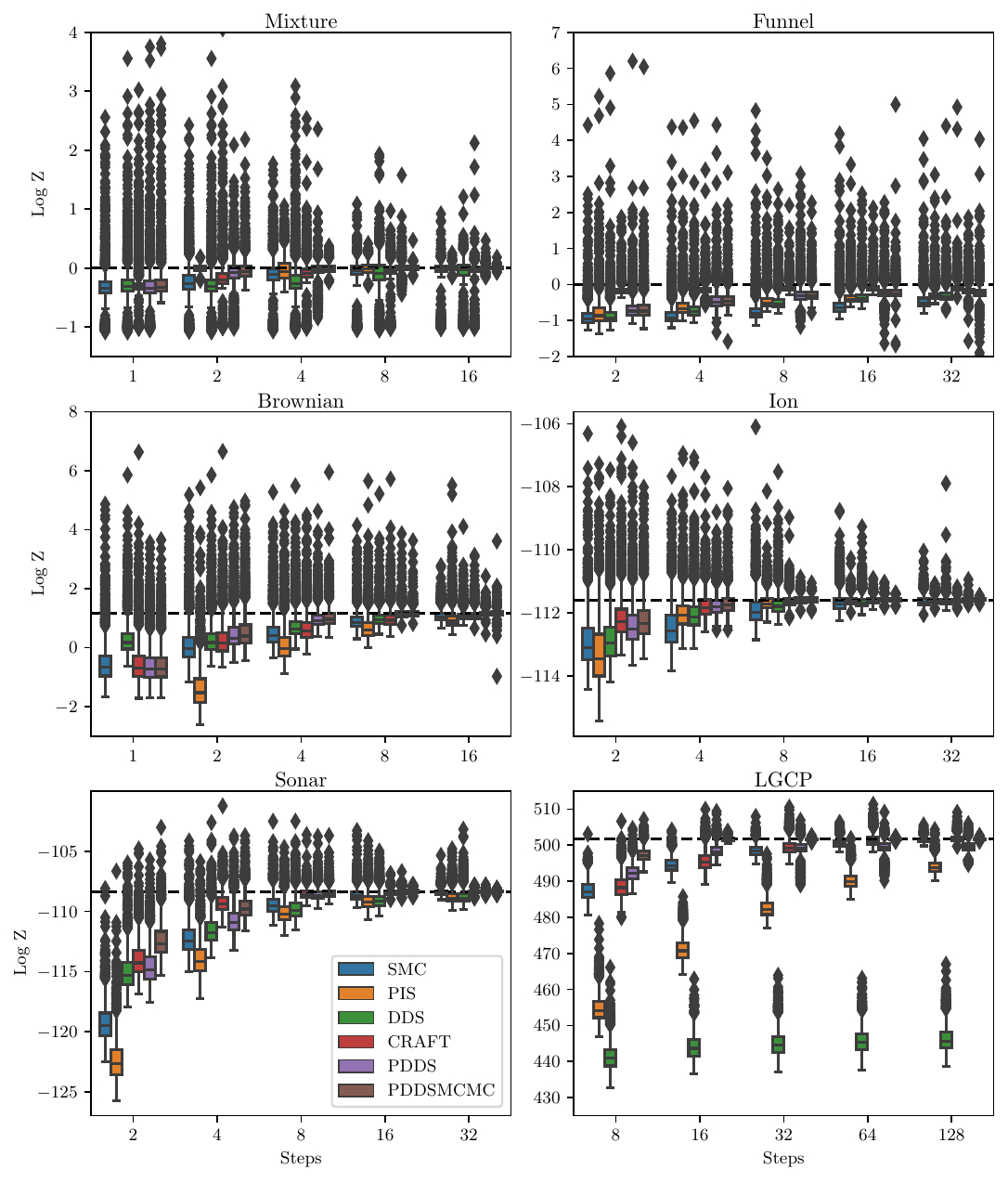}
	\caption{Normalizing constant estimation results on all tasks. Each box consists of 2000 estimates, coming from 20 training seeds each with 100 evaluation seeds.}
	\label{fig:log_Z_uncurated}
\end{figure}

\begin{figure}
	\centering
	\includegraphics[width=\textwidth]{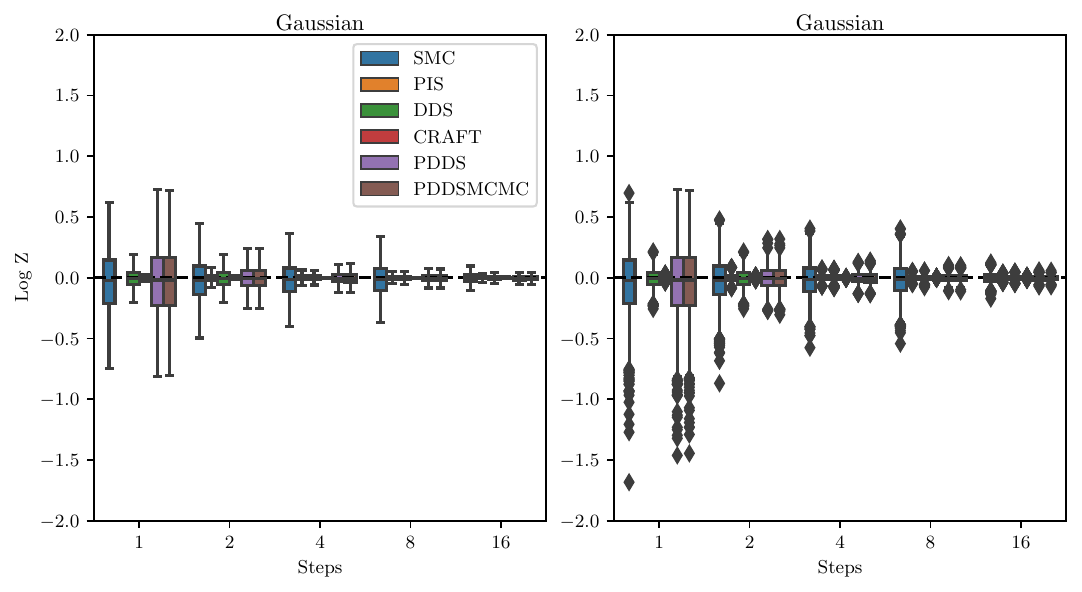}
	\caption{Normalizing constant estimation results on the Gaussian task. Each box consists of 2000 estimates, coming from 20 training seeds each with 100 evaluation seeds. Left: outliers removed, right: uncurated.}
	\label{fig:log_Z_gaussian}
\end{figure}

\end{document}